\pdfoutput=1
\documentclass{article} 
\usepackage{iclr2027_conference,times}
\makeatletter\@ifundefined{iflatexml}{\newif\iflatexml}{}\makeatother

\usepackage{amsmath,amsfonts,bm}

\def\eqref#1{equation~\ref{#1}}

\def\1{\bm{1}}

\DeclareMathAlphabet{\mathsfit}{\encodingdefault}{\sfdefault}{m}{sl}
\SetMathAlphabet{\mathsfit}{bold}{\encodingdefault}{\sfdefault}{bx}{n}

\usepackage{amsmath}
\usepackage{amsthm}
\usepackage{xcolor}
\usepackage{colortbl}
\definecolor{oursrow}{RGB}{230,239,252}   
\usepackage{longtable}   
\usepackage{hyperref}
\usepackage{url}

\usepackage{booktabs}
\newcommand{\dashvrule}{\raisebox{-\dp\strutbox}{\vbox to \dimexpr\ht\strutbox+\dp\strutbox\relax{%
  \leaders\vbox{\hrule height 1.1pt width 0.4pt\vskip 1.1pt}\vfill}}}
\newcolumntype{:}{!{\dashvrule}}
\iflatexml\newcolumntype{:}{|}\fi   
\usepackage{graphicx}
\graphicspath{{figures/}}
\providecommand{\ours}{\mbox{\textsc{CausalBridge}}}
\iflatexml\renewcommand{\ours}{\textsc{CausalBridge}}\fi   
\newcommand{\indep}{\perp\!\!\!\perp}            
\newcommand{\nindep}{\not\!\perp\!\!\!\perp}     
\iflatexml\renewcommand{\nindep}{\not\perp\!\!\!\!\!\perp}\fi   
\usepackage[section]{placeins}   
\usepackage{pifont}
\usepackage{framed}
\definecolor{shadecolor}{gray}{0.93}
\newenvironment{takeaway}{\setlength{\fboxsep}{3pt}\begin{snugshade}}{\end{snugshade}\vspace{-4pt}}

\title{How Causality Bridges the Semantic Gap}

\iflatexml
\author{Shuhao Zhang \\ {University of California, San Diego}
\And Xuran Zhou \\ {University of California, San Diego}
\And Han Guo \\ {University of California, San Diego}
\And Pengtao Xie\thanks{Corresponding author: \texttt{p1xie@ucsd.edu}.} \\ {University of California, San Diego}
\And Yujia Zheng\thanks{Corresponding author: \texttt{yujiaz@illinois.edu}.} \\ {University of Illinois Urbana-Champaign}}
\date{Project page: \url{https://sullivan07043.github.io/CausalBridge/}}
\else
\author{Shuhao Zhang$^{1}$\quad Xuran Zhou$^{1}$\quad Han Guo$^{1}$\quad Pengtao Xie$^{1}$\setcounter{footnote}{1}\thanks{Corresponding authors: \texttt{p1xie@ucsd.edu}, \texttt{yujiaz@illinois.edu}.}\quad Yujia Zheng$^{2}$\footnotemark[2] \\[3pt]
$^{1}$University of California, San Diego\qquad $^{2}$University of Illinois Urbana-Champaign \\
\normalfont Project page: \url{https://sullivan07043.github.io/CausalBridge/}
}
\fi

\iclrfinalcopy 
\newtheorem{theorem}{Theorem}
\newtheorem{proposition}{Proposition}
\newtheorem{definition}{Definition}
\newtheorem{condition}{Condition}

\renewcommand{\topfraction}{0.9}
\renewcommand{\bottomfraction}{0.6}
\renewcommand{\textfraction}{0.07}
\renewcommand{\floatpagefraction}{0.8}

\newsavebox{\fittabbox}
\newcount\fitcols
\newlength{\fitsep}
\makeatletter
\newcommand{\fittable}[2]{%
  \begin{lrbox}{\fittabbox}\begin{tabular}{#1}#2\end{tabular}\end{lrbox}%
  \ifdim\wd\fittabbox>\textwidth
    \resizebox{\textwidth}{!}{\usebox{\fittabbox}}%
  \else
    \fitcols=0 \@tfor\fitc:=#1\do{\if\fitc:\else\advance\fitcols by 1\fi}%
    \multiply\fitcols by 2
    \setlength{\fitsep}{\textwidth}\addtolength{\fitsep}{-\wd\fittabbox}%
    \divide\fitsep by \fitcols
    \addtolength{\fitsep}{\tabcolsep}%
    {\setlength{\tabcolsep}{\fitsep}\begin{tabular}{#1}#2\end{tabular}}%
  \fi}
\makeatother
\iflatexml\renewcommand{\fittable}[2]{\begin{tabular}{#1}#2\end{tabular}}\fi
\begin{document}

\maketitle
\lhead{Preprint}   

\begin{abstract}
Numerical measurements capture how a system behaves, but often leave the meanings of its variables unspecified. Some variables are measured but never labeled, and others are never measured at all. Existing methods assign semantics to such variables by consulting general human knowledge, but this inherits its biases where that knowledge exists and offers nothing where it does not.
We bridge this gap between measurements and their meanings with causal structure instead, reading a variable's semantics from how it acts on other variables. We formalize this as structure-constrained semantic alignment, in which the embedding of each unnamed variable is solved under the dependence relations implied by the causal graph, with the embeddings of a few known names as anchors. Accordingly, we build \textbf{\ours{}}, a framework that
discovers the causal graph from the measurements, latent variables included, solves for the embeddings under those relations, and expresses them as names through a language model. The causal structure reflects the mechanism that generated the measurements and is recovered from the measurements alone, which may make it the one source of information free of bias from human knowledge. We evaluate \ours{} on five questionnaires and three robotics scenarios, with 20 to 90\% of the variable names masked. It recovers the semantics of observed and latent variables more accurately than existing methods that rely on association, and its lead widens as less of the system is documented. The graph it discovers names variables as accurately as the documented one, and a new system is named in minutes and at a fraction of the cost of sampling methods. Once the semantic gap is bridged faithfully, machines can understand the world and take actions causally.
\end{abstract}
\suppressfloats[t]   
\begin{figure}[!htb]\centering
\includegraphics[width=0.93\textwidth]{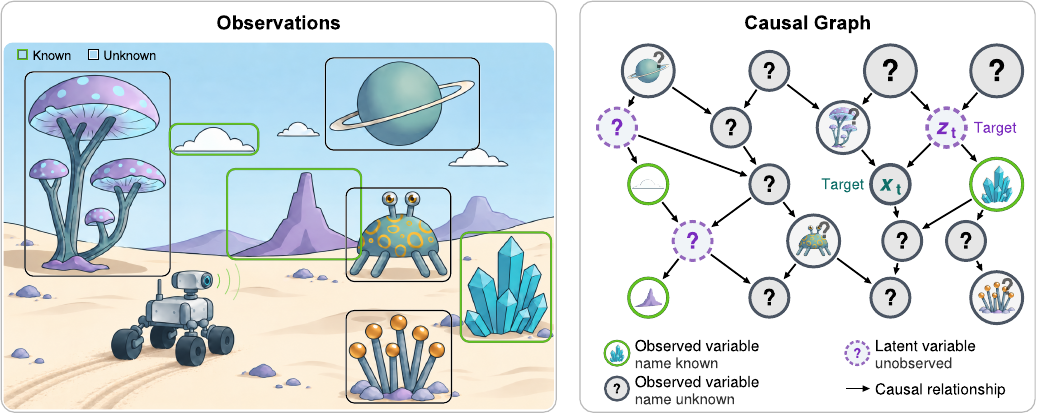}   
\caption{A system meets a world it has no words for. Left: a rover on a new planet, whose sensors return numbers for everything in view, can name a few things (green boxes) and has no name for most (gray boxes), and nobody wrote a manual for this world. Right: the causal graph that the numbers alone reveal, with latent variables (dashed purple) never measured yet acting on what is. We ask how this structure, anchored by the few known names, lets the system name the unnamed observed variable $x_t$ and the latent variable $z_t$. The arrows are the relations that the numbers alone reveal, and the question marks stand for the names the system does not have.}
\label{fig:setting}
\end{figure}

\section{Introduction}

An intelligent system sees the real world as numbers, while people understand the same world with semantics, and between the two sit unnamed variables with no semantics assigned. Some are measured but never documented, as in a shared corpus of robot demonstrations whose observation and action spaces differ from robot to robot \citep{OpenX}, so the meaning of each channel must be found again for every robot. Others, such as the latent variable behind a group of questionnaire items and the feature inside a network that a set of activations reflects, are never measured at all and appear only through their effect on what is measured. Causal discovery can recover such a latent variable from data \citep{Silva2006}, but only as a numerical variable with a structural position in a graph and no name (Figure~\ref{fig:setting}). A name turns that variable into a hypothesis a person can test or a concept a system can communicate. This semantic gap appears where human knowledge covers only a few of the variables, and closing it lets a system know the world it measures beyond what people have already written down. We therefore ask how a system can come to know such a world from its data alone, with few priors.

Existing strategies assign semantics to an unnamed variable by reading
human knowledge. Some propagate semantics from named neighbors over a supplied graph \citep{FeaturePropagation, GraphMAE}. Others match the activation profile of the variable against a fixed concept bank \citep{NetworkDissection, CLIPDissect}. Still others ask a language model what the active cases of the variable suggest \citep{autointerpBill, Delphi, SASC}, or what the variables around a latent variable in a recovered graph suggest \citep{TLVD}. Causal discovery and causal representation learning recover the structure among observed and latent variables and state when the latent ones are identifiable \citep{FOFC, GIN, Scholkopf2021, Khemakhem2020, Zheng2022}. ThoughtComm \citep{ThoughtComm} carries recovered latent variables into a language model, but none of these methods names what it recovers.

One instance shows where these strategies fail (Figure~\ref{fig:teaser}(b)). Consider a drive recorded from the CAN bus of a car \citep{comma2k19}, on which the car's controllers exchange messages. Six of its twelve channels are named, and one of the unnamed ones is the longitudinal acceleration. AutoInterp names it pitch rate and gives the same name to another unnamed channel, while the SASC-style method names it front left wheel speed and gives that name to three other unnamed channels. Neither name is correct. The unnamed channels carry different signals, yet each method gives one name to several of them. The name therefore reflects the language prior about what a car records, and the measurements have little effect on it. Where human knowledge covers the variables, the inferred name inherits its biases, and where it does not, there is nothing to retrieve.

\vspace{-4pt}
\begin{center}\itshape
How can a system recover the semantics of the variables it measures,
where human knowledge covers few of them?
\end{center}
\vspace{-4pt}

We answer this question with \textbf{\ours{}}, which first recovers the causal structure of a system from its measurements alone, without reading any name, and adds a latent variable wherever the dependence among the observed variables requires one (Figure~\ref{fig:teaser}b). It then solves for the meaning of every unnamed variable under that structure, where the vectors of the few named variables serve as anchors and the relations determine every unnamed vector relative to them. Finally, a frozen language model translates each vector into a name, and existing semantics enter at this step. In the graph discovered for that drive, the vertical acceleration causes the unnamed channel, which in turn causes the pitch rate. Both arise when the car brakes or accelerates, so the channel between them is its acceleration along the road, and \ours{} writes longitudinal acceleration. The structure supplies what human knowledge does not cover, and the language model supplies the words.

This paper makes four contributions. \textbf{First}, we \emph{discover the
semantics of a system from its causal structure}, and the language model
only writes the name. A system can therefore come to know the world it
measures from its own data. \textbf{Second}, we formalize the task as \emph{structure-constrained semantic alignment}. In this task, the embedding of every unnamed variable, observed or latent, is solved under the relations that the causal graph implies. The few known names serve as anchors. \textbf{Third}, we realize the alignment as \ours{}, which is \emph{a framework with discovered structure, causally grounded names and efficient deployment}. \textbf{Fourth}, we evaluate \ours{} on three robotics scenarios and five questionnaires with 20 to 90\% of the true names masked. We show that \emph{the lead is largest where the least has been written down}. At 50\% masking, \ours{} outperforms the strongest baseline by 2.2 times on the channels of an undocumented vehicle CAN. On latent variables, the lead grows from 1.5 to 1.9 times as the masking rises from 20 to 75\%. The graph \ours{} discovers matches the documented graph in naming accuracy, and one recipe fits any backbone. A new system is named in minutes, at 2.6 to 56.6 times fewer tokens per latent variable than the methods that describe sampled cases.

\begin{figure}[!h]\centering
\includegraphics[width=\textwidth,trim=0 7.5 0 12,clip]{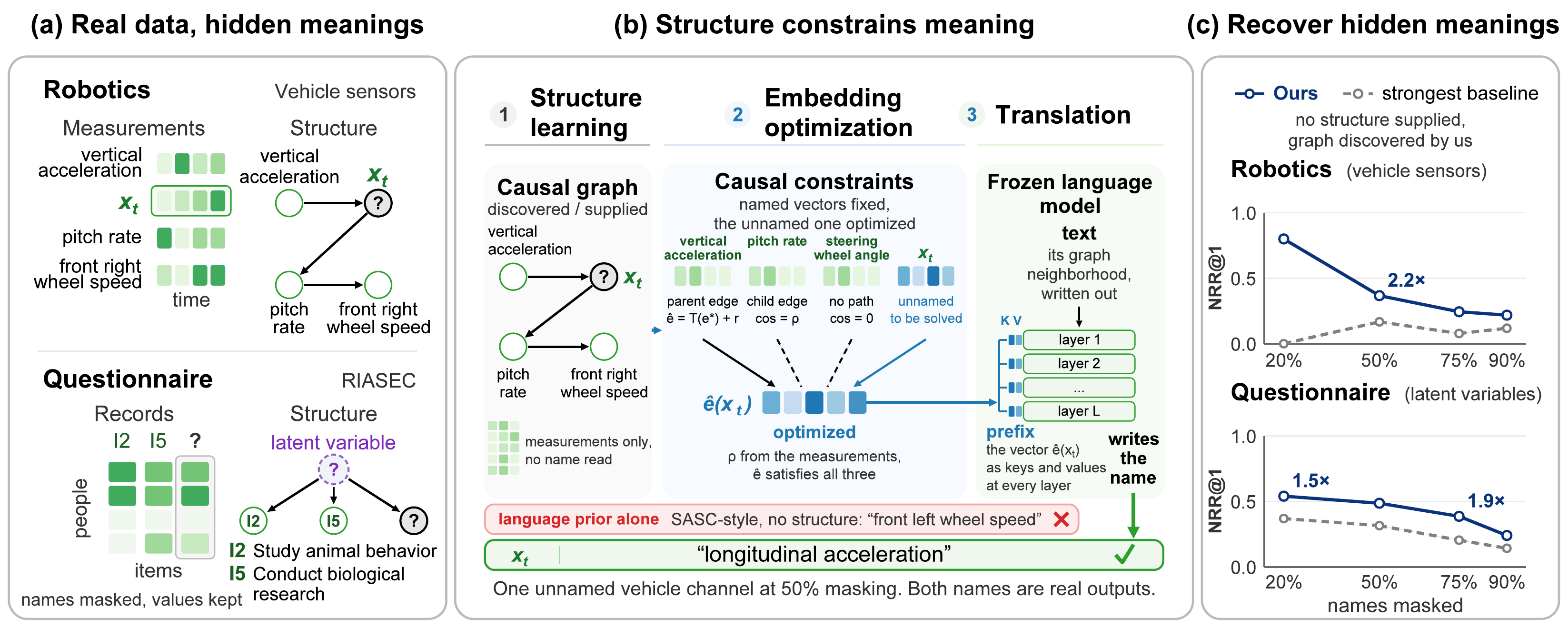}
\caption{(a) Two systems from the evaluation at 50\% masking: four channels of one vehicle CAN drive and three items of the RIASEC scale \citep{RIASEC} with the latent variable that causes them, each with its structure. A question mark marks an unknown name, and $x_t$ is the channel to be named. (b)~The three steps of \ours{} on $x_t$: the structure is learned from the measurements; the vector of $x_t$ is optimized under its structural relations to the named variables, whose vectors stay fixed; and a frozen language model takes this vector as a prefix and writes the name, shown below with the one the strongest baseline wrote without structure. (c)~NRR@1, the share of unnamed variables whose written name matches the true name best among all names of the same system, for \ours{} and the strongest baseline at each masking level with no structure supplied.}
\label{fig:teaser}
\end{figure}

\section{Problem Formulation}
\label{sec:formulation}

Given the measurements of a system and the names of a few of its variables, we aim to recover the semantics of the rest, including the latent variables that are never measured.

\begin{takeaway}
\textbf{Key idea.} The semantics of a variable can be recovered from its causal relations to the other variables. A causal graph contains no names, but once a few variables are named, the relations of an unnamed variable to them constrain its meaning.
\end{takeaway}

\paragraph{Data-generating process.}
Consider a system with variables $V = O \cup L$, where $O$ and $L$ denote the observed and the latent variables, respectively. We assume that the variables are generated by a structural causal model \citep{Pearl2009} with a directed acyclic graph $G$, in which an edge from $u$ to $v$ means that $u$ is a direct cause of $v$. The measurements $X$ contain $n$ samples of the observed variables, while the latent variables are not measured. Each variable $v$ has a name $\ell(v)$, but the names are known only for a subset $V_{\mathrm{N}}$ of the observed variables, which we call the named variables.
\begin{equation}
V_i = f_i\big(\mathrm{pa}(V_i), \varepsilon_i\big), \qquad
X \in \mathbb{R}^{n \times |O|}, \qquad
\ell : V \to \mathcal{T}, \qquad
V_{\mathrm{N}} \subseteq O, \quad
V_{\mathrm{U}} = (O \setminus V_{\mathrm{N}}) \cup L ,
\label{eq:setting}
\end{equation}
where $\mathrm{pa}(V_i)$ denotes the parents of $V_i$ in $G$, $\varepsilon_i$ is the noise term, $\mathcal{T}$ is the space of text, and $V_{\mathrm{U}}$ collects the unnamed variables. Variables with a common parent are siblings. In our evaluation, $L$ is empty for a robot or a vehicle CAN, and contains the latent variables of a questionnaire.

\paragraph{Causal graphs and d-separation.}
Following \citet{SGS2000}, the causal graph $G$ characterizes the causal relations among the variables, and d-separation derives from $G$ the conditional independence relations among them, associating dependence with connectedness in the graph and independence with separation. A path between two variables is any consecutive sequence of edges, disregarding their directions. A set $S$ blocks a path if the path contains a chain $a \to m \to b$ or a fork $a \leftarrow m \to b$ whose middle variable $m$ is in $S$, or a collider $a \to m \leftarrow b$ such that neither $m$ nor any descendant of $m$ is in $S$. $S$ d-separates $a$ and $b$ if every path between them is blocked by $S$. The collider case reflects that conditioning on a common effect makes its causes dependent. In Figure~\ref{fig:dsep}, $x$ and $y$ are d-separated by the empty set, since the only path between them passes the collider at $t$, while $s$ and $y$ are d-connected given $\{r, p\}$ because $p$ is a descendant of $t$ and is in the conditioning set. Under the Markov and faithfulness assumptions, $S$ d-separates $a$ and $b$ if and only if they are independent given $S$, denoted by $a \indep b \mid S$ (Definition~\ref{def:markov}). These relations are determined by the causal graph without semantic information, and we use them to find the semantics.

\begin{figure}[!h]\centering
\includegraphics[width=0.76\textwidth]{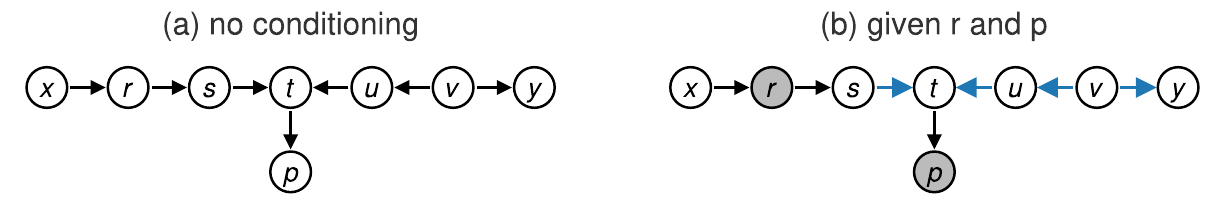}
\caption{d-separation on a graph adapted from \citet[Section 11.1.2]{Pearl2009}. A shaded
variable is in the conditioning set $S$. (a) With $S$ empty, $x$ and $y$
are d-separated by the collider at $t$. (b) Given $\{r, p\}$, $s$ and $y$ are d-connected, since $p$ is a descendant of the collider $t$ and is in $S$.}
\label{fig:dsep}
\end{figure}

\paragraph{The naming task.}
We formulate the naming task as finding vectors that realize the causal structure and translating the vector of each unnamed variable into a name. To this end, we assign each variable $v$ a vector $e_v$ in the embedding space of a text encoder $\phi : \mathcal{T} \to \mathbb{R}^d$, so that each named variable keeps the vector $e^\star_v = \phi(\ell(v))$ of its name, which we call an anchor. In this space a function $\delta$ of two vectors, such as their cosine similarity, measures how related their meanings are, so we use it as a surrogate for the dependence between the corresponding variables. We say that an assignment $e : V \to \mathbb{R}^d$ realizes the structure if it satisfies three conditions. First, the vector of each variable is generated from the vectors of its parents. Second, the values of $\delta$ among the vectors reflect the dependence and independence relations that $G$ implies, with the strength of each dependence estimated from the measurements $X$. Third, the vectors of the named variables stay fixed to their anchors. Let $\mathcal{A}(G, X)$ denote the feasible set of all assignments that realize the structure, which Section~\ref{sec:method-solver} states formally. The goal is to find an assignment $\hat e \in \mathcal{A}(G, X)$ and to translate $\hat e_v$ into a name $\hat\ell(v)$ that agrees with the true name $\ell(v)$ for each unnamed variable $v \in V_{\mathrm{U}}$.

The first two conditions constrain only the relations among the vectors. Edges determine which vectors should be similar, and d-separations determine which should remain distinct. The anchors remove part of this freedom, and the choice that remains within the feasible set is determined by existing knowledge, which enters as a regularizer toward English phrases (Section~\ref{sec:method-solver}) and as the language model that translates the vector (Section~\ref{sec:method-translation}).

\section{Method}
\label{sec:method}

Based on the formulation above, we propose the framework \ours{}, which names the unnamed variables of a system in three steps (Figure~\ref{fig:teaser}(b) and Figure~\ref{fig:pipeline}). First, we learn the causal structure $\hat G$ from the measurements alone so that its relations are free of the biases of human knowledge. Second, we optimize the vectors of the unnamed variables under this structure, with the named variables as anchors. Third, to write the meaning in words without losing the relations that determined it, a frozen language model takes the vector of each unnamed variable directly as input and translates it into a name. The steps are modular, so a supplied causal graph can substitute for the discovered one at the second step.

\begin{takeaway}
\textbf{Running example.} This section follows the vehicle CAN drive of the introduction at 50\% masking (Figure~\ref{fig:teaser}(a)). Its named channels include the vertical acceleration, the pitch rate and the steering wheel angle, and $x_t$ is an unnamed channel whose true name is longitudinal acceleration.
\end{takeaway}

\subsection{Causal Structure Learning}
\label{sec:method-discovery}

\looseness=-1
The causal structure $G$ describes how the variables of a system cause
one another, and it is essential for our task, since it
represents the true mechanism underlying the measurements and is therefore
agnostic to the biases of existing knowledge. On a vehicle the vertical acceleration and the pitch rate both arise when the car brakes or accelerates, so an unnamed channel between them on a causal path measures the acceleration along the road, whatever a language prior expects it to be.

\looseness=-1
To learn the causal structure with guarantees, we apply causal discovery algorithms, which under their assumptions identify it from the measurements without any semantic information (Appendix~\ref{apx:theory}). Specifically, when all variables are observed, as on a robot or a vehicle CAN, we use best order score search (BOSS) \mbox{\citep{BOSS}}, a score-based search over the orderings of the variables. Since these measurements form a time series, we search over two consecutive time steps and keep the edges that agree with the temporal order. When latent variables exist, as on a questionnaire, we use rank-based latent causal discovery (RLCD) \citep{Dong2024}, given its flexibility to handle general structures that involve both latent and observed variables. Intuitively, respondents who rate ``Study animal behavior'' highly on the RIASEC questionnaire of Figure~\ref{fig:teaser}(a) tend to rate ``Conduct biological research'' highly as well, because one interest causes both. RLCD detects such a latent variable from the rank of the covariance of the measurements. Both algorithms are theoretically supported and empirically verified on various setups (Appendices~\ref{apx:theory} and~\ref{apx:discovery}), and the output is the causal structure $\hat G$ over the observed and the latent variables, with no semantics attached.

\begin{takeaway}
\textbf{Takeaway.} In the running example, the discovered graph places
$x_t$ on a causal path from the vertical acceleration to the pitch rate
and separates it from the steering wheel angle. Because these relations
are learned from the measurements without reading any name, they already
restrict what $x_t$ can measure. None of the baselines learns a
structure, so none of them can use such relations.
\end{takeaway}

\subsection{Constrained Embedding Optimization}
\label{sec:method-solver}

We use the relations that the causal graph implies to constrain where the vector of each unnamed variable can lie in the embedding space, given the vectors of the named variables. In the running example, $x_t$ is generated from the vertical acceleration, generates the pitch rate and is d-separated from the steering wheel angle. Specifically, let $\rho_{ab \cdot S}$ denote the strength of the dependence between $a$ and $b$ given $S$, estimated from the measurements, and let $e_{v \cdot S}$ denote the vector of $v$ after removing the part explained by $S$. When $S$ is empty, $e_{v \cdot S}$ is simply $e_v$, and when $S$ consists of the parents of $a$ and $b$, $e_{v \cdot S}$ is the residual $r_v = e_v - g_v(e_{\mathrm{pa}(v)})$. Here $g_v$ is the embedding-space counterpart of the causal mechanism $f_v$ in Eq.~\ref{eq:setting}, generating the vector of $v$ from those of its parents, and $r_v$ captures what the parents do not explain, as the noise term $\varepsilon_v$ does in the measurements. An assignment $e$ realizes the three conditions if it satisfies
\begin{equation}
\begin{alignedat}{4}
&e_v = g_v\big(e_{\mathrm{pa}(v)}\big) + r_v
&& \quad\text{\scriptsize generation}
&\quad \delta(e_{a \cdot S}, e_{b \cdot S}) &= 0 \; \text{if } a \indep b \mid S
&& \quad\text{\scriptsize independence} \\
&\delta(e_{a \cdot S}, e_{b \cdot S}) = \psi(\rho_{ab \cdot S}) \; \text{if } a \nindep b \mid S
&& \quad\text{\scriptsize dependence}
&\quad e_v &= e^\star_v \;\; \forall v \in V_{\mathrm{N}}
&& \quad\text{\scriptsize anchoring}
\end{alignedat}
\label{eq:realization}
\end{equation}
where $\psi$ denotes a monotone function that maps a dependence to the scale of $\delta$ (Appendix~\ref{apx:theory}). The function $\delta$ admits several choices, such as the cosine similarity or the conditional mutual information of a Gaussian model of the vectors (Appendix~\ref{apx:cmi-vectors}). We turn these conditions into an optimization problem by fixing each named variable at the vector $e^\star_v$ of its name and optimizing a vector $e_v \in \mathbb{R}^d$ for each unnamed variable, observed or latent. The objective:
\begin{align}
\mathcal{L}(e) =&
\underbrace{\sum\nolimits_{v:\, \mathrm{pa}(v) \neq \emptyset} \|r_v\|^2}_{\text{\scriptsize generation}}
+ \underbrace{\tfrac{\lambda_{\mathrm{ind}}}{|\mathcal{P}_{\mathrm{ind}}|} \sum\nolimits_{\mathcal{P}_{\mathrm{ind}}}
      ( \delta(e_a, e_b) - \rho_{ab} )^2}_{\text{\scriptsize marginal independence}}
\label{eq:objective} \\[-3pt]
&+ \underbrace{\tfrac{\lambda_{\mathrm{dep}}}{|\mathcal{P}_{\mathrm{dep}}|} \sum\nolimits_{\mathcal{P}_{\mathrm{dep}}}
      [\kappa |\rho_{ab}| - |\delta(e_a, e_b)| ]_+^2}_{\text{\scriptsize marginal dependence}}
+ \underbrace{\tfrac{\lambda_{\mathrm{con}}}{|\mathcal{P}_{\mathrm{con}}|} \sum\nolimits_{\mathcal{P}_{\mathrm{con}}}
      ( \delta(r_a, r_b) - \rho_{ab \cdot \mathrm{pa}} )^2}_{\text{\scriptsize conditional dependence}}
+ \mathcal{R}(e),
\nonumber
\end{align}
where $\mathcal{P}_{\mathrm{ind}}$ contains the pairs that $\hat G$ d-separates by the empty set, $\mathcal{P}_{\mathrm{dep}}$ the strongly dependent pairs, $\mathcal{P}_{\mathrm{con}}$ the pairs whose variables both have parents, $\rho_{ab}$ denotes the distance correlation between $a$ and $b$ \citep{DistanceCorrelation}, and $[x]_+ = \max(x, 0)$. The weights $\lambda$ scale the relation terms and $\kappa < 1$ is a fixed fraction. The four terms encode a common principle: the semantics of a variable should reflect those of its causes, causally related variables should be semantically related and unrelated ones distinct, and the semantics not explained by its causes should be specific to the variable itself. The generation term instantiates $g_v$ as $g_v(e_{\mathrm{pa}(v)}) = \sum_{u \in \mathrm{pa}(v)} T(e_u; \rho_{uv})$, where $T$ is a learned operator shared across all edges. $T$ scales the parent's vector by $|\rho_{uv}|$ after a learned correction. On a negatively correlated edge, it applies a learned semantic negation map first, so the parent contributes its opposite meaning (Appendix~\ref{apx:solver}). The marginal dependence term imposes only a lower bound, which leaves the vectors the freedom to fit the other constraints. The conditional dependence term matches $\delta(r_a, r_b)$ to the dependence that remains between $a$ and $b$ given their parents. Because the constraints admit multiple solutions, the regularizer $\mathcal{R}$ favors unit-norm vectors close to English phrases and to the named variables they depend on most. On questionnaires, the generation constraints also let latent variables carry semantics from named to unnamed items.

The operator $T$, the weight network that predicts per-variable weights for the terms of Eq.~\ref{eq:objective}, and the encoder fine-tuning are fitted on training scales and on generated arms and vehicles (Appendix~\ref{apx:solver}). The operator and the weight network are trained with names masked as at evaluation, minimizing the cosine distance between each optimized vector and the embedding of its true name. For a new system, we jointly optimize the vectors of the unnamed variables by gradient descent.

\begin{takeaway}
\textbf{Takeaway.} In the running example, Eq.~\ref{eq:realization}
turns the three relations of $x_t$ into three constraints on its vector.
The vector must be generated from the anchor of the vertical
acceleration, generate the anchor of the pitch rate, and be unrelated to
the anchor of the steering wheel angle. The more relations tie an unnamed
variable to the named ones, the more tightly its vector is constrained.
\end{takeaway}

\subsection{Translation via Prefix Adaptation}
\label{sec:method-translation}

We translate $\hat e_v$ into a name by passing it to a frozen language model as a prefix \citep{PrefixTuning}, so the model, which keeps its knowledge of words, attends to the whole vector and thus to every relation while it writes the name. Decoding $\hat e_v$ into words first would replace it by its few nearest phrases and drop its relations to every named variable, which the graph determined.

To build the prefix, we concatenate $\hat e_v$ with the vectors of the latent parent of $v$, when it has one, and a mean of the vectors of the named variables most correlated with $v$, into $z_v$. A learned adapter $q$ maps $z_v$ to the prefix $P_v = q(z_v) \in \mathbb{R}^{m \times d_p}$, where $m$ is the prefix length and $d_p$ the dimension of a prefix vector. Following \citet{PTuningV2}, learned matrices project the prefix into the key-value (KV) cache of every layer $j$ as $\mathbf{K}_j = P_v W^K_j$ and $\mathbf{V}_j = P_v W^V_j$. The prompt text $t_v$ carries the names around $v$. It lists the directed edges incident to $v$, and an evidence line gives the named variables most correlated with $v$ and the phrases nearest to $\hat e_v$.

We keep the language model frozen except for a low-rank update \citep{LoRA} and minimize
\begin{equation}
\mathcal{L}_{\mathrm{adapter}} =
-\log p_\theta\big(\ell(v) \mid P_v, t_v\big)
+ \lambda_{\mathrm{sim}} \big(1 - \cos(\bar\phi_{\mathrm{LM}}(\hat y),\, \bar\phi_{\mathrm{LM}}(\ell(v)))\big),
\label{eq:adapter-loss}
\end{equation}
where $\theta$ collects the adapter and the low-rank update and $\hat y$ is the predicted name. The first term supervises the words of the name, and the second its meaning under the mean token embeddings $\bar\phi_{\mathrm{LM}}$. Two further terms separate the name from those of the siblings of $v$ (Appendix~\ref{apx:plugin}). The adapter is trained on synthetic data only, and only the shape of the KV cache depends on the language model.

\begin{takeaway}
\textbf{Takeaway.} In the running example, the phrases nearest to the optimized vector of $x_t$ are heading angle and lateral acceleration. Both describe the motion of the car body, yet neither names $x_t$. Decoding the vector into words discards the relations that distinguish $x_t$ from the other channels. Reading the vector directly as a prefix, the frozen model writes longitudinal acceleration. The meaning of a name comes from the structure, and the language model supplies only the words.
\end{takeaway}

\textit{Generality of the framework.} The graph and the vectors are computed before the language model is called. Any model that takes vectors as input, such as a world model or a policy, can therefore receive the causal structure of its system in this form. The framework
applies to any system with measurements and a few named variables, under
the conditions of causal discovery (Appendix~\ref{apx:theory}).

\section{Experiments}
\label{sec:experiments}

RQ1 asks whether causal structure recovers the semantics of the variables a language prior cannot, and whether \ours{} can discover that structure itself. RQ2 to RQ4 ask where the gain comes from, whether one recipe carries across backbones and what naming a new system costs.

\paragraph{Data.}
\label{sec:environments}
We evaluate on two kinds of data that differ in how much text is available for naming, and their contrast tests the claim that causal structure supplies what a language prior does not cover. Robot and vehicle channels present the semantic gap in practice, since channel definitions vary across platforms and vehicle bus signals are often undocumented. Their names are short labels, which provide a language prior with little semantic information. We use six simulated arms \mbox{\citep{robosuite, robomimic}}, three logs of real arms \mbox{\citep{OpenX, ALOHA}} and ten drives of one car recorded from its CAN bus. The three scenarios differ in how much of the platform is documented, from a simulator with the true graph of its physics to a vehicle CAN with no documentation from the maker of the car. The graph derived from the physics or the documentation of a platform serves as the supplied graph. Questionnaire items are statements that respondents rate, so the text of the named items indicates what an unnamed neighbor measures. A latent variable causes the answers to several items, so the questionnaires also test the claim on variables that no one measured. We use five published scales, DASS \citep{DASS}, HEXACO \citep{HEXACO}, KIMS \citep{KIMS}, RIASEC and the World Values Survey \citep{Welzel2013, WVS7}. The supplied graph is the scoring key of a scale, which names every latent variable and lists its items. Every variable has a true name, so we mask 20 to 90\% of the names and score the name a method writes against the masked names. Besides generated arms and vehicles, the operator and the weights of the objective are fitted on sixteen further public scales and the adapter is trained on synthetic scales. No training set overlaps the evaluation data (Appendix~\ref{apx:data}).

\paragraph{Protocol.}
Masking hides 20, 50, 75 or 90\% of the names. Every method receives the same measurements, known names and unnamed variables. When no structure is supplied, \ours{} discovers the graph from the measurements and each baseline uses the task adaptation described in Appendix~\ref{apx:roster}. When a structure is supplied, every method that accepts a graph as input receives the same graph. \ours{} instantiates $\delta$ with the cosine similarity. With the conditional mutual information as $\delta$, the results remain essentially unchanged (Appendix~\ref{apx:cmi}). The name a method writes is scored as text in a frozen public embedding space \citep{E5} against every name of the same arm, log, drive or scale, so a copied known name scores zero. NRR@1 is the rate at which the true name is ranked first among the candidates, and MRR is the mean reciprocal rank of the true name \citep{MRR}. Identification is the rate at which a judge selects the target from five candidates. Naming quality, reported for latent variables, requires the output to be a name and rejects a copied item text. Under the discovered graph a discovered latent variable is scored against the published latent variable whose key contains most of its items (Appendix~\ref{apx:latent-ref}). Appendix~\ref{apx:protocol} gives the judge, the pools and the auxiliary metrics.

\paragraph{Baselines.}
Every baseline receives the variable's data and the known names.
CLIP-Dissect-e5 selects a concept from a fixed bank. Automated
Interpretability (AutoInterp), Delphi and SASC-style prompt a language
model to describe the variable. DeViSE-style \citep{DeViSE}
regresses the data into the embedding space of the names. When a structure is supplied, the base language model is given the graph as text and GraphMAE reconstructs the masked node from its neighbors. Naming a latent variable requires a grouping of items. With no structure supplied, every questionnaire baseline therefore takes its item groups from a shared factor analysis (Appendix~\ref{apx:roster}).


\begin{takeaway}
\textbf{RQ1: Does causal structure recover the semantics of the variables a language prior cannot, and can \ours{} discover that structure from the measurements themselves?}\\
\ours{} gives the most accurate names on every set of variables, including the latent variables that no one measured, and it needs no supplied structure to do so.
\end{takeaway}
\textit{The gain of \ours{} is largest on the robotics tasks, where a language prior provides the least information.} Without structure, a baseline relies on this prior. A language model carries broad knowledge of psychological scales, so the text of one item already indicates which latent variable it measures. A channel label provides no comparable information.
On the three robot columns, the baselines therefore remain low, while our method derives the meaning of each channel from its relations to the named channels. The lead narrows at 90\% masking, where
few anchors remain.

\textit{\ours{} also names the latent variables more accurately than every baseline at every masking level.} The baselines name a latent variable from its items, so their evidence shrinks as the items lose their names. Our method relies on the causal relations between a latent variable and its items, and these relations are estimated from the measurements and remain available under masking. On naming quality \ours{} also outperforms the baselines, since the name is written from the vector of the latent variable and its relations to every item.

\textit{\ours{} gains more over the baselines when no structure is supplied, since it discovers the causal graph itself while the baselines cannot} (Appendix~\ref{apx:j-graphs}). The latent variables show this pattern most clearly, since our method discovers each of them from the measurements together with the items it causes. In Figure~\ref{fig:main} (bottom), \ours{} on its discovered graph closely tracks its run on the supplied graph and outperforms it on the vehicle CAN. On the latent variables, as on the vehicle CAN, its lead is wider than under the supplied graph at every masking level. On every set of observed variables, our method on its own graph reaches an NRR@1 at least as high as any baseline given the supplied graph. A method that waits for a supplied structure can know only what someone has already written down, with the purpose and the bias of whoever wrote it. The graph \ours{} discovers reflects the mechanism that produced the measurements, so a system can come to know its world from its own data alone.

\begin{figure}[!h]\centering
\includegraphics[width=\textwidth,trim=0 4.5 0 0,clip]{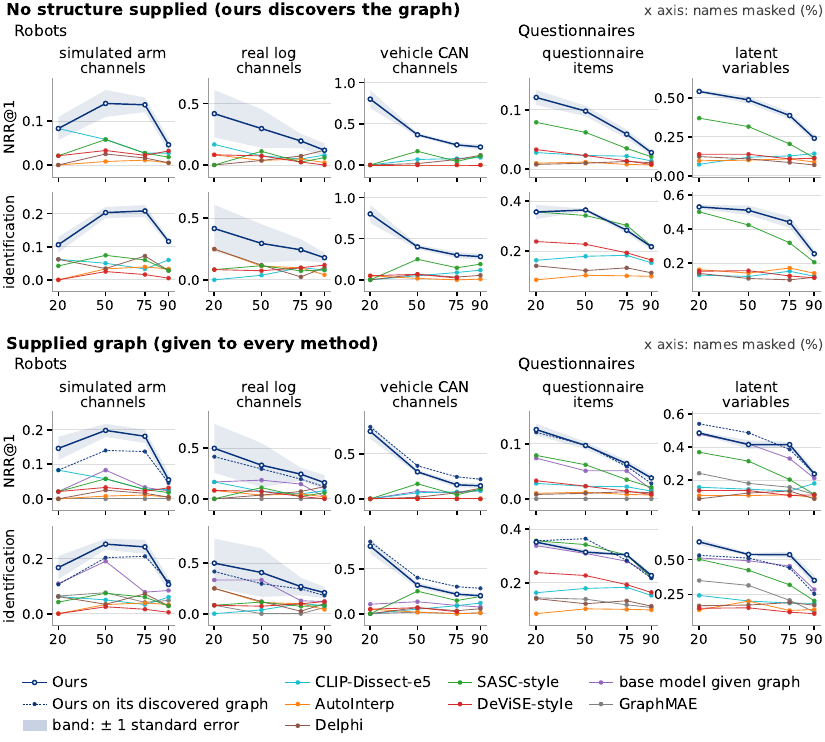}
\caption{\ours{} against every baseline on the five sets of variables. Top,
no structure is supplied, so our method runs on the graph it discovers. Bottom, the supplied graph is given to every method that accepts a graph as input. The thin dashed line with filled markers is our method on its discovered graph, whose latent variables differ from those of the supplied graph. Rows
are NRR@1 and identification. Bands are one standard error over folds or
over arms, logs and drives.}
\label{fig:main}
\end{figure}

\begin{takeaway}
\textbf{RQ2: Where does the gain come from?}\\
\ours{} gains most from the prefix, which carries the structure into the language model, while the evidence line supplies the vocabulary of the name (Figure~\ref{fig:comp}).
\end{takeaway}
Removing the prefix causes the largest loss. The prefix passes the vector $\hat e_v$ to every attention layer through the KV cache, so the model conditions the name on the relations of the variable to all named variables. Without the prefix, the model receives $\hat e_v$ only as the phrases nearest to it, which preserve the meaning of $\hat e_v$ but lose these relations. The evidence line lists the named variables most correlated with the unnamed one in the measurements, so its value depends on how many names remain. On the simulated arms and the real logs, removing the evidence line costs more than removing the prefix through 50\% masking, and the order reverses from 75\% on. The two inputs also resolve different parts of a channel name, which consists of a type and a coordinate (Appendix~\ref{apx:protocol}). The prefix determines the type on every robot column, and the evidence line resolves the coordinate on the vehicle CAN and the real logs. On latent variables, the prefix contributes most under the discovered graph, since a discovered latent variable exists only as a vector, whereas a latent variable of a supplied key can enter the text as its list of items (Appendix~\ref{apx:j-components}). The low-rank update adds little on its own. Its weights are shared across all variables and therefore cannot represent variable-specific structural relations, which enter the model through the prefix and the prompt.

\begin{takeaway}
\textbf{RQ3: Does one recipe carry across backbones?}\\
\ours{} generalizes across backbones with one adapter recipe, since the graph and the vectors are computed before the language model is called and only the adapter is trained per backbone.
\end{takeaway}
The same recipe trains an adapter for the 3B model \citep{Llama32} and the 4B and 8B models \citep{Qwen3}. The causal structure and the embedding do not depend on the backbone, and only the shape of the KV cache enters the adapter. Across the backbones, \ours{} yields the largest improvements on the channels, whose short
labels provide the backbone alone with little semantic information. The
improvements are smaller on latent variables and small on items, where
the item text already provides strong semantic cues. At 90\% masking the
three backbones converge under both the supplied and the discovered graph.
As the names become scarce, the structure determines
performance, and the differences among the backbones shrink (Figure~\ref{fig:comp}, Appendix~\ref{apx:j-backbones}).

\begin{takeaway}
\textbf{RQ4: What does naming a new system cost?}\\
\ours{} uses fewer tokens than any sampling method and names a new system in minutes, since it writes each name from one prompt, whereas a sampling method describes many examples.
\end{takeaway}
Our method spends 2.6 to 56.6 times fewer tokens on the same latent variable than the sampling methods and names the latent variables more accurately (Figure~\ref{fig:deploy}). A sampling method draws and describes examples of each variable, whereas our method writes each name from a single prompt. The three steps name ten drives from their raw traces in about seven minutes (Appendix~\ref{apx:cost}). The cost of a sampling method grows with the text that describes a variable, from the vehicle CAN to the questionnaires, whereas ours stays within a few hundred tokens on every column.

\section{Conclusion and Discussion}
\label{sec:conclusion}

We asked how causality bridges the semantic gap. Our answer is that, once a few variables are named, the meaning of another can be read from its causal relations, namely what causes it, what it causes and what it is independent of. These relations come from the mechanism behind the measurements and are recovered before any name is read, whereas association reaches a name through human knowledge and inherits its biases. The lead of \ours{} is largest where the least has been written down, and the discovered graph names variables as accurately as the documented one.

\looseness=-1
The main limitation of this work is its scope. We evaluated \ours{} only on systems whose true names are known but masked, so that the names it writes can be scored. We have not yet applied it to the settings that motivate it, such as the features of neural networks, latent variables in scientific data and recursively self-improving systems, which we discuss below.

The same semantic gap appears wherever a system produces measurements faster than its variables are named, and it includes every latent variable, which causal discovery returns without a name. Some of these systems were built by people who left them undocumented, such as the robots of a shared corpus \citep{OpenX}, the CAN bus of a vehicle or a system under audit. A language prior knows only what such systems usually record and can assign one name to several channels, whereas their causal relations distinguish them. Other variables have no name yet, such as a feature that a sparse autoencoder recovers from a neural network \citep{Monosemanticity, SAECunningham}, the states behind a population of neurons or the modules of a gene regulatory network. Where no text describes such a variable, a language prior has nothing to read, but its causal relations remain in the measurements, and the framework turns the variable into a named hypothesis that can be tested.

A system that collects its own data meets the semantic gap again with each new variable, and a closed loop forms: the system measures, recovers the structure of what it measures and names the variables in that structure, and each recovered name can then serve as an anchor for the next. Repeating the loop expands the system's model of the world and its vocabulary from its own data. RSI \mbox{(\textbf{R}ecursively \textbf{S}elf-\textbf{I}mproving)} systems require this capability on their perceptual side. The loop is also a step from measuring the world to knowing it, and from signals to concepts that agree with the underlying causes. Causal relations are written in the data even when semantics are not.

\subsection*{Acknowledgments}

We appreciate the helpful discussion with Isaac Ang.

\subsection*{AI use statement}

In this work, we used generative AI tools to generate a synthetic data
set and to aid and polish the writing. A local open-weight language
model (Qwen3-8B) wrote the latent variable names and the item texts of
the synthetic questionnaire scales on which the adapter is trained
(Appendix~\ref{apx:data}), and a leakage filter removes every generated
text that is close to a real one. A language model aided and polished
the writing of this manuscript. We did not use generative AI tools to
develop the conceptual framework or theoretical models, to formulate
mathematical claims, to provide ingredients for proving them or to write
proofs, to propose or refine the hypotheses, to design or provide
feedback on the methodology or the experiments, to implement methods, or
to interpret results, which are the authors' own. Translation, dataset
cleaning and qualitative data analysis do not apply to this work. We
have reviewed all AI-assisted work. Audit scripts recompute every number
in the paper from the stored records, and the authors revised every
AI-assisted passage. The language models that serve as a component of
the method, as baselines and as the identification judge are described
in Sections~\ref{sec:method} and~\ref{sec:experiments} and
Appendices~\ref{apx:plugin} to~\ref{apx:roster}. We take
responsibility for the final content of this work, including text,
claims and artifacts produced with the aid of generative AI.

\subsection*{Ethics statement}

This work uses public questionnaire response data and public robot and
vehicle telemetry. The questionnaire responses are anonymous and were released
for research use. We do not attempt to re-identify any respondent, and we
report results only at the level of items and latent variables. Naming variables
in a system that a person did not document can make that system easier to
audit, and it can also make an undocumented system easier to reverse
engineer without the owner's consent. We consider the audit case the
motivating one and note the other.

\subsection*{Reproducibility statement}

Appendix~\ref{apx:data} gives the data, their preprocessing, the masking
procedure and the training data of every trained part.
Appendix~\ref{apx:protocol} fixes the evaluation, including the scoring
encoder, the candidate pools and the judge protocol.
Appendix~\ref{apx:discovery} gives the discovery algorithms and their
tests, Appendix~\ref{apx:solver} the objective and the trained parts of
the optimization, and Appendix~\ref{apx:plugin} the adapter architecture
and its training objective. Appendix~\ref{apx:compute} lists the settings
of every stage. We will open-source the framework and the related data in the future. Every number in the tables and figures is computed from
the stored records of the runs.

\bibliography{iclr2027_conference}
\bibliographystyle{iclr2027_conference}

\renewcommand{\topfraction}{0.92}
\renewcommand{\bottomfraction}{0.92}
\renewcommand{\textfraction}{0.06}
\renewcommand{\floatpagefraction}{0.80}
\setcounter{topnumber}{4}
\setcounter{bottomnumber}{4}
\setcounter{totalnumber}{8}
\appendix
\clearpage

\section*{Appendix contents}

\begin{itemize}\setlength{\itemsep}{1pt}
\item \textbf{\ref{apx:notation}} Notation (p.~\pageref{apx:notation})
\item \textbf{\ref{apx:cases}} Case studies (p.~\pageref{apx:cases})
\item \textbf{\ref{apx:theory}} Theoretical background (p.~\pageref{apx:theory})
\item \textbf{\ref{apx:data}} Datasets and preprocessing (p.~\pageref{apx:data})
\item \textbf{\ref{apx:discovery}} Causal structure learning (p.~\pageref{apx:discovery})
\item \textbf{\ref{apx:solver}} Constrained embedding optimization (p.~\pageref{apx:solver})
\item \textbf{\ref{apx:plugin}} Translation via prefix adaptation (p.~\pageref{apx:plugin})
\item \textbf{\ref{apx:protocol}} Evaluation protocol (p.~\pageref{apx:protocol})
\item \textbf{\ref{apx:roster}} Baselines and fidelity to the published designs (p.~\pageref{apx:roster})
\item \textbf{\ref{apx:tables}} Supplemental experimental results (p.~\pageref{apx:tables})
\item \textbf{\ref{apx:cmi}} Results with conditional mutual information as the dependence measure (p.~\pageref{apx:cmi})
\item \textbf{\ref{apx:compute}} Settings, compute and licenses (p.~\pageref{apx:compute})
\item \textbf{\ref{apx:struct-merged}} Recovered structures (p.~\pageref{apx:struct-merged})
\end{itemize}

\section{Notation}
\label{apx:notation}

\begingroup\footnotesize\renewcommand{\arraystretch}{1.25}\setlength{\tabcolsep}{3pt}
\begin{longtable}{@{}>{\raggedright\arraybackslash}p{0.20\textwidth}>{\raggedright\arraybackslash}p{0.78\textwidth}@{}}
\caption{Summary of notation, in the order of the pipeline.}\label{tab:notation}\\
\toprule
symbol & meaning \\
\midrule
\endfirsthead
\toprule
symbol & meaning \\
\midrule
\endhead
\bottomrule
\endfoot
\bottomrule
\endlastfoot
\multicolumn{2}{@{}l}{\textit{Setting and causal structure}} \\
$V = O \cup L$ & the variables, observed and latent \\
$f_i$, $\varepsilon_i$ & the causal mechanism of $V_i$ in Eq.~\ref{eq:setting}, and its noise term \\
$X$, $n$ & the measurement matrix over the observed variables, and the number of samples \\
$\ell$ & the true naming map, from a variable to its text \\
$\mathcal{T}$, $\hat\ell$ & the space of text, and the name written for an unnamed variable \\
$V_{\mathrm{N}}, V_{\mathrm{U}}$ & the named variables, and the unnamed variables to be named \\
$G$, $\hat G$ & the causal graph, and the discovered graph \\
$\mathrm{pa}(v)$ & the parents of $v$ in the graph \\
$a \indep b \mid S$ & $a$ and $b$ are independent given $S$ \\
$(A \indep B \mid S)_G$ & $S$ d-separates $A$ from $B$ in $G$ \\
$\Sigma_{A,B}$, $\Sigma_{ab \cdot S}$ & the covariance between $A$ and $B$, and of $a$ and $b$ given $S$ \\
$\operatorname{pcor}(a, b \mid S)$ & the partial correlation of Eq.~\ref{eq:pcor} \\
\midrule
\multicolumn{2}{@{}l}{\textit{Constrained embedding optimization}} \\
$\phi$, $d$ & the frozen text encoder, and the dimension of its embedding space \\
$e_v$ & the vector assigned to variable $v$ \\
$e^\star_v = \phi(\ell(v))$ & the vector of the name of $v$, an anchor when $v$ is named \\
$\hat e_v$ & the optimized vector of an unnamed variable \\
$\mathcal{A}(G, X)$ & the feasible set of assignments that satisfy Eq.~\ref{eq:realization} \\
$g_v$ & the generating function of the vector of $v$ from the vectors of its parents \\
$r_v$ & the residual $e_v - g_v(e_{\mathrm{pa}(v)})$ \\
$T$ & the operator shared across edges \\
$b_{uv}, \Delta, \eta_{uv}$ & the base map of the operator on an edge, its trained correction and the edge type \\
$f_{\mathrm{neg}}$ & the frozen semantic negation map \\
$e_{v \cdot S}$ & the vector of $v$ given $S$, $r_v$ for the parents and the projection residual in Appendix~\ref{apx:cmi-vectors} \\
$\rho_{ab}$ & the distance correlation of $a$ and $b$, signed by their Pearson correlation \\
$\tilde\rho_{ab}$ & $\rho_{ab}$ set to zero below the significance threshold \\
$\rho_{ab \cdot S}$ & the dependence of $a$ and $b$ given $S$, a distance correlation in the main method \\
$\delta$ & a function of two vectors for the dependence, the cosine similarity in the main method \\
$\psi$ & the map from a dependence to the scale of $\delta$, the identity in the main method \\
$\mathcal{P}_{\mathrm{ind}}$ & pairs d-separated by the empty set \\
$\mathcal{P}_{\mathrm{dep}}$ & the trek-connected pairs of strongest dependence \\
$\mathcal{P}_{\mathrm{con}}$ & pairs of variables that have parents \\
$\lambda_{\mathrm{ind}}, \lambda_{\mathrm{dep}}, \lambda_{\mathrm{con}}, \kappa$ & the weights of the relation terms, and the fraction in the dependence bound \\
$\Pi$ & the soft nearest-neighbor projection onto the concept bank \\
$c_v$ & the neighbor mean of an unnamed observed variable \\
$\mathcal{R}$ & the regularizer of Eq.~\ref{eq:reg} \\
$\mathcal{L}$, $\mathcal{L}_0$ & the objective of Eq.~\ref{eq:objective}, and the objective without the regularizer \\
$\hat{\mathcal{A}}(G, X)$ & the estimated feasible set of Definition~\ref{def:relations} \\
$\omega$, $\mathcal{S}$ & the weight network on the terms of Eq.~\ref{eq:objective}, and the sign-consistency penalty \\
\midrule
\multicolumn{2}{@{}l}{\textit{Translation via prefix adaptation}} \\
$z_v, s_v$ & the adapter input of a variable, and its neighborhood scalars \\
$x_1, \ldots, x_5$ & the five input vectors of the adapter \\
$P_v, m, d_p$ & the prefix of a variable, its length and the dimension of a prefix vector \\
$q$, $q_i$, $q_\gamma$, $\gamma$ & the adapter, its five networks, its gate network and the gates \\
$W^K_j, W^V_j$, $\mathbf{K}_j, \mathbf{V}_j$ & the key and value matrices of layer $j$, and the keys and values they produce \\
$t_v$ & the prompt text around $v$ \\
$\phi_{\mathrm{LM}}$, $\bar\phi_{\mathrm{LM}}$ & the token embeddings of the language model, and their mean over a text \\
$\hat y$, $\theta$ & the written name, and the parameters of the adapter and the low-rank update \\
$\mathcal{L}_{\mathrm{NLL}}, \mathcal{L}_{\mathrm{ctr}}, \mathcal{L}_{\mathrm{sib}}$ & the likelihood, contrastive and sibling terms of Eq.~\ref{eq:plugin-loss} \\
$\lambda_{\mathrm{sim}}, \lambda_{\mathrm{ctr}}, \lambda_{\mathrm{sib}}$ & the weights of the terms of Eq.~\ref{eq:plugin-loss} \\
\midrule
\multicolumn{2}{@{}l}{\textit{Conditional mutual information (Appendices~\ref{apx:cmi-vectors} and~\ref{apx:cmi-objective})}} \\
$Z$, $X^e_v$ & a standard Gaussian vector, and the Gaussian model of an assignment (Definition~\ref{def:gauss-model}) \\
$I(a; b \mid S)$, $\hat I$ & the conditional mutual information, and its estimate by CCMI \\
$\tau_{ab \cdot S}$, $\xi_{ab \cdot S}$ & the tested and signed $\psi$ of $\hat I$ of Eq.~\ref{eq:cmi-target}, and its sign \\
$\beta_{|S|}$ & the threshold of the test for a conditioning set of size $|S|$ \\
$S_{ab}$ & the conditioning set $(\mathrm{pa}(a) \cup \mathrm{pa}(b)) \setminus \{a, b\}$ of the conditional term \\
$w$ & the Theiler window of serial rows \\
\end{longtable}
\endgroup

\section{Case Studies}
\label{apx:cases}
\makeatletter\setlength{\@fptop}{0pt plus 1fil}\setlength{\@fpsep}{16pt plus 2fil}\setlength{\@fpbot}{0pt plus 1fil}\makeatother

Figure~\ref{fig:examples} compares the names that three baselines and our method wrote for the same variables. The variables are three questionnaire latent variables and three unnamed channels of one drive. Table~\ref{tab:case-worked} then shows one latent variable of DASS in full. The four figures that follow show our own outputs together with the context the model received. They cover two questionnaires, one vehicle drive and one real robot log.

\begin{figure}[!ht]\centering
\includegraphics[width=\textwidth]{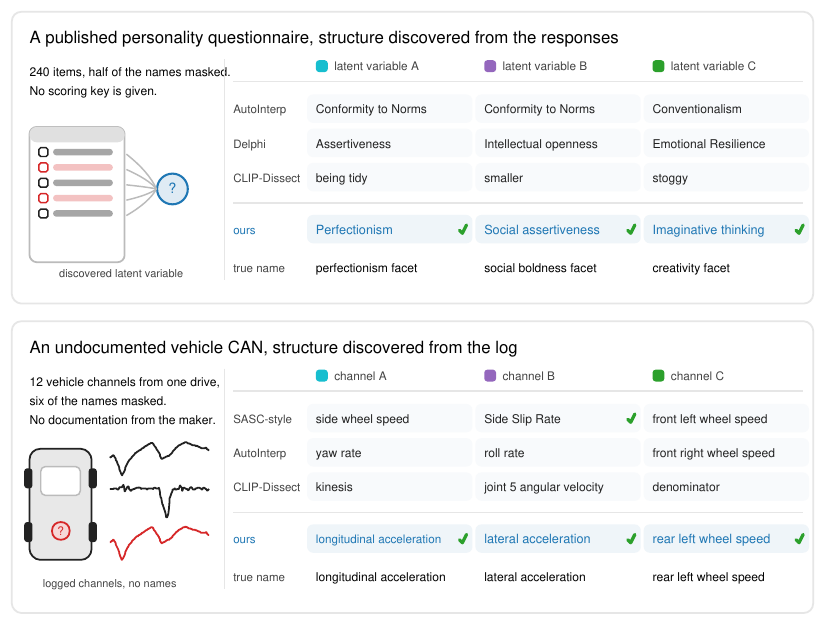}
\caption{A questionnaire with three unnamed latent variables at 50\%
masking, and one drive with three of its unnamed
channels at 50\%. \textbf{Left}, the data. \textbf{Right}, what existing
methods (above) and ours (below) wrote, with the true name underneath. A
tick marks an answer the identification judge accepted.}
\label{fig:examples}
\end{figure}

Table~\ref{tab:case-worked} shows one latent variable at 50\% masking.
The first block holds one neighbor line per item the latent variable
causes, with \texttt{(label hidden)} where the item is unnamed. The second block is the evidence line, which lists the named items that the latent variable causes. It also gives the correlations of the factor score of the latent variable with the factor scores of the other latent variables of the scale. The last two lines are the name \ours{} wrote and
the true name.

\begin{table}[!ht]\centering\fontsize{9.5}{11.5}\selectfont\renewcommand{\arraystretch}{1.1}
\caption{One latent variable end to end at 50\% masking, printed with the context the model received. It causes 4 items, 2 of them named and 2 unnamed.}
\label{tab:case-worked}
\begin{tabular}{@{}p{\textwidth}@{}}
\toprule
\textbf{Causal neighbors given to the model}\\
\hspace*{1em}- EFFECT: it causes the item "I experienced breathing difficulty (eg, excessively rapid breathing, breathlessness in the absence of physical exertion)."\\
\hspace*{1em}- EFFECT: it causes the item "(label hidden)"\\
\hspace*{1em}- EFFECT: it causes the item "(label hidden)"\\
\hspace*{1em}- EFFECT: it causes the item "I experienced trembling (eg, in the hands)."\\[2pt]
\textbf{Structural evidence given to the model}\\
\hspace*{1em}\begin{minipage}{0.96\textwidth}a factor with 4 items (2 hidden), visible items: "I experienced breathing difficulty (eg, excessively rapid breathing, breathlessness in the absence of physical exertion).", "I experienced trembling (eg, in the hands).". its factor score correlates +0.73 with the factor of "I found myself in situations that made me so anxious I was most relieved when they ended." / "I had a feeling of faintness.", +0.68 with the factor of "I found it difficult to relax." / "I found it hard to wind down.", +0.52 with the factor of "I felt down-hearted and blue." / "I was unable to become enthusiastic about anything.". cue phrases decoded from its data: feeling nervous, nervousness. in the constrained embedding its vector lies closest to "I experienced trembling (eg, in the hands)." (cos 0.94), "I felt that I was using a lot of nervous energy." (cos 0.92). its hidden items read from the data as: feeling nervous, trembling; feeling nervous, nervousness.\end{minipage}\\[2pt]
\midrule
\textbf{\ours{} writes} \quad Anxiety symptoms\\
\textbf{True name} \quad the DASS clinical scale: anxiety\\
\bottomrule\end{tabular}\end{table}

\begin{figure}[!ht]\centering
\makebox[\textwidth][c]{\includegraphics[width=1.08\textwidth]{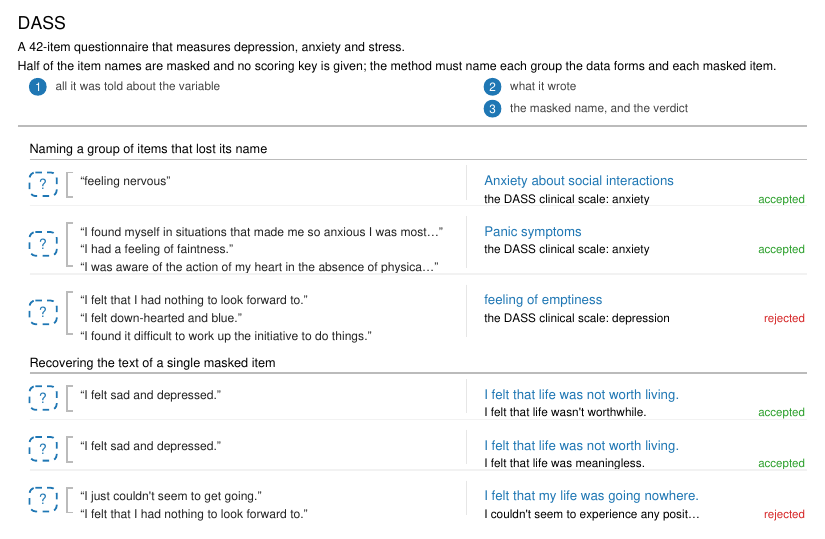}}
\caption{Cases from DASS at 50\% masking. The upper block covers discovered latent variables and the lower block covers masked items. Each row shows
what the model received, the name it wrote, and the reference name with
the judge's verdict.}
\label{fig:case-dass}
\end{figure}

\begin{figure}[!ht]\centering
\makebox[\textwidth][c]{\includegraphics[width=1.08\textwidth]{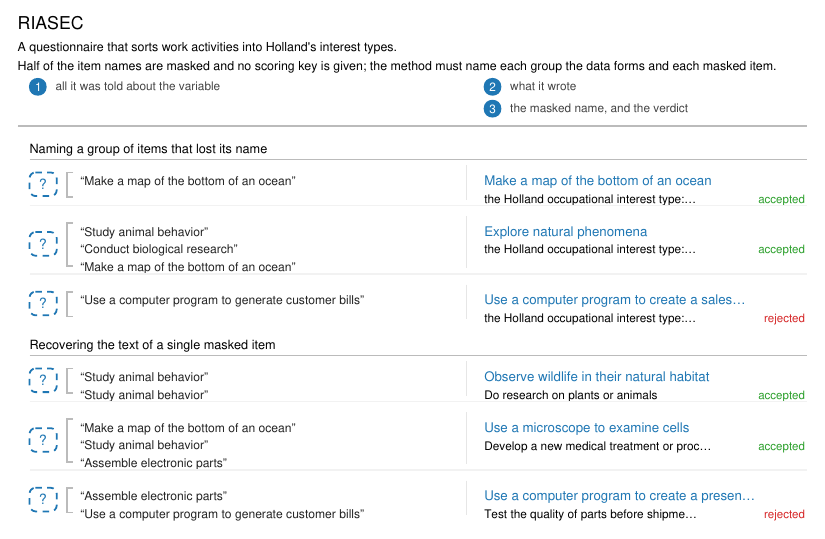}}
\caption{Cases from RIASEC at 50\% masking, in the layout of
Figure~\ref{fig:case-dass}.}
\label{fig:case-riasec}
\end{figure}

\begin{figure}[!ht]\centering
\includegraphics[width=\textwidth]{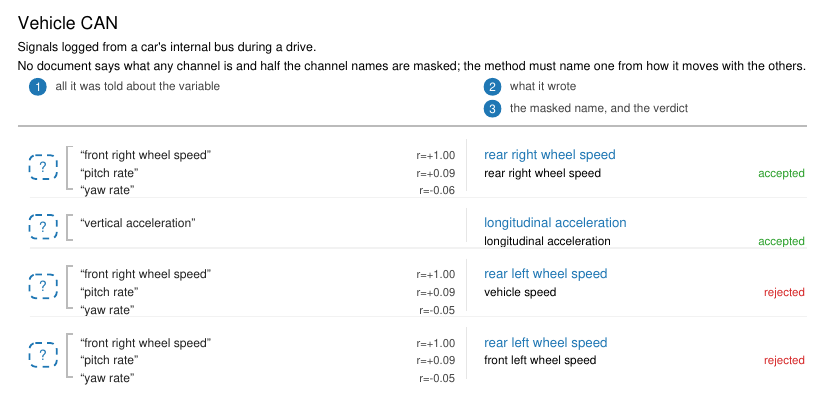}
\caption{Cases from one vehicle drive at 50\% masking. Each row shows the named neighbors of the unnamed channel in the graph, the name written, and the true name with the judge's verdict. For a channel without named neighbors, the row shows the channels it moves with.}
\label{fig:case-can}
\end{figure}

\begin{figure}[!ht]\centering
\includegraphics[width=\textwidth]{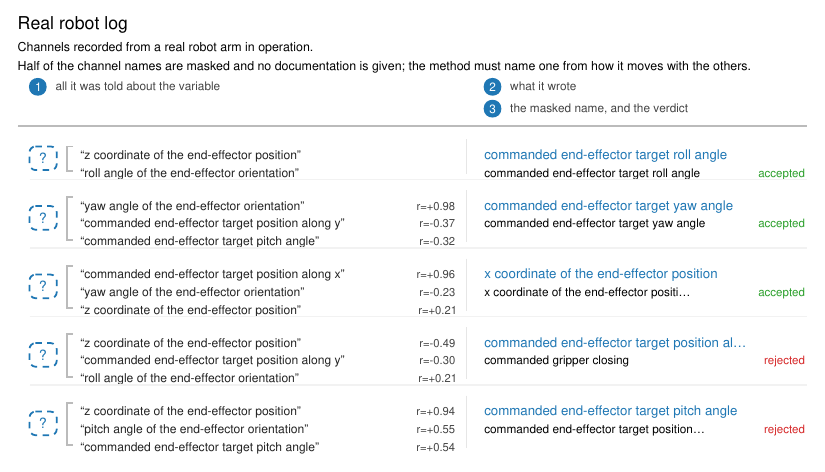}
\caption{Cases from a real robot log at 50\% masking, in the layout of
Figure~\ref{fig:case-can}.}
\label{fig:case-real}
\end{figure}

Figure~\ref{fig:case-can} shows a failure mode that the aggregate numbers
hide. On a straight drive, the speed channels have a pairwise Pearson correlation of $+1.00$, so nothing in the data separates them. Two of the three unnamed speed channels therefore receive the same name. One is then counted as an
identification and two as misses.

\section{Theoretical Background}
\label{apx:theory}
\makeatletter\setlength{\@fptop}{0pt plus 1fil}\setlength{\@fpsep}{8pt plus 2fil}\setlength{\@fpbot}{0pt plus 1fil}\makeatother

Every cited statement below is restated from its source in the notation of
Section~\ref{sec:formulation}, and its proof is in the cited work.
Definitions~\ref{def:relations} and~\ref{def:gauss-model} and Propositions~\ref{prop:feasible} and~\ref{prop:cmi} are ours. Let $G=(V,E)$ be a directed acyclic graph and $P$ a distribution over $V$ with covariance matrix $\Sigma$, and let $\Sigma_{A,B}$ denote the submatrix of $\Sigma$ with rows $A$ and columns $B$. For disjoint $A, B, S \subseteq V$ we write
$(A \indep B \mid S)_G$ when $S$ d-separates $A$ from $B$ in $G$, and
$(A \indep B \mid S)_P$ when $A$ and $B$ are independent given $S$ in $P$.

\subsection{The graph and the distribution}
\label{apx:theory-dist}

\begin{definition}[Markov and faithfulness; \citealp{Pearl2009, SGS2000}]
\label{def:markov}
$P$ is Markov relative to $G$ if $(A \indep B \mid S)_G \Rightarrow (A \indep B \mid S)_P$ for all disjoint $A, B, S \subseteq V$, and faithful to $G$ if the converse holds for all such triples. Under both,
\begin{equation}
\{(A,B,S): (A \indep B \mid S)_P\} \;=\; \{(A,B,S): (A \indep B \mid S)_G\}.
\label{eq:markov}
\end{equation}
\end{definition}

\begin{proposition}[Gaussian case; \citealp{Sullivant2010}, Prop.~2.2 and the remark after it]
\label{prop:gauss}
Let $P = \mathcal{N}(\mu, \Sigma)$ with $\Sigma$ positive definite, and let $A, B, S$ be disjoint. For two variables $a$ and $b$, the covariance given $S$ and the partial correlation given $S$ are
\begin{equation}
\Sigma_{ab \cdot S} = \Sigma_{a,b} - \Sigma_{a,S}\, \Sigma_{S,S}^{-1}\, \Sigma_{S,b},
\qquad
\operatorname{pcor}(a, b \mid S) = \frac{\Sigma_{ab \cdot S}}{\sqrt{\Sigma_{aa \cdot S}\; \Sigma_{bb \cdot S}}} ,
\label{eq:pcor}
\end{equation}
with $\Sigma_{ab \cdot \emptyset} = \Sigma_{a,b}$ for an empty $S$. Then
\begin{equation}
(A \indep B \mid S)_P
\;\Longleftrightarrow\;
\operatorname{rank} \Sigma_{A \cup S,\, B \cup S} = |S|
\;\Longleftrightarrow\;
\operatorname{pcor}(a, b \mid S) = 0 \;\; \forall a \in A,\, b \in B .
\end{equation}
\end{proposition}

\emph{Interpretation.} Under the Markov and faithfulness assumptions, Eq.~\ref{eq:markov} turns every d-separation of $G$ into an independence of $P$ and every d-connection into a dependence. The graph therefore determines which pairs of variables are independent given a set $S$ and which are dependent. In the Gaussian case each of these statements is a condition on $\Sigma$ alone, namely whether the partial correlation of Eq.~\ref{eq:pcor} vanishes. The objective imposes them for two conditioning sets, the empty set and the parents (Appendix~\ref{apx:theory-objective}).

\subsection{Latent variables and rank}
\label{apx:theory-rank}

\begin{definition}[Trek and t-separation; \citealp{Sullivant2010}, Defs.~2.6 and~2.7]
\label{def:trek}
A trek from $u$ to $v$ is a pair of directed paths $(P_1, P_2)$ with sinks
$u$ and $v$ and a common source, either path possibly a single vertex. For
subsets $A, B, C_A, C_B$ of $V$, not necessarily disjoint, $(C_A, C_B)$
t-separates $A$ from $B$ if every trek from $A$ to $B$ contains a vertex of
$C_A$ on $P_1$ or a vertex of $C_B$ on $P_2$.
\end{definition}

\begin{theorem}[Trek separation; \citealp{Sullivant2010}, Thm.~2.8]
\label{thm:trek}
For every covariance matrix of a linear structural equation model with
graph $G$ and mutually independent noise terms,
\begin{equation}
\operatorname{rank} \Sigma_{A,B} \;\le\; \min \big\{ |C_A| + |C_B| : (C_A, C_B) \text{ t-separates } A \text{ from } B \big\},
\label{eq:trek}
\end{equation}
with equality for generic parameters.
\end{theorem}

\emph{Interpretation.} For generic parameters Eq.~\ref{eq:trek} holds with equality, so $\operatorname{rank} \Sigma_{A,B} = k < \min(|A|, |B|)$ exactly when the smallest pair $(C_A, C_B)$ that t-separates $A$ from $B$ has $|C_A| + |C_B| = k$. Such a pair may contain latent variables. A rank deficiency among observed variables can therefore reveal latent variables through which every trek between them passes, and Appendix~\ref{apx:theory-discovery} states when RLCD identifies them. The rank tests of Appendix~\ref{apx:discovery} search for such deficiencies at every cardinality up to a maximum.

\subsection{What structure learning returns}
\label{apx:theory-discovery}

\paragraph{Identifiability of the recovered latent structure.}

\begin{definition}[Latent linear causal model; \citealp{Dong2024}, Def.~1 and Asm.~1]
\label{def:llcm}
$V = L \cup O$ with latent variables $L$ and observed variables $O$, and
\begin{equation}
V_i \;=\; \sum_{V_j \in \mathrm{pa}(V_i)} a_{ij} V_j + \varepsilon_i,
\end{equation}
where $a_{ij}$ is the causal coefficient from $V_j$ to $V_i$ and $\varepsilon_i$ the noise term. The noise terms are mutually independent, as in the linear model of Theorem~\ref{thm:trek}. $P$ is rank faithful to $G$ if Eq.~\ref{eq:trek} holds with equality for all $A, B \subseteq V$.
\end{definition}

\begin{definition}[Pure children and atomic cover; \citealp{Dong2024}, Defs.~4 and~5]
\label{def:cover}
The pure children of a set $\mathbf{V} \subseteq V$ are
$\mathrm{PCh}(\mathbf{V}) = \{v : \mathrm{pa}(v) \neq \emptyset,\ \mathrm{pa}(v) \subseteq \mathbf{V}\}$. For a set $\mathcal{C}$ of sets let
$\|\mathcal{C}\| = |\bigcup_{C \in \mathcal{C}} C|$. Let $l$ of the
variables of $\mathbf{V}$ be latent. $\mathbf{V}$ is an atomic cover if it is a single observed variable, or if the three conditions below hold. (i) There is a set $\mathcal{C}$ of atomic covers with $\bigcup_{C \in \mathcal{C}} C \subseteq \mathrm{PCh}(\mathbf{V})$ and $\|\mathcal{C}\| \ge l + 1$. (ii) There is a set $\mathcal{N}$ of covers with $(\bigcup_{N \in \mathcal{N}} N) \cap (\bigcup_{C \in \mathcal{C}} C) = \emptyset$, every element of $\bigcup_{N \in \mathcal{N}} N$ a neighbor of $\mathbf{V}$, $\mathcal{N}$ and $\mathcal{C}$ d-separated by $\mathbf{V}$, and $\|\mathcal{N}\| \ge l + 1$. (iii) $\mathbf{V}$ has no partition into two atomic covers.
\end{definition}

\begin{condition}[\citealp{Dong2024}, Conds.~1 and~2]
\label{cond:rlcd}
(a) Every latent variable belongs to at least one atomic cover, and for
each atomic cover that contains a latent variable, none of its children is
adjacent to any of its neighbors. (b) If a set $C$ consists of colliders
of two atomic covers $\mathbf{V}_1, \mathbf{V}_2$, $A$ is the minimal set
that d-separates $\mathbf{V}_1$ from $\mathbf{V}_2$, and a latent variable
lies in $\mathbf{V}_1 \cup \mathbf{V}_2 \cup C \cup A$, then
$|C| + |A| \ge |\mathbf{V}_1| + |\mathbf{V}_2|$.
\end{condition}

\begin{theorem}[Identifiability of RLCD; \citealp{Dong2024}, Thm.~12]
\label{thm:rlcd}
Let $(G, P)$ be a latent linear causal model of Definition~\ref{def:llcm}
with $P$ rank faithful to $G$, and let $G$ satisfy
Condition~\ref{cond:rlcd}. Then RLCD asymptotically identifies the Markov equivalence class of $\mathcal{O}_{\min}(\mathcal{O}_s(G))$. Here the skeleton operator $\mathcal{O}_s$ introduces edges that involve latent variables, and the minimal-graph operator $\mathcal{O}_{\min}$ absorbs redundant latent variables. The rank information over $O$ is invariant to both operators.
\end{theorem}

\emph{Interpretation.} By Theorem~\ref{thm:rlcd}, RLCD recovers $\mathcal{O}_{\min}(\mathcal{O}_s(G))$ up to its Markov equivalence class, and a latent variable in it is identified up to its atomic cover, which its pure children and its neighbors jointly determine (Definition~\ref{def:cover}). A discovered latent variable is therefore identified only through the items it reaches, which is also how the evaluation refers it to a published latent variable (Appendix~\ref{apx:latent-ref}).

\paragraph{Score-based search over observed variables.}

\begin{proposition}[Correctness of BOSS; \citealp{BOSS}, Prop.~2]
\label{prop:boss}
Let $P$ belong to a curved exponential family and be causal Markov and
causal faithful to $G$, and let $X$ be drawn i.i.d.\ from $P$. Then BOSS
with its backward phase returns the Markov equivalence class of $G$ for
every initial permutation as the sample size grows.
\end{proposition}

\emph{Interpretation.} On robot and vehicle channels, BOSS runs over the channels at two consecutive time steps (Appendix~\ref{apx:discovery}). Under its hypotheses, Proposition~\ref{prop:boss} guarantees the skeleton and the unshielded colliders of the graph over these columns. A cause precedes its effect, so an edge between the two steps runs from the earlier step, and the discovery step keeps the edges that run forward in time. On questionnaires, BOSS takes the place of the first stage of RLCD. It supplies an initial graph over the items and partitions them into groups, within which the rank tests of RLCD then run. The discovered graph $\hat G$ is the input of the constrained embedding optimization.

\subsection{From the graph to the objective}
\label{apx:theory-objective}
Eq.~\ref{eq:objective}, with $\delta$ the cosine similarity, imposes the lines of Eq.~\ref{eq:realization} for two conditioning sets and takes $\psi$ as the identity. With $S$ empty, $e_{v \cdot S} = e_v$. The independence line then applies to $\mathcal{P}_{\mathrm{ind}}$, the pairs d-separated by the empty set, which include the causes of a common effect. The dependence line applies to $\mathcal{P}_{\mathrm{dep}}$, the most strongly dependent trek-connected pairs, as the one-sided bound $|\cos(e_a, e_b)| \ge \kappa\,|\rho_{ab}|$. With $S$ the parents, $e_{v \cdot S}$ is the residual $r_v$ of Eq.~\ref{eq:residual}, and the conditional line applies to $\mathcal{P}_{\mathrm{con}}$. The objective imposes no relation for other conditioning sets.

The dependence $\rho_{ab}$ is the distance correlation \citep{DistanceCorrelation} signed by the Pearson correlation, and it vanishes under independence without the Gaussian assumption. On $\mathcal{P}_{\mathrm{ind}}$ it enters as the thresholded $\tilde\rho_{ab}$ of Appendix~\ref{apx:solver}. The target $\rho_{ab \cdot \mathrm{pa}}$ is the dependence between the residuals of $a$ and $b$ after regression on their own parents. When the noise enters additively, this residual is the noise term of the variable, so under a correct graph with mutually independent noise terms the residuals of any two variables are independent \citep{Peters2014} and the target vanishes. Where the measurements show a residual dependence, the conditional term matches the residual cosine similarity to it.

\begin{definition}[The estimated feasible set]
\label{def:relations}
With $\mathcal{P}_{\mathrm{ind}}$, $\mathcal{P}_{\mathrm{dep}}$,
$\mathcal{P}_{\mathrm{con}}$, $r_v$ and $\kappa$ as in
Eq.~\ref{eq:objective} and $\tilde\rho_{ab}$ the thresholded distance
correlation of Appendix~\ref{apx:solver}, the estimated feasible set $\hat{\mathcal{A}}(G, X)$ consists of the assignments with $e_v = e^\star_v$ on $V_{\mathrm{N}}$ and
\begin{equation}
\begin{aligned}
r_v &= 0 \;\; \forall v : \mathrm{pa}(v) \neq \emptyset, &
\cos(e_a, e_b) &= \tilde\rho_{ab} \;\; \forall (a,b) \in \mathcal{P}_{\mathrm{ind}}, \\
\cos(r_a, r_b) &= \rho_{ab \cdot \mathrm{pa}} \;\; \forall (a,b) \in \mathcal{P}_{\mathrm{con}}, &
|\cos(e_a, e_b)| &\ge \kappa\,|\rho_{ab}| \;\; \forall (a,b) \in \mathcal{P}_{\mathrm{dep}}.
\end{aligned}
\label{eq:relations}
\end{equation}
\end{definition}

\begin{proposition}
\label{prop:feasible}
Let $\mathcal{L}$ be the objective of Eq.~\ref{eq:objective} with $\delta$ the cosine similarity, with $\rho_{ab}$ on $\mathcal{P}_{\mathrm{ind}}$ taken as $\tilde\rho_{ab}$ and with positive term weights. Let $\mathcal{R}$ be the regularizer of Eq.~\ref{eq:reg} and $\mathcal{L}_0 = \mathcal{L} - \mathcal{R}$. Here $\cos(r_a, r_b)$ is taken on the unit-normalized residuals and is defined as $0$ when a residual vanishes. For every assignment $e$ with
$e_v = e^\star_v$ on $V_{\mathrm{N}}$,
\begin{equation}
\mathcal{L}_0(e) = 0 \;\Longleftrightarrow\; e \in \hat{\mathcal{A}}(G, X).
\end{equation}
The set $\hat{\mathcal{A}}(G, X)$ is nonempty only if
$\rho_{ab \cdot \mathrm{pa}} = 0$ on every pair of $\mathcal{P}_{\mathrm{con}}$,
since $r_v = 0$ on every child sets every residual cosine similarity to
$0$. The set is also nonempty only if the marginal lines can hold together with the anchors.
\end{proposition}

\begin{proof}
Each term of $\mathcal{L}_0$ is a sum or a mean, over one line of Eq.~\ref{eq:relations}, of squares or of squared hinges that vanish exactly when the line holds, and its weight is positive.

($\Leftarrow$) Let $e \in \hat{\mathcal{A}}(G, X)$. Every line of Eq.~\ref{eq:relations} holds, so every term of $\mathcal{L}_0$ is zero and $\mathcal{L}_0(e) = 0$.

($\Rightarrow$) Let $\mathcal{L}_0(e) = 0$. Every term is nonnegative, so every term is zero, and every line of Eq.~\ref{eq:relations} holds. Anchoring holds by assumption, so $e \in \hat{\mathcal{A}}(G, X)$.

For the first necessary condition, the generation line requires $r_v = 0$ for every variable with parents. Every residual cosine similarity is then $0$ by the convention above, so the conditional line holds only if $\rho_{ab \cdot \mathrm{pa}} = 0$ on $\mathcal{P}_{\mathrm{con}}$. The second necessary condition restates that the marginal lines are part of Eq.~\ref{eq:relations}.
\end{proof}

\emph{Interpretation.} Each marginal cosine relation is an entry of the Gram matrix of the unit-normalized assignment, and each conditional cosine relation is an entry of the Gram matrix of the normalized residuals. The relation terms therefore make the second step a constrained multidimensional scaling of the matrix of dependences onto the unit sphere \citep{Torgerson1952, KTA}, with the anchors held fixed. Since $\mathcal{L} = \mathcal{R}$ on $\hat{\mathcal{A}}(G,X)$, a minimizer of $\mathcal{L}$ with $\mathcal{L}_0 = 0$ lies in $\hat{\mathcal{A}}(G,X)$ and minimizes $\mathcal{R}$ over it. On finite data $\mathcal{L}_0$ does not reach zero, and a minimizer of $\mathcal{L}$ balances the squared violations of the lines of Eq.~\ref{eq:relations} against $\mathcal{R}$. Which member of $\mathcal{A}(G,X)$ carries the semantics of a variable is left to existing knowledge, which enters as $\mathcal{R}$ and as the language model (Section~\ref{sec:formulation}).

\subsection{The conditional mutual information of the vectors}
\label{apx:cmi-vectors}

Eq.~\ref{eq:objective}, with $\delta$ the cosine similarity, takes $\psi$ as the identity and measures the dependence by the distance correlation. With the conditional mutual information (CMI) as the dependence measure, $\psi$ takes an exact form, which follows from a Gaussian model of the vectors.

\begin{definition}[Gaussian model of an assignment]
\label{def:gauss-model}
For an assignment $e : V \to \mathbb{R}^d$, let $Z \sim \mathcal{N}(0, I_d)$ and $X^e_v = \langle e_v, Z \rangle$ for every $v \in V$. For $S \subseteq V$, let $e_{v \cdot S}$ be the residual of $e_v$ after its orthogonal projection onto $\operatorname{span}\{e_s : s \in S\}$, with $e_{v \cdot \emptyset} = e_v$.
\end{definition}

\begin{proposition}[CMI of the vectors]
\label{prop:cmi}
Let $a, b \in V$ and $S \subseteq V \setminus \{a, b\}$. With the cosine similarity taken as $0$ when a residual vanishes,
\begin{equation}
I\big(X^e_a; X^e_b \mid X^e_S\big) = -\tfrac{1}{2} \log\big(1 - \cos^2(e_{a \cdot S}, e_{b \cdot S})\big),
\label{eq:cmi-vectors}
\end{equation}
which is $+\infty$ when $|\cos(e_{a \cdot S}, e_{b \cdot S})| = 1$. With $\psi(I) = \sqrt{1 - e^{-2I}}$, it follows that
\begin{equation}
\psi\big(I(X^e_a; X^e_b \mid X^e_S)\big) = \big|\cos(e_{a \cdot S}, e_{b \cdot S})\big| .
\label{eq:cmi-psi}
\end{equation}
When the vectors of $S$ are linearly independent and neither residual vanishes, $\cos(e_{a \cdot S}, e_{b \cdot S})$ is the partial correlation of Eq.~\ref{eq:pcor} for the distribution of $X^e$, whose covariance matrix $\Sigma$ is the Gram matrix of the vectors.
\end{proposition}

\emph{Interpretation.} The Gaussian model is a model of the vectors alone, and the measurements enter only through their estimated CMI (Appendix~\ref{apx:cmi-objective}). By Proposition~\ref{prop:cmi}, the absolute cosine similarity of two residuals given $S$ is the CMI of the vectors given $S$, on the scale of $\psi$. This $\psi$ is Linfoot's informational coefficient of correlation \citep{Linfoot1957}. It is defined for any joint distribution, equals the absolute correlation for a bivariate Gaussian and does not change under an invertible transformation of either variable. Take $\delta$ as the cosine similarity and the dependence $\rho_{ab \cdot S}$ as the CMI $I(a; b \mid S)$ of the measurements, with the sign of the dependence applied after $\psi$. The dependence line of Eq.~\ref{eq:realization} then sets the absolute cosine similarity to $\psi(I(a; b \mid S))$ and its sign to that sign. Since $\psi$ is strictly increasing, Eq.~\ref{eq:cmi-psi} turns the first condition into $I(X^e_a; X^e_b \mid X^e_S) = I(a; b \mid S)$. Equivalently, $\delta$ can be the CMI of the vectors itself, $\delta(e_{a \cdot S}, e_{b \cdot S}) = I(X^e_a; X^e_b \mid X^e_S)$, which by Eq.~\ref{eq:cmi-vectors} depends only on the two residuals. With this $\delta$ and $\psi$ the identity, the dependence line states the same equality, and the cosine form adds the sign. The independence line holds exactly when $X^e_a \indep X^e_b \mid X^e_S$, and $a \indep b \mid S$ holds exactly when $I(a; b \mid S) = 0$. Both sides of these two lines are therefore conditional mutual informations, and the dependence measure places no constraint on the distribution of the measurements or on the form of their relations.

\begin{proof}
The vector $X^e$ is Gaussian with mean zero and
\begin{equation}
\operatorname{Cov}\big(X^e_u, X^e_v\big) = e_u^\top \mathbb{E}[Z Z^\top]\, e_v = \langle e_u, e_v \rangle .
\label{eq:cmi-cov}
\end{equation}
Decompose $e_a = \sum_{s \in S} \nu_s e_s + e_{a \cdot S}$ and $e_b = \sum_{s \in S} \nu'_s e_s + e_{b \cdot S}$, with $e_{a \cdot S}$ and $e_{b \cdot S}$ orthogonal to every $e_s$. Then
\begin{equation}
X^e_a = \sum_{s \in S} \nu_s X^e_s + R_a, \qquad X^e_b = \sum_{s \in S} \nu'_s X^e_s + R_b, \qquad R_a = \langle e_{a \cdot S}, Z \rangle,\;\; R_b = \langle e_{b \cdot S}, Z \rangle .
\label{eq:cmi-split}
\end{equation}
By Eq.~\ref{eq:cmi-cov}, $\operatorname{Cov}(R_a, X^e_s) = \langle e_{a \cdot S}, e_s \rangle = 0$ for every $s \in S$, and likewise for $R_b$. Since $(R_a, R_b, X^e_S)$ is jointly Gaussian,
\begin{equation}
(R_a, R_b) \indep X^e_S .
\label{eq:cmi-indep}
\end{equation}
Given $X^e_S = x_S$, Eqs.~\ref{eq:cmi-split} and~\ref{eq:cmi-indep} make $(X^e_a, X^e_b)$ the pair $(R_a, R_b)$ shifted by constants, and a shift of each argument leaves the mutual information unchanged. Hence
\begin{equation}
I\big(X^e_a; X^e_b \mid X^e_S\big) = \mathbb{E}_{x_S}\, I\big(X^e_a; X^e_b \mid X^e_S = x_S\big) = I(R_a; R_b).
\label{eq:cmi-reduce}
\end{equation}
If $e_{a \cdot S} = 0$ or $e_{b \cdot S} = 0$, then $R_a$ or $R_b$ is $0$ almost surely and $I(R_a; R_b) = 0$, as Eq.~\ref{eq:cmi-vectors} gives under the convention for a vanishing residual. Otherwise $(R_a, R_b)$ is a bivariate Gaussian whose correlation is $\cos(e_{a \cdot S}, e_{b \cdot S})$ by Eq.~\ref{eq:cmi-cov}, so \citep[Chapter 8]{CoverThomas2006}
\begin{equation}
I(R_a; R_b) = -\tfrac{1}{2} \log\big(1 - \cos^2(e_{a \cdot S}, e_{b \cdot S})\big),
\label{eq:cmi-gauss}
\end{equation}
which is $+\infty$ when the absolute cosine similarity is $1$. Eqs.~\ref{eq:cmi-reduce} and~\ref{eq:cmi-gauss} give Eq.~\ref{eq:cmi-vectors}. Eq.~\ref{eq:cmi-psi} follows, since Eq.~\ref{eq:cmi-vectors} gives $e^{-2I} = 1 - \cos^2(e_{a \cdot S}, e_{b \cdot S})$ for $I = I(X^e_a; X^e_b \mid X^e_S)$.

Last, let the vectors of $S$ be linearly independent and neither residual vanish. By Eq.~\ref{eq:cmi-cov}, the covariance matrix $\Sigma$ of $X^e$ is the Gram matrix of the vectors, and $\Sigma_{S,S}$ is invertible. The projection of $e_a$ onto the span of the $e_s$ has the coefficients $\Sigma_{S,S}^{-1} \Sigma_{S,a}$, so
\begin{equation}
\langle e_{a \cdot S}, e_{b \cdot S} \rangle = \Sigma_{a,b} - \Sigma_{a,S}\, \Sigma_{S,S}^{-1}\, \Sigma_{S,b} = \Sigma_{ab \cdot S},
\label{eq:cmi-schur}
\end{equation}
the covariance given $S$ of Eq.~\ref{eq:pcor}. With $b = a$, Eq.~\ref{eq:cmi-schur} gives $\|e_{a \cdot S}\|^2 = \Sigma_{aa \cdot S}$, and likewise for $b$, so $\cos(e_{a \cdot S}, e_{b \cdot S}) = \operatorname{pcor}(a, b \mid S)$.
\end{proof}

\section{Datasets and Preprocessing}
\label{apx:data}

The evaluation uses two kinds of data. Robot and vehicle channels are time series whose names are short labels, and questionnaire items are statements that respondents rate. In both, a share of the true names is masked and the remaining names serve as anchors. Every trained part of \ours{} is fitted on data of its own.

\subsection{Robotics scenarios}
\label{apx:data-robot}

The three scenarios differ in how much of the platform is documented, which is what the claim of the paper depends on. A simulator comes with the true graph of its physics, and a real platform comes with a declared channel layout. A vehicle CAN comes with no documentation from the maker of the car. Within each scenario the platforms are those whose channel layout is known, and for the drives the layout comes from the comma2k19 release. Together they cover the kinds of channel system the method has to handle. Table~\ref{tab:data-r} lists the evaluated arms, logs and drives, and Figure~\ref{fig:data-bodies} shows the channel families each of them carries, which is what differs between platforms.

\paragraph{Simulated arms.}
For five arms we collect 1{,}000 episodes of 200 steps in \texttt{robosuite}
\citep{robosuite} on the MuJoCo engine \citep{MuJoCo} under the default
\href{https://robosuite.ai/docs/modules/controllers.html}{operational-space controller} \citep{Khatib1987}, with random actions. For the Panda we use the 1{,}500 machine-generated lift episodes of 150 steps released with robomimic \citep{robomimic}, whose actions come from a policy. Discovery treats these episodes as it treats the other arms. Each simulator step is one row and each
scalar state variable one column. Columns that are functions of other columns are dropped before the pipeline runs. The dropped columns are of three kinds. The concatenated state vectors repeat the individual channels, and the cosine and sine of each joint angle are functions of the angle column. The object-relative position is the difference of two recorded positions. The manipulated
object is excluded, because its coupling to the end-effector exists only
during contact and a static graph cannot express a state-dependent edge.
The evaluation uses 1{,}200 steps of each arm. The true graph of a simulated arm is written from the controller and the kinematics and checked against the simulator edge by edge. Appendix~\ref{apx:discovery-recovered} compares the discovered structure with it.

\paragraph{Real robot logs.}
The real logs are three datasets read through their LeRobot parquet form \citep{LeRobot}. A UTokyo xArm
pick-and-place log \citep{WeblabXArm} and a Stanford HYDRA log \citep{HYDRA} come from the Open X-Embodiment release
\citep{OpenX}, and both record a single arm at the end-effector. A cup-handling log comes from the ALOHA release \citep{ALOHA} and records a bimanual arm at every joint. Their channels are the recorded state and action columns of each platform, from 13 to 28 per log. The evaluation uses 1{,}200 steps of each log.

\paragraph{Vehicle drives.}
The CAN bus (Controller Area Network) is the network on which the controllers of a car exchange messages. The engine, brake and steering controllers are three of them. Every message carries the current value of a few signals. A logging device attached to the bus records them without any document that says what a signal is. A drive is therefore a natural case of measured but unnamed variables. The vehicle drives come from
comma2k19 \citep{comma2k19}. A drive has twelve channels, six from the CAN bus and six from the inertial sensor of the recording device. The CAN bus supplies the vehicle speed, the steering wheel angle and the four wheel speeds. The inertial sensor supplies three accelerometer axes and three gyroscope rates. Each channel is linearly resampled to a common 10 Hz grid over the segments of the drive and cut to 2{,}400 samples, or to 1{,}199 on one shorter drive. The supplied graph for the drives encodes the standard vehicle dynamics in nine edges. Steering points to yaw rate and lateral acceleration, and speed points to each wheel speed and to yaw rate. Longitudinal acceleration points to speed, and roll rate points to lateral acceleration. The evaluation drives are the third to the twelfth in sorted order, all from one car. They range from stop-and-go traffic to steady highway speed (Figure~\ref{fig:data-drives}). Figure~\ref{fig:data-robot} shows one arm and one drive as they enter the pipeline.

\begin{table}[!htb]\centering\footnotesize\setlength{\tabcolsep}{3.5pt}
\caption{The robot and vehicle data sets of the evaluation. Channels are counted at one time step, actions included.}
\label{tab:data-r}
\begin{tabular}{@{}>{\raggedright\arraybackslash}p{0.17\textwidth}>{\raggedright\arraybackslash}p{0.36\textwidth}rr>{\raggedright\arraybackslash}p{0.18\textwidth}@{}}
\toprule
set & members & number & channels & steps \\
\midrule
simulated arms & Panda (lift), Sawyer, IIWA, UR5e, Kinova3, xArm7 & 6 & 32 to 44 & 1{,}200 \\
real logs & UTokyo xArm, Stanford HYDRA, ALOHA cup & 3 & 13, 14, 28 & 1{,}200 \\
vehicle drives & comma2k19 drives & 10 & 12 & 2{,}400, one drive 1{,}199 \\
\bottomrule
\end{tabular}
\end{table}

\begin{figure}[!htb]\centering
\includegraphics[width=\textwidth]{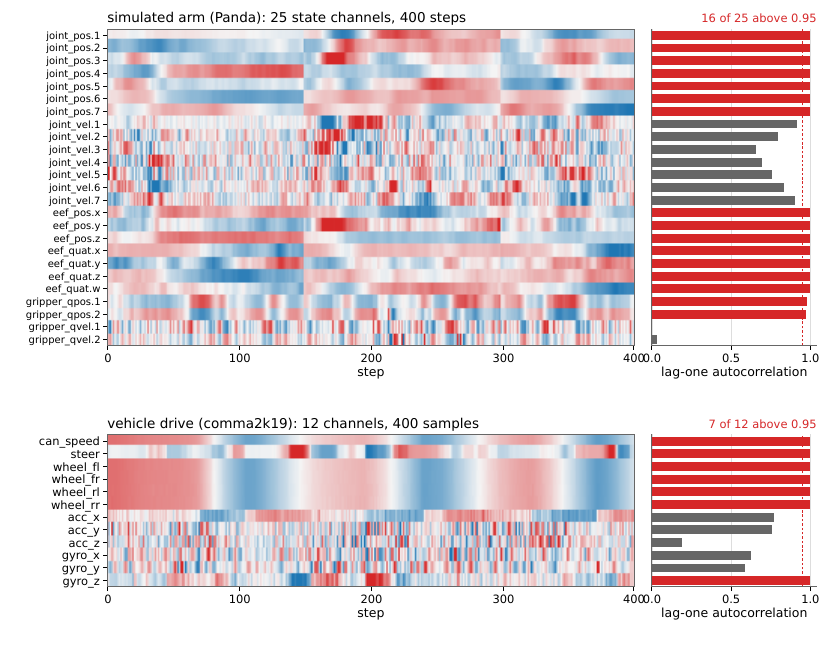}
\caption{One simulated arm and one vehicle drive. Left, every state channel standardized
over 400 consecutive steps. Right, the lag-one autocorrelation of each
channel, in red above .95.}
\label{fig:data-robot}
\end{figure}

\begin{figure}[!htb]\centering
\includegraphics[width=\textwidth]{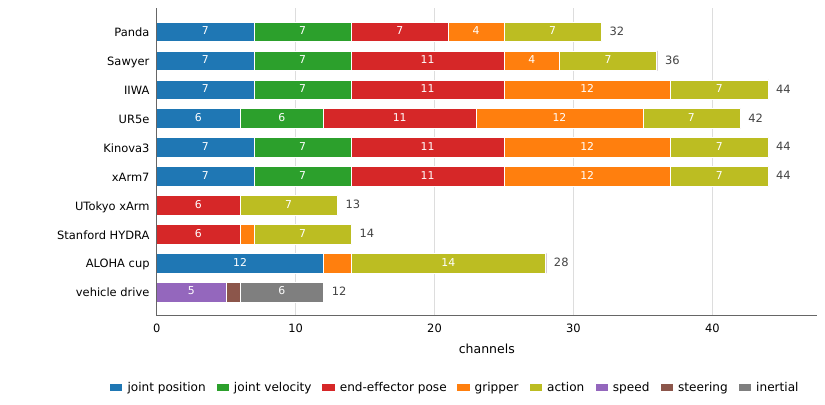}
\caption{Channels of every evaluated arm, log and drive, by family. A simulated arm
carries joint positions and velocities, the end-effector pose, the gripper
and the commanded action, a real log what its platform recorded, and a
vehicle drive the speed signals, the steering wheel angle and the inertial unit.}
\label{fig:data-bodies}
\end{figure}

\begin{figure}[!htb]\centering
\includegraphics[width=\textwidth]{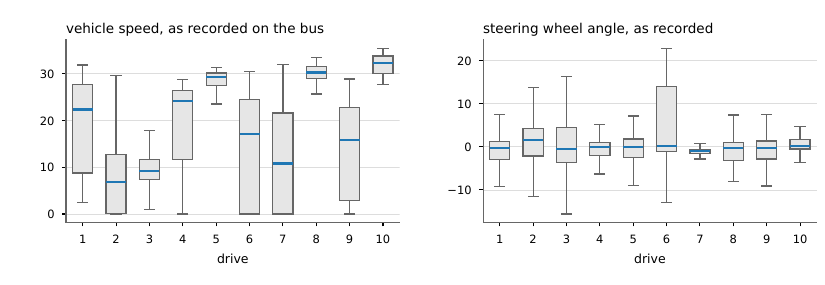}
\caption{Vehicle speed and steering wheel angle over the ten evaluation
drives, as recorded on the bus. Boxes span the middle half of the samples
and whiskers reach 1.5 times that range. The drives range from stop-and-go traffic to steady
highway speed.}
\label{fig:data-drives}
\end{figure}

\subsection{Questionnaires}
\label{apx:data-quest}

\paragraph{Corpus and preprocessing.}
Table~\ref{tab:data-q} lists all twenty-one scales. Seventeen are public response matrices from the Open-Source Psychometrics Project \citep{openpsychometrics}. The World Values Survey variables are the components of the indices of \citet{Welzel2013}, taken from the Survey's Wave 7 release \citep{WVS7}. The Holzinger-Swineford battery \citep{HolzingerSwineford} is the 1939 set of cognitive tests, of which we use twenty-four. Seventeen measures of the Himi multitasking study \citep{Himi2019} and its TLVD benchmark subset \citep{TLVD} come from the study's OSF archive \citep{HimiOSF}. Figure~\ref{fig:data-corpus} shows how the corpus is composed. Items are the
observed variables. The published scoring key supplies the latent variables, which
are the reference for the latent side. Where a scale has a nested key
the nesting is kept. HEXACO has six domains over twenty-four facets, and the World Values Survey has the two overall indices of Welzel over seven of their eight sub-indices. A respondent with a missing entry on any item of the scale is dropped. The response matrix is then capped at 5{,}000 respondents drawn with a fixed seed, and every column is standardized. The correlation matrix that enters the discovery step is therefore complete.

\paragraph{Evaluation scales.}
A scale enters the evaluation only if its scoring key is published and
every latent variable in that key has a published name, so that every
target has a reference. Within that rule the five are chosen to span the forms a questionnaire can take. They cover a clinical scale, a personality inventory, a
mindfulness scale, a vocational interest inventory and a values survey.
They run from 21 to 240 items and from 3 to 30 latent variables counted
with their parent level. HEXACO and the World Values Survey carry a nested key, which allows a test of a latent structure of more than one level. The five also differ in how visible the latent structure is in the responses.

Figure~\ref{fig:data-corr} shows the five as correlation matrices with items ordered by latent variable. DASS,
RIASEC and KIMS show clean blocks, and HEXACO shows twenty-four small
blocks inside six larger ones. The block structure in that figure is what
the discovery step has to find without the key.
For each matrix, Figure~\ref{fig:data-corr-dist} shows the absolute correlation of item pairs inside one latent variable and of item pairs across latent variables. The rank tests have to detect the contrast between these two distributions. The two medians lie farthest apart on the World Values Survey
and RIASEC and closest on DASS, whose three latent variables correlate
with one another.

\begin{table}[!htb]\centering\scriptsize\setlength{\tabcolsep}{3.5pt}
\caption{Every questionnaire scale used in this paper, latent variables counted with any
parent level. The source column cites the data release: the raw-data index of the Open-Source Psychometrics Project (OSP), the World Values Survey Wave 7 release, the \texttt{sem} package of R for the 1939 battery, and the Himi OSF archive cited above, shared by the Himi and TLVD rows.}
\label{tab:data-q}
\resizebox{\textwidth}{!}{%
\fittable{llrrrll}{
\toprule
scale & instrument & items & latent variables & respondents & source & role \\
\midrule
DASS & Depression Anxiety Stress Scales \citep{DASS} & 42 & 3 & 5{,}000 & \citep{OSPDASS} & evaluation \\
HEXACO & IPIP-HEXACO scales \citep{HEXACO} & 240 & 6 domains, 24 facets & 5{,}000 & \citep{OSPHEXACO} & evaluation \\
KIMS & Kentucky Inventory of Mindfulness Skills \citep{KIMS} & 39 & 4 & 521 & \citep{OSPKIMS} & evaluation \\
RIASEC & Holland occupational interest types \citep{RIASEC} & 48 & 6 & 5{,}000 & \citep{OSPRIASEC} & evaluation \\
WVS & World Values Survey, Welzel indices \citep{Welzel2013} & 21 & 2 indices, 7 sub-indices & 5{,}000 & \citep{WVS7} & evaluation \\
\midrule
16PF & IPIP scales for Cattell's sixteen personality factors \citep{Cattell16PF, IPIP} & 162 & 16 & 5{,}000 & \citep{OSP16PF} & training \\
WPI & Woodworth Psychoneurotic Inventory \citep{WPI} & 116 & 1 & 3{,}888 & \citep{OSPWPI} & training \\
TMA & Taylor Manifest Anxiety Scale \citep{TMA} & 50 & 1 & 4{,}474 & \citep{OSPTMA} & training \\
Big Five & IPIP fifty-item Big Five markers \citep{BigFiveMarkers, IPIP} & 50 & 5 & 3{,}000 & \citep{OSPBigFive} & training \\
HSQ & Humor Styles Questionnaire \citep{HSQ} & 32 & 4 & 993 & \citep{OSPHSQ} & training \\
SD3 & Short Dark Triad \citep{SD3} & 27 & 3 & 5{,}000 & \citep{OSPSD3} & training \\
NPAS & Nerdy Personality Attributes Scale \citep{NPASDevelopment} & 26 & 1 & 5{,}000 & \citep{OSPNPAS} & training \\
HS & Holzinger-Swineford cognitive tests \citep{HolzingerSwineford} & 24 & 5 & 301 & \citep{semR} & training \\
HSNS+DD & Hypersensitive Narcissism Scale with the Dirty Dozen \citep{HSNS, DirtyDozen} & 22 & 4 & 5{,}000 & \citep{OSPHSNSDD} & training \\
MACH-IV & Machiavellianism scale \citep{MACH} & 20 & 1 & 5{,}000 & \citep{OSPMACH} & training \\
Himi & multitasking and executive function battery \citep{Himi2019} & 17 & 6 & 202 & Himi OSF archive & training \\
GCBS & Generic Conspiracist Beliefs Scale \citep{GCBS} & 15 & 5 & 2{,}400 & \citep{OSPGCBS} & training \\
CFCS & Consideration of Future Consequences Scale \citep{CFCS} & 12 & 1 & 5{,}000 & \citep{OSPCFCS} & training \\
SCS & Sexual Compulsivity Scale \citep{SCS} & 10 & 1 & 3{,}243 & \citep{OSPSCS} & training \\
RSE & Rosenberg Self-Esteem Scale \citep{RSE} & 10 & 1 & 5{,}000 & \citep{OSPRSE} & training \\
TLVD & the TLVD benchmark subset of the Himi study \citep{TLVD} & 9 & 4 & 202 & Himi OSF archive & training \\
\bottomrule
}}
\end{table}

\begin{figure}[!htb]\centering
\includegraphics[width=\textwidth]{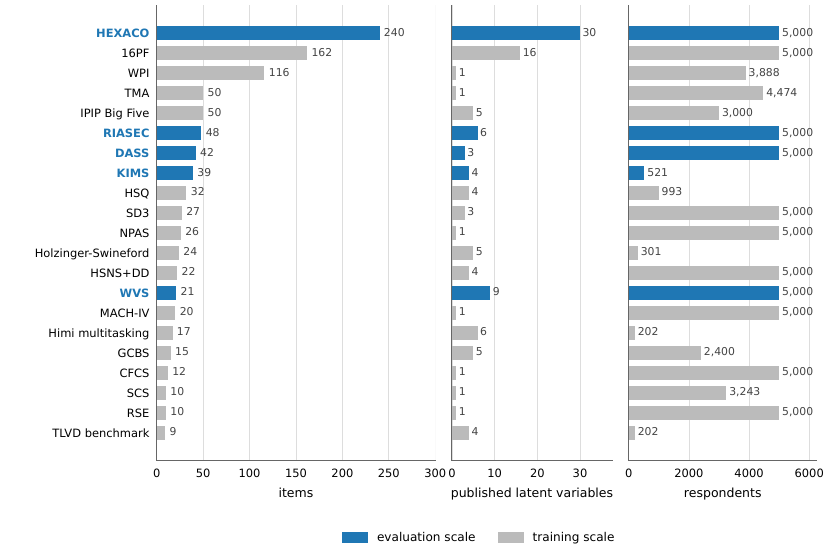}
\caption{Composition of the questionnaire corpus. Bars give the items, the published latent variables and the respondents per scale, with the five evaluation scales in color. HEXACO counts its six domains and twenty-four facets together.}
\label{fig:data-corpus}
\end{figure}

\begin{figure}[!htb]\centering
\includegraphics[width=0.99\textwidth]{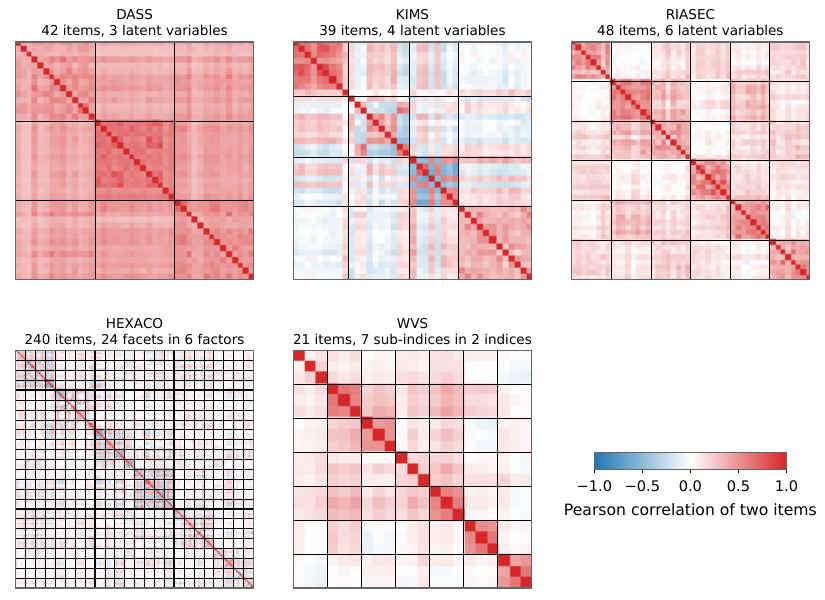}
\caption{The five evaluation scales as item correlation matrices. Items are ordered by the published latent variable, and the borders of the latent variables are drawn. Red is positive and blue negative.}
\label{fig:data-corr}
\end{figure}

\begin{figure}[!htb]\centering
\includegraphics[width=\textwidth]{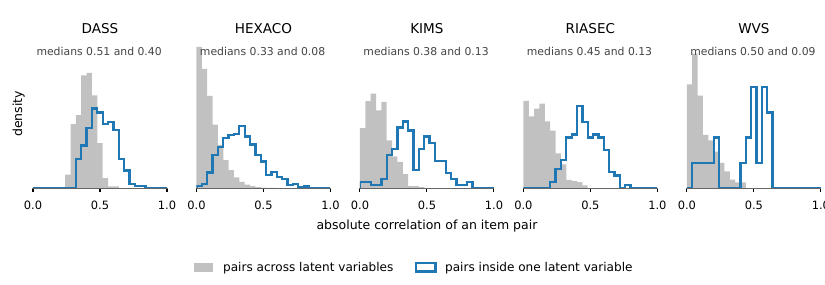}
\caption{Absolute correlation of item pairs inside one published latent variable against
pairs across latent variables, per evaluation scale, with the two medians printed
in each panel.}
\label{fig:data-corr-dist}
\end{figure}

\subsection{Masking}
\label{apx:data-cols}

Names are masked at 20, 50, 75 and 90\%. On a questionnaire scale
five folds are drawn at each level, each with its own set of masked
names. On an arm, log or drive one set is drawn at each level. The
same set is reused for every method and every structure source. Only the name of an unnamed variable is masked. Its data stay available, since the task is to name a variable that is measured but undocumented.
Every latent variable is a target at every level, and the level sets how
many of its items keep a name.
Figure~\ref{fig:data-masking} shows the numbers of masked
items and named anchors on each evaluation scale.

\begin{figure}[!htb]\centering
\includegraphics[width=\textwidth]{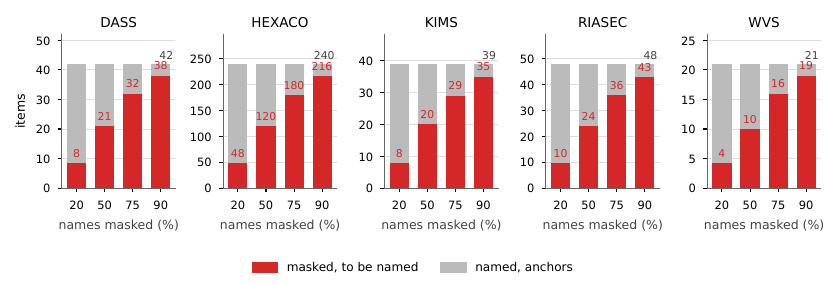}
\caption{What one fold leaves named on each evaluation scale. Bars are the items of the scale. The masked part is what the fold asks the method to name, and the named part is what remains as anchors. The pink count above each masked part is the number of masked items, and the gray count is the total number of items in the scale.}
\label{fig:data-masking}
\end{figure}

\subsection{Training data}
\label{apx:data-roles}

The trained parts of \ours{} belong to the optimization or to the adapter, and each has training data of its own. Table~\ref{tab:data-audit} lists every trained
part against its data, and none of it overlaps the evaluation scales,
arms, logs or drives.

\paragraph{The optimization.}
The operator, the weight network and the encoder fine-tuning of Appendix~\ref{apx:solver} are fitted on the sixteen training scales of Table~\ref{tab:data-q} with their published keys. These public scales span nine to 162 items and one to sixteen latent variables. On the robot side the operator and the weight network are fitted on twenty generated arms and vehicles. The encoder fine-tuning is warm-started from the questionnaire one and fitted on 300 generated arms and vehicles, while it keeps unchanged the vectors of 20{,}000 English words and phrases sampled from the concept bank. This bank combines the ConceptNet Numberbatch vocabulary, WordNet and the Cognitive Atlas \mbox{\citep{ConceptNet, WordNet, CognitiveAtlas}}. The
semantic negation map is trained on WordNet antonym pairs and on the pole pairs of the sixteen training scales.

\paragraph{The adapter.}
The questionnaire adapter is trained on 600 synthetic scales. The local language model Qwen3-8B \citep{Qwen3} writes latent variable names for forty domains of psychology and eight first-person items for each, two or three of them reverse keyed. A scale is assembled from one to five latent variables of one domain. Likert responses for 600 respondents are then simulated from a factor model with loadings, reverse keying and correlations between the latent variables. A generated item or latent variable is dropped when its encoder cosine to any of the 1{,}121 real item texts and latent variable descriptions reaches .92. This filter removes 5{,}666 of 12{,}376 items and 108 of 1{,}547 latent variables. The 600 solved scales yield 33{,}073 training
examples. The robot adapter is trained on 1{,}600 generated arms and vehicles, each entered twice with two edge sets. The entries yield 391{,}352 examples in all, of which the examples of 96 entries are held out for early stopping.

\paragraph{Generated arms and vehicles.}
A generated arm or vehicle is a causal channel system that comes with
its true graph, its measurements and a name for every channel. It takes one of the four shapes of the evaluation platforms. These are an arm, a log that records only the end-effector, a bimanual joint log and a vehicle CAN. An arm has five to eight joints, an end-effector pose block, a gripper and a commanded action. Each joint has an angle, which follows a commanded value with a lag, and a velocity. The edges run from the action to the joints, between the angle and the velocity of a joint, and from the joints to the end-effector pose and the gripper. A vehicle has the twelve channels of the vehicle CAN, with the steering exogenous. The wheel speeds and the vehicle speed lie on the drivetrain, and the inertial channels are responses. The signals are simulated from these dynamics with autoregressive noise over 600 to 1{,}200 steps, and the names come from template banks. About 15\% of the channels have their sign flipped, so that the optimization is trained on inverted channels. Two real logs and the first two comma2k19 drives supply signal statistics for the generator and are not evaluated. The two logs are an ALOHA coffee task and a Berkeley AutoLab UR5 log \citep{BerkeleyUR5}.

\begin{table}[!htb]\centering\scriptsize\setlength{\tabcolsep}{4pt}
\caption{Every trained part of \ours{} and the data it was fitted on.}
\label{tab:data-audit}
\begin{tabular}{p{0.30\textwidth}p{0.44\textwidth}p{0.18\textwidth}}
\toprule
trained part & training data & size \\
\midrule
optimization, questionnaires (operator, weight network, encoder fine-tuning) & the sixteen public scales with their published keys (Table~\ref{tab:data-q}) & 16 scales \\
optimization, robots (operator, weight network) & generated arms and vehicles & 20 \\
encoder fine-tuning, robots & generated arms and vehicles, warm-started from the questionnaire one, 20{,}000 general words held fixed & 300 \\
semantic negation map & WordNet antonym pairs and pole pairs from the sixteen training scales & --- \\
adapter and low-rank update, questionnaires & synthetic scales, after the leakage filter & 600 scales, 33{,}073 examples \\
adapter and low-rank update, robots & generated arms and vehicles, each under its true graph and under a degraded edge set & $1{,}600 \times 2$, 391{,}352 examples \\
\bottomrule
\end{tabular}
\end{table}

\section{Causal Structure Learning}
\label{apx:discovery}

\begin{figure}[!htb]\centering
\includegraphics[width=\textwidth]{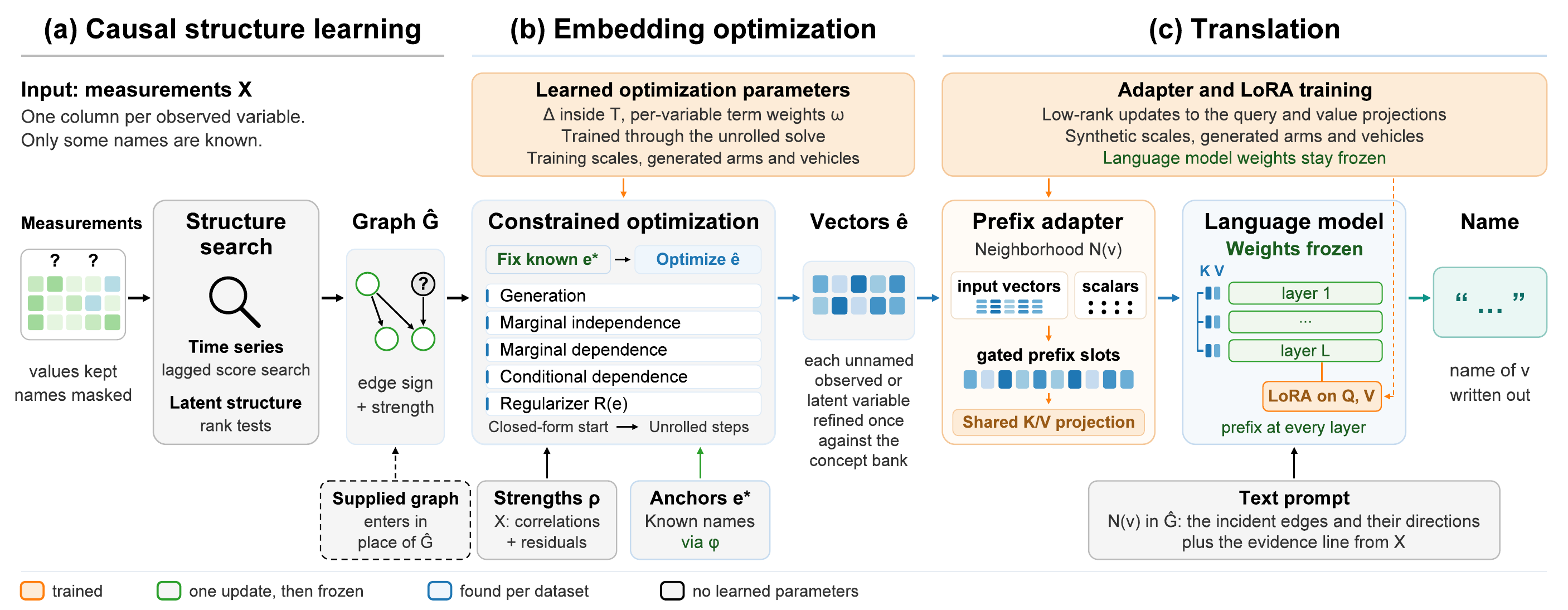}
\caption{The objects \ours{} passes from measurements to a name. Structure learning returns the graph. The optimization assigns a vector to every unnamed variable under the relations of the graph, given the anchors. The adapter passes the vector to a frozen language model as a prefix, with the neighborhood written as text. Inputs enter from below the row and
training from above it.}
\label{fig:pipeline}
\end{figure}

Structure learning learns the graph $\hat G$ from the measurements alone, and the optimization of Appendix~\ref{apx:solver} receives it. Robot and vehicle channels are observed time series, and BOSS runs over two consecutive time steps. Questionnaires carry latent variables, and RLCD runs with BOSS as its first stage. Appendix~\ref{apx:theory-discovery} states what each method guarantees.

\subsection{Robot and vehicle channels}
\label{apx:discovery-robot}

Each arm, log or drive enters the search as a lag-one design matrix. A row holds every channel at the previous and the current time step and the action at the previous time step. A channel whose lag-one autocorrelation is near one is replaced by its first difference. At high sampling rates, such a near-integrated channel would otherwise conceal every other parent. On the real logs and the vehicle drives a channel is
replaced when its lag-one autocorrelation exceeds .95, and on the
simulated arms every state channel is replaced. BOSS runs over all columns on the correlation
matrix. With $u_{t-1}$ and $v_t$ two channels at the previous and the
current time step, the retained edge set is
\begin{equation}
E \subseteq \{u_{t-1} \to v_t : u, v \in V\},
\label{eq:temporal}
\end{equation}
and each child is then regressed on its retained parents by least
squares, which gives every edge a signed coefficient. The lagged graph is then reduced to a summary graph over the channels, so every channel enters the next step as one variable. An edge $u_{t-1} \to v_t$ becomes $u \to v$, and a self-edge $u_{t-1} \to u_t$ is dropped. Two passes on partial correlations then refine the
edge set. The first pass runs on the simulated arms, whose actions are inputs. It adds an edge from an input to a channel when their partial correlation is high, given the channel's own lag and the other inputs. The second pass drops an edge whose partial correlation, given the other parents and the inputs, is low. Table~\ref{tab:settings} gives the two thresholds.

\paragraph{A physical law as an optional prior.}
The discovery step can also admit the edges within a time step that the physics of the platform orients. On an arm these run from the joints to the end-effector pose and the gripper. On a drive they run from the steering wheel angle and the wheel speeds to the vehicle speed, and from those to the accelerations and the rates. Such a law contains no information about the name of a channel, so it can be supplied wherever a platform makes it available. Appendix~\ref{apx:j-graphs} reports how this prior changes the names.

\subsection{Questionnaires}
\label{apx:discovery-quest}

The discovery method is RLCD as released in
\texttt{causal-learn} \citep{causallearn}, with its first stage exchanged. BOSS
on the correlation matrix replaces greedy equivalence search \citep{GES}
and partitions the items into the groups the rank tests then work within.
The penalty discount of the score is raised where a group would otherwise be too large for those tests, since their cost grows exponentially with the number of latent variables in a group. The replacement also keeps the first stage fast on the larger scales (Appendix~\ref{apx:compute}).

Rank tests then introduce latent variables. The search keeps an active set of covers,
which starts as the single observed variables and gains every cover it
finds. Within a group, the search looks for $k + 1$ variables made of observed variables $\mathcal{X}$ and atomic covers $\mathcal{C}$ of the active set. The covariance of these variables with $\mathcal{X}$ and the rest of the group has rank $k$,
\begin{equation}
\operatorname{rank} \Sigma_{\mathcal{C} \cup \mathcal{X},\; \mathcal{X} \cup R} = k,
\qquad |\mathcal{C}| + |\mathcal{X}| = k + 1,
\label{eq:rank}
\end{equation}
where $R$ holds the other variables of the group without the measured
descendants of $\mathcal{C}$. For generic parameters such a rank constraint means that a pair of sets with $k$ members in total t-separates the two sides (Theorem~\ref{thm:trek}). Under the conditions of \citet[Theorem~8]{Dong2024}, it comes from an atomic cover made of $\mathcal{X}$ and $k - |\mathcal{X}|$ latent variables with $\mathcal{C}$ among its pure children, and the search introduces this cover. The latent variables are new unless found before. The rank
of a covariance submatrix is tested through its canonical correlations
with a Wilks statistic \citep[Appendix~A.19]{Dong2024}. As in the RLCD paper, the tests run on the correlation matrix of the standardized item responses. A collider check prevents a collider from being
mistaken for a latent parent. The search starts at $k = 1$, returns to
$k = 1$ after every cover it finds and moves to $k + 1$ otherwise, up to a
maximum cardinality. A cover that is found joins the variables of the
next tests, so that a latent variable may have latent parents. The rank tests run batched on one GPU in a single serial process, which
keeps the search deterministic. Their accept and reject decisions agree
with the released CPU test on all 1{,}000 test cases, and on a
fifteen-variable scale the two return the same graph.

\subsection{Recovered structures}
\label{apx:discovery-recovered}

\paragraph{Robot and vehicle channels.}
On the simulated arms the discovered graph has about as many edges as the true graph and shares about a quarter of the true graph's edges. The true graph is the simulator's own physics, so the discovered graph can be compared with it. Table~\ref{tab:struct-robot} counts the edges of both graphs per arm in the time-lagged form of Eq.~\ref{eq:temporal}. Appendix~\ref{apx:struct-robot} shows one arm, one log and one drive as a summary graph over the channels. The discovered graph recovers the edges from the actions to the joints and the edges from one step to the next among the joints. The edges it adds join channels that are correlated. The forward kinematics act within a time step, so the retention rule of Eq.~\ref{eq:temporal} drops their edges. The edges from a velocity to the position it integrates are the ones the search misses most, since those channels enter the search as first differences.

On the xArm and HYDRA logs and on the vehicle drives, the discovered graph shares few edges with the supplied graph. The vehicle drives are the informative case.
The supplied graph has an edge from the vehicle speed to each of the four
wheel speeds and nine edges in all. The search finds one to six edges per drive, and most of them lie among the accelerations and the rates. The yaw rate enters the lateral acceleration on five drives and the vertical acceleration enters the longitudinal acceleration on five. The search shares at most two edges with the supplied graph on any drive, and it returns none of the four edges from the vehicle speed to the wheel speeds.

\begin{table}[!htbp]\centering\small\setlength{\tabcolsep}{3.5pt}
\caption{The discovered structure of every arm, log and drive beside the supplied graph. Simulated arms are counted at the lagged level, real logs and vehicle drives at the channel level. Shared counts the edges the two graphs have in common.}
\label{tab:struct-robot}
\fittable{lrrr}{
\toprule
arm, log or drive & supplied & found & shared \\
\midrule
Panda & 113 & 148 & 29 \\
Sawyer & 108 & 100 & 27 \\
IIWA & 122 & 106 & 31 \\
UR5e & 106 & 117 & 34 \\
Kinova3 & 121 & 109 & 24 \\
xArm7 & 119 & 92 & 28 \\
\midrule
UTokyo xArm & 6 & 5 & 1 \\
Stanford HYDRA & 7 & 5 & 2 \\
ALOHA cup & 14 & 37 & 10 \\
vehicle CAN, 10 drives & 9 & 1 to 6 & 0 to 2 \\
\bottomrule
}
\end{table}

\begin{table}[!htbp]\centering\small\setlength{\tabcolsep}{4pt}
\caption{The structure discovered on each evaluation scale against the published scoring key. Published latent variables are the groups of the key at the item level, and items reached counts the items under at least one discovered latent variable.}
\label{tab:struct-q}
\fittable{lrrrrr}{
\toprule
scale & published latent variables & discovered latent variables & latent to latent edges & items reached & purity \\
\midrule
DASS & 3 & 9 & 2 & 40 / 42 & 0.86 \\
HEXACO & 24 & 25 & 5 & 64 / 240 & 0.80 \\
KIMS & 4 & 11 & 5 & 29 / 39 & 0.88 \\
RIASEC & 6 & 13 & 8 & 28 / 48 & 0.95 \\
WVS & 7 & 12 & 14 & 21 / 21 & 0.53 \\
\midrule
all five & 44 & 70 & 34 & 182 / 390 & --- \\
\bottomrule
}
\end{table}

\paragraph{Questionnaires.}
The discovered structure is finer than the published scoring key on four
of the five scales and cuts across it on the World Values Survey
(Table~\ref{tab:struct-q}). Purity is the share of the items of a discovered latent variable that belong to its reference (Appendix~\ref{apx:latent-ref}), averaged over the discovered latent variables of the scale. DASS has three published latent variables and nine discovered ones. Six of the nine lie inside one published latent variable, and the other three mix anxiety with stress. The method therefore splits each published latent variable into smaller groups of items and mixes only anxiety and stress. RIASEC and KIMS show the same pattern. HEXACO is
the scale where the rank tests are most conservative, with twenty-five latent variables over 64 of the 240 items. The remaining items enter the next two steps with their BOSS edges and no latent parent. On the World Values Survey the
groups cut across the key the most, with a purity of .53 against .80 to
.95 on the other four scales. There a sub-index holds about three items, while the largest discovered group holds eight. The groups therefore span several sub-indices. Appendix~\ref{apx:struct-merged} shows the discovered graphs of DASS and the World Values Survey.

\section{Constrained Embedding Optimization}
\label{apx:solver}

The optimization receives the graph $\hat G$ of Appendix~\ref{apx:discovery}, the measurements and the names of the named variables. It returns a vector for every unnamed variable, which the adapter of Appendix~\ref{apx:plugin} translates into a name.

\subsection{Space and variables}
\label{apx:solver-space}

Embeddings lie in the space of \texttt{e5-large-v2} \citep{E5}, with the low-rank update of Appendix~\ref{apx:solver-trained}. The vectors to be
optimized are one per unnamed variable. The vectors of the named variables
are the encodings of their names and enter as constants, which enforces the
anchoring line of Eq.~\ref{eq:realization} exactly.

\subsection{The generating operator}
\label{apx:solver-operator}

The \emph{operator} $T$ of the generating function
$g_v(e_{\mathrm{pa}(v)}) = \sum_{u \in \mathrm{pa}(v)} T(e_u; \rho_{uv})$ of
Section~\ref{sec:method-solver} is shared across edges,
\begin{equation}
r_v = e_v - \sum_{u \in \mathrm{pa}(v)} T(e_u; \rho_{uv}),
\qquad
T(e_u; \rho_{uv}) = |\rho_{uv}|\,\big( b_{uv}(e_u) + \Delta(e_u, \rho_{uv}, \eta_{uv}) \big),
\label{eq:residual}
\end{equation}
with $b_{uv}(e_u) = e_u$ on an edge with $\rho_{uv} > 0$ and
$b_{uv}(e_u) = f_{\mathrm{neg}}(e_u)$ on an edge with $\rho_{uv} < 0$, and
$\eta_{uv}$ the type of the edge, which records whether it joins two latent
variables. $\Delta$ is a small network conditioned on $\rho_{uv}$ and
$\eta_{uv}$. Its output starts at zero, so the operator initially scales the parent's vector or its image under the semantic negation map by $|\rho_{uv}|$. The correction is learned during training. $T$ is additive over the parents, so a child receives
nothing from a non-parent and the sparsity of the Jacobian follows the
graph without a penalty. $\Delta$ can change the sign of a contribution, so its training includes the sign-consistency penalty $\mathcal{S}$ of Eq.~\ref{eq:solve}. This penalty aligns the contribution of a positive edge with the parent's vector and that of a negative edge with its image under the semantic negation map.

The \emph{semantic negation map} $f_{\mathrm{neg}}$ is a small network in
the encoder space. It exists because a reverse-keyed item is generated
from the opposite pole of its latent parent's meaning, and the embedding
of the antonym differs from the negated vector $-e_u$.

\subsection{The objective}
\label{apx:solver-objective}

The objective is Eq.~\ref{eq:objective} of Section~\ref{sec:method-solver} with $\delta$ the cosine similarity, and every relation enters as a quadratic penalty \citep{NocedalWright2006}.
$\mathcal{P}_{\mathrm{ind}}$ is the set of pairs that are d-separated by the empty set in $G$. In such a pair the two variables have no common ancestor, and neither is an ancestor of the other. $\mathcal{P}_{\mathrm{con}}$ is the set of pairs of
variables that have parents, and $\mathcal{P}_{\mathrm{dep}}$ the set of
trek-connected pairs of strongest dependence. The distance correlation
$\rho_{ab}$ is computed in the measurements and
signed by the Pearson correlation of the two variables. In the marginal independence term it is set to zero where the absolute Pearson correlation falls below the significance threshold $2/\sqrt{n}$. This thresholded value appears as $\tilde\rho_{ab}$ in Eq.~\ref{eq:relations}. A pair that is independent both in the graph and in the data therefore becomes orthogonal. A measured dependence that the graph does not imply is kept.
The target $\rho_{ab \cdot \mathrm{pa}}$ is the distance correlation of the
residuals of the two variables after a gradient-boosted regression \citep{Friedman2001} on their
own parents, signed by the Pearson correlation of the residuals. The conditional dependence term matches it to the cosine similarity of the residuals of the two vectors on their own parents, formed through $r_v$ of Eq.~\ref{eq:residual}. Both sides therefore compare residuals, as a partial correlation does \citep{Whittaker1990}. The
marginal dependence term keeps the absolute cosine similarity of a
strongly dependent pair above the fraction $\kappa < 1$ of its unsigned
distance correlation, a lower bound that leaves the vectors the freedom to fit the other constraints.
A latent variable enters every correlation of the objective through its
score on the first principal component of its observed descendants \citep{Jolliffe2002}. The
first four terms vanish exactly on the estimated feasible set
(Proposition~\ref{prop:feasible}).

\subsection{The regularizer}
\label{apx:solver-reg}

The feasible set has more than one member, so the regularizer selects one of them, as regularization does in an ill-posed problem. The term $\mathcal{R}$ of Eq.~\ref{eq:objective} is
\begin{equation}
\mathcal{R}(e) = \lambda_{\mathrm{bank}} \sum_{v \in V_{\mathrm{U}}} \big(1 - \cos(e_v, \Pi(e_v))\big)
+ \lambda_{\mathrm{nbr}} \sum_{v \in O \cap V_{\mathrm{U}}} \big(1 - \cos(e_v, c_v)\big)
+ \frac{\lambda_{\mathrm{norm}}}{|V_{\mathrm{U}}|} \sum_{v \in V_{\mathrm{U}}} \big(\|e_v\| - 1\big)^2,
\label{eq:reg}
\end{equation}
where $\Pi(e_v)$ is the soft nearest-neighbor projection of $e_v$ onto the concept bank, which is the softmax-weighted mean of its nearest bank entries. The concept bank is the English phrase bank described in Appendix~\ref{apx:data-roles}. The neighbor mean $c_v$ is
the correlation-weighted mean of the vectors of the few named variables
that an unnamed observed variable correlates with most. On a questionnaire a negatively correlated named variable enters through the semantic negation map $f_{\mathrm{neg}}$, and on a robot it enters by its magnitude alone. The last
term holds the vectors at unit norm.

\subsection{Optimization}
\label{apx:solver-opt}

We initialize the embeddings in closed form by ridge-regularized least squares \citep{Tikhonov1977} with the linear approximation $T(e_u; \rho_{uv}) = \rho_{uv} e_u$. We then run unrolled Adam steps \citep{Adam, Domke2012}. Unrolling keeps the optimization differentiable, so the shared operator and the per-variable term weights can be fitted through it. At inference the same loop optimizes the embeddings with the trained operator and term weights held fixed.

\paragraph{Refinement against the concept bank.}
The optimized vectors are then refined once against the concept bank by
coordinate ascent. Each vector of an unnamed variable takes its fifty
nearest bank phrases as candidates, together with the vector itself. A candidate is scored by its cosine similarities to the current vectors of the other variables, each weighted by the measured correlation with that variable. The mean pairwise cosine similarity of the bank is subtracted from this score.
At most three sweeps of coordinate ascent pick the best candidate for
every unnamed variable, with the vectors of the named variables held
fixed. Where a phrase is picked, the returned vector is the normalized mean
of the optimized vector and that phrase. The operator and the term weights
are fitted without this step.

\subsection{The trained components}
\label{apx:solver-trained}

Four parts of the optimization are fitted once, each on training data of its own (Appendix~\ref{apx:data-roles}). These are the low-rank update of the encoder, the semantic negation map, the operator and the weight network.

\paragraph{Encoder fine-tuning.}
The encoder fine-tuning trains small low-rank matrices in the encoder on the sixteen training scales. Its loss uses the same dependence and independence relations as the objective. A dictionary of English words and phrases is held fixed during this fine-tuning, so that the space keeps its semantics. The trained matrices are then frozen, and scoring uses the unmodified encoder.

\paragraph{Semantic negation map.}
The semantic negation map is trained on
WordNet antonym pairs in both directions and on pole pairs
from the training scales. A pole pair is the mean of a latent
variable's positive item texts against the mean of its negative ones. The loss is a cosine mapping term and an involution term for the condition $f_{\mathrm{neg}}(f_{\mathrm{neg}}(x)) \approx x$.

\paragraph{Operator and weight network.}
The \emph{weight network} $\omega$ is a small network that emits per-variable weights for the terms of Eq.~\ref{eq:objective}, which the equation omits. With these weights the balance between terms can differ between a variable with many named neighbors and one with few. The operator and the weight network are fitted on the sixteen training scales for questionnaires and on generated arms and vehicles for robots. The parameters of the operator
network $\Delta$ and of the weight network $\omega$ are fitted through the
unrolled optimization with
\begin{equation}
\min_{\Delta, \omega}\;
\frac{1}{|O \cap V_{\mathrm{U}}|} \sum_{v \in O \cap V_{\mathrm{U}}} \big(1 - \cos(\hat e_v, e^\star_v)\big)
+ \frac{1}{2\,|L|} \sum_{v \in L} \big(1 - \cos(\hat e_v, e^\star_v)\big)
+ 0.1\, \mathcal{S}(\Delta),
\label{eq:solve}
\end{equation}
where $\hat e$ is the output of the unrolled optimization on a training
scale, arm or vehicle with its names masked as at evaluation, $e^\star_v$
is the encoding of the masked true name, and $\mathcal{S}$ is the
sign-consistency penalty of the operator (Appendix~\ref{apx:solver-operator}).

\subsection{Robot channels}
\label{apx:solver-robot}

On the robot channels the same objective runs, with the parts that carry vocabulary replaced by their robot counterparts. The encoder fine-tuning uses a second set of low-rank matrices, warm-started from the questionnaire ones and fitted on generated arms and vehicles. The concept bank is a bank of templated robot and vehicle phrases. The operator and the weight network are the pair fitted on the generated arms and vehicles (Table~\ref{tab:data-audit}).

\subsection{Conditional mutual information as the dependence measure}
\label{apx:cmi-objective}

The objective can also take the CMI as its dependence measure, with $\psi$ the coefficient of Proposition~\ref{prop:cmi}. In this form Eq.~\ref{eq:objective} keeps its terms and its pair sets and changes in two places. First, every dependence value that the objective reads becomes a tested and signed $\psi$ of an estimated CMI,
\begin{equation}
\tau_{ab \cdot S} = \xi_{ab \cdot S}\; \psi\big(\hat I(a; b \mid S)\big)\; \mathbf{1}\big[\hat I(a; b \mid S) > \beta_{|S|}\big],
\qquad \tau_{ab} = \tau_{ab \cdot \emptyset},
\label{eq:cmi-target}
\end{equation}
where $\hat I$ is the estimate described below, $\beta_{|S|}$ the threshold of its test for a conditioning set of size $|S|$ and $\xi_{ab \cdot S} \in \{-1, +1\}$ a sign. The sign is that of the Spearman correlation of $a$ and $b$ for an empty $S$, and that of their residuals after a nearest-neighbor regression on $S$ otherwise. The target on $\mathcal{P}_{\mathrm{ind}}$, the bound on $\mathcal{P}_{\mathrm{dep}}$ and the edge strength $\rho_{uv}$ of the operator of Eq.~\ref{eq:residual} read $\tau_{ab}$ in place of the distance correlation. The weights of the neighbor mean $c_v$ of Eq.~\ref{eq:reg} and of the refinement read it in place of the Pearson correlation.

Second, the conditional dependence term compares the residuals of the vectors given one conditioning set per pair,
\begin{equation}
\frac{\lambda_{\mathrm{con}}}{|\mathcal{P}_{\mathrm{con}}|} \sum_{(a,b) \in \mathcal{P}_{\mathrm{con}}} \big( \cos(e_{a \cdot S_{ab}}, e_{b \cdot S_{ab}}) - \tau_{ab \cdot S_{ab}} \big)^2,
\qquad
S_{ab} = \big(\mathrm{pa}(a) \cup \mathrm{pa}(b)\big) \setminus \{a, b\} .
\label{eq:cmi-con}
\end{equation}
By Proposition~\ref{prop:cmi}, the absolute value of the first argument is $\psi$ of the CMI of the vectors given $S_{ab}$, and that of the target is $\psi$ of the CMI of the measurements given $S_{ab}$. The target of the main method compares regression residuals, which are the noise terms only when the noise enters additively (Appendix~\ref{apx:theory-objective}). The CMI given $S_{ab}$ measures the dependence given $S_{ab}$ directly, with no premise on how the noise enters. Every term keeps its form, and the anchors and the trained components are those of the main method.

\paragraph{The estimator.}
The CMI of the measurements is estimated by CCMI \citep{CCMI}, with the networks and the training settings of the source. For any conditional law $q(b \mid S)$,
\begin{equation}
I(a; b \mid S) = D_{\mathrm{KL}}\big(p(a, b, S) \,\|\, p(a, S)\, q(b \mid S)\big) - D_{\mathrm{KL}}\big(p(b, S) \,\|\, p(S)\, q(b \mid S)\big),
\label{eq:ccmi}
\end{equation}
since the difference of the two divergences is $\mathbb{E} \log \big(p(b \mid a, S) / p(b \mid S)\big)$. The rows split into two halves. For each row of the first half, $b$ is taken from the row of the second half whose value of $S$ is nearest, which gives samples of $q(b \mid S)$. A classifier separates each pair of samples, and its likelihood ratio enters the Donsker-Varadhan bound of the divergence. The nearest-neighbor draw only approximates $p(b \mid S)$, and Eq.~\ref{eq:ccmi} holds for any $q$, so the second divergence removes the error of the draw from the first. For an empty $S$, the estimate is the divergence between $p(a, b)$ and $p(a)\, p(b)$, with $b$ permuted over the rows.

\paragraph{Ranks and latent codes.}
Every variable is replaced by its ranks before the estimation, and the estimates use at most 2{,}000 rows of a system. The CMI does not change under an invertible transformation of each argument, so the ranks leave it unchanged. On a questionnaire, a latent variable enters through a one-dimensional code of its observed descendants, which a denoising autoencoder \citep{DAE} computes. By the same invariance, a code that recovers the latent variable up to an invertible transformation suffices. A latent code is computed without the columns that the other parts of the same estimate read, so that no dependence comes from a shared column. An estimate in which a latent variable keeps no column of its own gives no target.

\paragraph{Tests.}
A value enters Eq.~\ref{eq:cmi-target} only when a test against surrogates rejects independence at the 5\% level. The threshold $\beta_{|S|}$ is the 95th percentile of the estimator on surrogate estimates with conditioning sets of the same size, drawn from the pairs of the same system. For a nonempty $S$, a surrogate permutes $a$ among rows that are near in $S$, by the local permutation of \citet{Runge2018}. This removes the dependence of $a$ and $b$ given $S$ and keeps the dependence of $a$ on $S$. The rows of an arm, a log or a drive are serial, and the estimator respects their order through a window of $w$ rows. This $w$ is the Theiler window \citep{Theiler1986}, the largest lag at which the median absolute autocorrelation of the signals is at least one half, capped at a tenth of the rows. For $w > 0$, the halves of the estimator are contiguous and separated by a gap of $w$ rows, and a nearest neighbor is never taken within $w$ rows of its row. A surrogate for an empty $S$ then shifts $a$ circularly by at least $w + 1$ rows and at least a tenth of the rows, which keeps its serial correlation. For $w = 0$, as on the exchangeable rows of a questionnaire, a surrogate for an empty $S$ permutes $a$ over all rows.

\section{Translation via Prefix Adaptation}
\label{apx:plugin}

The adapter receives the vector $\hat e_v$ of every unnamed variable from the optimization of Appendix~\ref{apx:solver} and passes it to a frozen language model as a prefix. The prompt gives the neighborhood of $v$ in $\hat G$ as text, and the model writes the name. The prefix enters the KV cache of every attention layer, so the model attends to $\hat e_v$ at every step of writing the name. In contrast, a soft prompt at the input layer \citep{ClipCap} reaches the later layers only through the hidden states. Since $\hat e_v$ lies in the embedding space of the optimization, which also holds the names, one adapter serves every questionnaire scale and one every arm, log and drive.

\subsection{Inputs}
\label{apx:plugin-inputs}

The adapter input $z_v$ of Section~\ref{sec:method-translation} concatenates five vectors $x_1, \ldots, x_5$, each normalized when nonzero, and a vector $s_v$ of a few scalars. For an observed variable, $x_1$ is its vector $\hat e_v$, $x_2$ the vector of its latent parent and $x_3$ the component of $x_1$ orthogonal to $x_2$. The vector $x_4$ is the correlation-weighted mean of the vectors of the named variables most correlated with $v$. The vector $x_5$ is that mean with its negatively correlated members passed through the semantic negation map $f_{\mathrm{neg}}$. An observed variable with no latent parent in the graph supplies its own vector as $x_2$. The discovered graph leaves part of the items of a scale without a latent parent (Table~\ref{tab:struct-q}). For a latent variable, $x_1$ is the loading-weighted mean of the vectors of its children and $x_2$ is its own vector $\hat e_v$. The means of $x_4$ and $x_5$ are then taken over its children.

The scalars $s_v$ record the position of $v$ among its neighbors. They include its correlation to its most correlated named variables and its signed loading on the latent parent. They also include the share of its siblings that are unnamed and the cosine similarity between its vector and the neighbor mean. A robot channel keeps the five vectors, with $x_5$ set to zero since a channel carries no keying. Its scalars describe its parents and children, its lagged correlations and its role as an actuator or a sensor.

\subsection{Architecture}
\label{apx:plugin-arch}

The adapter maps $z_v$ to the prefix $P_v$ and the prefix to the KV cache of every layer $j$ by
\begin{equation}
\begin{aligned}
z_v &= [\,x_1;\; \ldots;\; x_5;\; s_v\,], \qquad
\gamma = \sigma\big(q_\gamma(s_v, \|x_1\|, \ldots, \|x_5\|)\big), \\
P_v &= q(z_v) = \big[\,\gamma_1\, q_1(x_1, s_v);\; \ldots;\; \gamma_5\, q_5(x_5, s_v)\,\big] \in \mathbb{R}^{m \times d_p}, \\
\mathbf{K}_j &= P_v W^K_j, \qquad \mathbf{V}_j = P_v W^V_j
\qquad \text{for every layer } j.
\end{aligned}
\label{eq:adapter}
\end{equation}
Each input vector $x_i$, with the scalars appended, passes through its own two-layer network $q_i$ into a $2 \times d_p$ block of two prefix vectors, so the prefix length is $m = 10$. A learned position embedding is added to the prefix. The gate network $q_\gamma$ computes one sigmoid gate $\gamma_i$ per input vector from the scalars and from the norms $\|x_i\|$, which are one for a present vector and zero for an absent one. The gates are initialized near one, and they raise the weight of the prefix where the neighborhood has few named variables. The matrices $W^K_j$ and $W^V_j$ of layer $j$ are learned and shared by every variable.

\subsection{Prompt and evidence line}
\label{apx:plugin-prompt}

The prompt text $t_v$ is a fixed instruction, the incident edges of $v$ and the evidence line. The instruction states what the prefix carries and asks for one short name. Each variant receives an instruction whose wording is true of its own input. The edge block is rendered from the graph with one line per incident edge.

\begin{quote}\small\ttfamily
CAUSAL NEIGHBORS:\\
- CAUSE: an unnamed latent factor also causes: "$\langle$named sibling$\rangle$"; \ldots\\
- CAUSE: "$\langle$named parent$\rangle$" influences it\\
- EFFECT: it influences "$\langle$named child$\rangle$"\\
- EFFECT: it causes the item "$\langle$named child, or (label hidden)$\rangle$"
\end{quote}

An item shows at most six named siblings and a latent variable at most
twenty named children. An unnamed neighbor of a latent variable appears as the literal string \texttt{(label hidden)}, so the prompt shows how much of the group is named. The identifier of the target
never appears.

The evidence line reports the strongest signed correlations with named variables and the membership and keying of the variable within its group. It also reports the nearest named neighbors of the variable in the embedding space with their cosine similarities and its partial correlations given the group. Last, it gives the sibling that shares the most variance with the variable beyond the group. The robot
form reports lead and lag relations with signed correlations. The evidence
line also carries the nearest bank phrases decoded from the variable's
own vector $\hat e_v$ and from the vectors of its unnamed neighbors.
Where a neighborhood carries no name at all, those phrases are the only
information left.

\subsection{Training}
\label{apx:plugin-training}

The language model stays frozen. A low-rank update is inserted into the query and value projections of every layer, and only that update and the adapter are optimized, under the loss
\begin{equation}
\mathcal{L}_{\mathrm{adapter}} = \mathcal{L}_{\mathrm{NLL}}
+ \lambda_{\mathrm{sim}}\,\big(1 - \cos(\bar\phi_{\mathrm{LM}}(\hat y), \bar\phi_{\mathrm{LM}}(\ell(v)))\big)
+ \lambda_{\mathrm{ctr}}\, \mathcal{L}_{\mathrm{ctr}}
+ \lambda_{\mathrm{sib}}\, \mathcal{L}_{\mathrm{sib}},
\label{eq:plugin-loss}
\end{equation}
where $\mathcal{L}_{\mathrm{NLL}}$ is the negative log-likelihood of the
true name at its token positions. The second term is the semantic similarity loss of Section~\ref{sec:method-translation}. Here $\bar\phi_{\mathrm{LM}}(\hat y)$ is the mean over those positions of the expected token embedding under the distribution that the model predicts for the name $\hat y$. The expected embedding is taken under the token embedding table $\phi_{\mathrm{LM}}$ of the model. The term pulls $\bar\phi_{\mathrm{LM}}(\hat y)$ toward the mean embedding of the tokens of the true name in the same table. The
contrastive loss $\mathcal{L}_{\mathrm{ctr}}$ is a softmax cross-entropy
over the cosine similarities between that expected embedding and the
embeddings of the true name and of a few sibling names. The sibling loss $\mathcal{L}_{\mathrm{sib}} = -\log(1 - p_{\mathrm{sib}})$ penalizes the
probability mass $p_{\mathrm{sib}}$ that a token position of the name
places on tokens unique to a sibling name. The last two push the model
toward a name that separates the variable from its siblings.

One adapter is trained for the questionnaires and one for the robots, on the data of Appendix~\ref{apx:data-roles} and with the same recipe. Only the shape of the KV cache enters the adapter, so the recipe is also the same for every backbone. On the questionnaires, sets of masked names are drawn at the four masking levels of the evaluation, half under the true graph and half under a degraded edge set. Each generated arm or vehicle enters once under its true graph and once under an edge set degraded to the quality of a discovered graph. The adapter is thus trained on the kind of graph error that the discovery step makes. This design applies domain randomization \citep{Tobin2017} to the structure.

The examples are also corrupted, so that the adapter learns to use both the prefix and the text. Of the 28{,}014 base examples of the questionnaires, the prompt text is removed in 5{,}590 and the evidence line in 5{,}719. An unrelated variable's vector replaces $z_v$ in 3{,}420 examples, and a named sibling's vector replaces it in 1{,}565. A further 5{,}059 examples repeat an item whose sibling is also unnamed, with the evidence line removed. Latent targets are drawn more often than item targets, and training stops early on a validation split of the synthetic scales and of the generated arms and vehicles.

\subsection{Decoding}
\label{apx:plugin-decoding}

At inference the adapter maps $z_v$ to the prefix, and the model decodes greedily from the prefix and the prompt text up to a short budget of new tokens. The first line of the output is taken as the name.

\section{Evaluation Protocol}
\label{apx:protocol}

The evaluation scores the name a method writes for an unnamed variable against a target. NRR@1 and MRR rank the name in an embedding space, and identification and naming quality ask a language-model judge.

\subsection{Reference for discovered latent variables}
\label{apx:latent-ref}

A discovered latent variable is given a reference by majority over the observed items beneath it. It carries no name of its own, since a rank test introduces it and only the items beneath it identify it. Every observed item that it reaches in the graph, directly or through nested latent variables, is looked up in the published scoring key. The published latent variable that owns most of those items becomes the target against which the written name is scored. Every discovered latent variable reaches at least one observed item. On four scales the reference owns most of the items beneath a discovered latent variable. On the World Values Survey, whose discovered groups cut across the Welzel sub-indices, it owns about half of them (Table~\ref{tab:struct-q}).

The rank tests introduce a latent variable from the rank deficiency of a cross-covariance block, so its identity is fixed only up to that block (Theorem~\ref{thm:rlcd}). Factor analysis uses the same convention and matches an estimated factor to a reference factor by its loading pattern \citep{LorenzoSeva2006}. The majority rule is the discrete form of that match.

\subsection{Scoring space, candidate pool and metrics}
\label{apx:protocol-metrics}

\paragraph{Scoring space.}
Every method is scored from the text it writes, in one frozen public embedding space shared by every method. We use \texttt{intfloat/e5-large-v2} at a pinned revision and rank by cosine similarity. The encoder's \texttt{query:} prefix is applied to both the written name and the candidates.

\paragraph{Candidate pool.}
The pool is every name of the same arm, log, drive or scale. For a robot channel it is every channel name of that arm, log or drive. For an item it is every item text of the scale, and for a latent variable every latent variable name of its key. Because the known names stay in the pool, a
method that copies a known name scores zero on NRR@1 by
construction. Under the discovered graph, the targets are the discovered latent variables, each assigned to a latent variable of the key by the majority rule above. The pool stays the full list of the key. A tie takes the mean rank of the tied positions, and an empty
output ranks last.

\paragraph{Metrics.}
NRR@1 stands for the recall at rank 1 of name retrieval. It is the rate at which the true name is ranked first among the candidates. MRR is the mean reciprocal rank of the true name. The identification metric asks a judge to pick the target out
of five candidates and credits the case only when the judge picks it in
both presentation orders. Naming quality, reported for latent variables, adds two requirements to identification. One requires the output to have the form of a name, and one rejects an output that repeats an item text verbatim. Every cell is the mean over the cases of its scales, arms, logs or drives, so a scale with more unnamed items weighs more. Every questionnaire cell is a mean over folds.

\paragraph{Auxiliary metrics.}
Three auxiliary metrics are reported in Appendix~\ref{apx:tables}. Sibling discrimination accuracy (SDA) is computed over the variables that have a sibling. It counts the case when the written name lies closer in the scoring space to the true name than to every sibling of the variable. The siblings are the published siblings of an item or the channels of the same type on an arm, log or drive. The metric therefore measures whether a name that reaches the right latent variable or channel type also reaches the right item or channel.
Robot channels report two further metrics. Type
accuracy assigns the written name a channel type and counts the case when
that type is the type of the true name. The type of a true name follows from its documented wording. A channel that a platform names after an arm joint, such as the waist or the elbow, is a joint channel. A written
name that equals a
documented channel name takes the type of that name. Any other written name is embedded in the scoring space and takes the nearest of 21 fixed channel-type descriptions. These descriptions range from the angle of one arm joint to the rotation speed of one wheel. Coordinate accuracy counts the case when the
digit, axis letter or side word in the written name equals the one in the
true name, over the names that carry one.

\subsection{The judge}
\label{apx:protocol-judge}

The judge is Gemma 3 12B in its instruction-tuned release \citep{Gemma3},
served locally at temperature 0. For one case it receives the written name
and five candidates, the target among them. The four distractors for an
item are two published siblings of the item, drawn from the same
latent variable, and two other items of the same scale. The distractors for a latent variable are up to four sibling latent variables of the same scale. A scale with fewer siblings fills the remaining slots with latent variables of other scales. Under the discovered graph the siblings are the
latent variables of the key that the discovered groups reach. For a robot channel, two distractors are the channel names of the same arm, log or drive that share the most words with the true name. The other two are further channel names of that arm, log or drive. A name that fits only the latent variable or the channel type therefore competes with the siblings. On items a description that names only the latent variable is right about one time in three. Every case is judged twice, once in a seeded candidate order and once in the reverse order \citep{MTBench, NotFairEvaluators}, which controls for the position bias of a language-model judge. A case is kept only when both responses parse.

\paragraph{The two requirements of naming quality.}
The form check asks the same judge whether the output names a concept or quotes a specific statement, and only a name passes.
The verbatim check runs offline and overrides the judge. The check lowercases the output, collapses whitespace and removes a final period. An output that then equals an item text of the same scale is rejected whatever the judge said. Together the two checks remove the degenerate solution of repeating a neighbor. Identification alone can credit such a repeat when the neighbor belongs to the same latent variable.

\subsection{Variation across folds}
\label{apx:protocol-variation}

The spread between folds is wide and grows as names are removed, since a
high masking level leaves fewer anchors and each one weighs more. On DASS latent variables under the discovered graph, identification is $.689 \pm .183$ at 20\% masking and $.222 \pm .208$ at 90\%. The $\pm$ term is the standard deviation across the five folds of nine discovered targets each.
On RIASEC the same quantities are $.800 \pm .088$ and $.369 \pm .084$ over
thirteen discovered targets per fold. Differences between methods within one fold are
therefore better evidence than differences between cells, which is why the comparisons in this paper are same-fold and same-target. The one exception is latent variables with no structure supplied, where \ours{} names the latent variables it discovers (Appendix~\ref{apx:j-graphs}). The band around our curves in every result figure is one standard error over replicates. It equals the sample standard deviation across the replicates divided by the square root of their number. On a questionnaire column a replicate is one fold, with the five
scales pooled by sample inside it, so the band rests on five values. On a
robot column a replicate is one arm, log or drive, so the band rests on the six
simulated arms, the three real logs or the ten drives. The value of one arm, log or drive at
20\% masking rests on a few unnamed channels, which is why the robot
bands are widest there. The bands are clipped to the unit interval.
Appendix~\ref{apx:tables} reports the paired differences to the strongest
baseline on these replicates.

\section{Baselines and Fidelity to the Published Designs}
\label{apx:roster}

The comparison covers the three ways a variable has been named from its
data. A method selects a concept from a fixed bank, asks a language model,
or regresses into the embedding space of the names. With a structure
supplied the comparison adds two methods that consume the graph.

\begin{itemize}
\item \textbf{CLIP-Dissect-e5} \citep{CLIPDissect} selects the concept from a fixed bank whose
activation pattern best matches the variable.
\item \textbf{Automated Interpretability} \citep{autointerpBill} prompts a language model to explain the variable from a sample of its cases and then scores the explanation by simulation.
\item \textbf{Delphi} \citep{Delphi} explains the variable from a sample of its activating cases and adds detection scoring passes.
\item \textbf{SASC-style} \citep{SASC} prompts a language model for a short explanation of
the variable's response profile.
\item \textbf{DeViSE-style} \citep{DeViSE} regresses the variable's correlation profile into
the embedding space of the names and decodes the predicted vector to its three
nearest entries in the concept bank of Appendix~\ref{apx:solver}. With fewer
than two named variables there is no regression to fit and the method
returns no name, which NRR@1 and MRR score as a miss.
\item \textbf{Base language model reading the graph} receives the same prompt as \ours{}, which holds the edges and the evidence line. It runs with no prefix and no low-rank update, so it isolates what the prefix and the low-rank update contribute.
\item \textbf{GraphMAE} \citep{GraphMAE} reconstructs the masked node representation by
masked graph autoencoding with a GCN encoder \citep{GraphMAEGCN} and decodes the reconstruction to the nearest known name. On a latent variable it decodes to the nearest named item text. A masked name is therefore never among its candidates.
\end{itemize}

With no structure supplied, every baseline names latent variables behind a shared factor-analysis front end, since naming a latent variable requires a grouping and the front end supplies it. With a structure supplied, each baseline takes its latent grouping from that graph, except SASC-style and DeViSE-style, which keep the factor-analysis grouping. The front end is given the list of latent variables of the
key and matches each to the factor whose item set overlaps it most. For a
parent-level latent variable of a nested key the items are those of its
child latent variables. A baseline therefore names every latent variable
of the key, and one factor can serve two of them. \ours{} names the
latent variables it discovers, referred to the key as
Appendix~\ref{apx:latent-ref} describes.

\paragraph{Name-copy references.}
Three references build a name from the known names alone, and Appendix~\ref{apx:tables} reports them apart from the roster above.
Correlation kNN returns the name of the most correlated named variable.
Feature Propagation \citep{FeaturePropagation} diffuses the vectors of the known names over a correlation graph that the port builds. It decodes the result to the nearest known name on an arm, log or drive and to the nearest phrases of the concept bank on a questionnaire. EFA loading naming returns the text of
the item that loads highest on the factor.

\paragraph{Common protocol.}
Every baseline that calls a language model calls the same pinned Qwen3-4B snapshot, served locally at seed 0 in bfloat16. This is the default backbone of our own method. The one exception is SASC-style, which samples its candidate explanations at temperature .7 and scores them deterministically. Every output of our method and of the baselines is ranked in the same frozen encoder and judged by the same judge in the same dual-order protocol. The cases of a baseline copy the fold, the set of masked names and the target list from ours. The exception is latent variables with no structure supplied, where a baseline names the latent variables of the key. Every method sees the known names of the fold only. A missing or refused output is a failure on NRR@1 and MRR and
is left out of the identification rate. Methods with no structure
supplied receive no edges and no grouping from the key. The factor
analysis supplies the grouping, and the key's item sets serve only to
assign each factor its scoring target.

\paragraph{Fidelity to the published designs.}
Each reimplementation is a port to a new domain.

\begin{itemize}\setlength{\itemsep}{1pt}
\item \textbf{Automated Interpretability} keeps the explain-then-simulate
structure, standardized activations, simulation on a disjoint held-out set
and correlation scoring. The examples are respondent profiles or robot windows, where the original uses token windows. They are drawn across the activation range, where the original explains from top-activating snippets.
\item \textbf{Delphi} uses the contrastive explainer of the library, which explains top-activating examples against hard negatives. The hard negatives here are the response-space nearest neighbors of the strong examples. The port keeps detection scoring on balanced held-out profiles. The port scores three candidate explanations jointly in one
detector request, with disjoint validation and test stages.
\item \textbf{CLIP-Dissect-e5} keeps SoftWPMI and the rank-reorder scoring
at the source defaults. Probe images become profile texts built from the
known names of the fold, and CLIP \citep{CLIP} becomes the frozen e5 encoder, which
the table name records. The two similarity functions are combined by rank averaging, since their scales differ.
\item \textbf{SASC-style} keeps the two phases, in which it proposes short candidate explanations and scores each against data. Candidates come from anchor
correlations, where the original uses n-gram summaries. The score is sign
agreement between predicted and measured correlations, since the original
scorer runs the module on new inputs and an unnamed survey item has no new
respondents. This is the largest structural change among the language
model baselines, and the table name records it.
\item \textbf{GraphMAE} keeps masked feature reconstruction with
learnable mask tokens, re-mask decoding and the paper's scaled cosine error.
The encoder is a GCN, where the paper defaults to GAT \citep{GAT}. A latent
variable has no node text to reconstruct, so its prediction is the
decoder output of its neighbor-conditioned state.
\item \textbf{DeViSE-style} keeps only the mapping of a variable into
the embedding space of the names and the decoding by nearest neighbors. Ridge
regression on correlation features replaces the deep visual encoder and its
hinge ranking loss, so the training signal, the model class and the loss all
differ. The table name records the change.
\item The \textbf{factor-analysis front end} is our own component.
It uses parallel analysis \citep{Horn1965} for the number of factors, varimax rotation \citep{Kaiser1958} and
argmax loading membership.
\end{itemize}

The scored output of Automated Interpretability and Delphi is a short
name of one to four content words, which our interface requires.
Their one-sentence explanation is stored, since ranking a sentence
against a set of names would measure something else.

\section{Supplemental Experimental Results}
\label{apx:tables}

This appendix supplements the first three research questions of
Section~\ref{sec:experiments} in their order.
Appendices~\ref{apx:j-lead} to~\ref{apx:j-checks} belong to RQ1. They give the lead of our method on every metric, the comparison of the discovered graph with the supplied one, and the same lead replicate by replicate and scale by scale. Appendix~\ref{apx:j-components} belongs to RQ2 and takes
the components in turn. Appendix~\ref{apx:j-backbones} belongs to
RQ3. The complete tables close the appendix.

\subsection{RQ1: the lead with no structure supplied}
\label{apx:j-lead}

The finding of RQ1 is that \ours{} leads on every column and that the lead is largest where the names give a language prior the least information.
Tables~\ref{tab:st1-robot} to~\ref{tab:st1-latents} give every metric behind that claim beside the three metrics of Table~\ref{tab:st1-main}. These are the sibling discrimination of items and channels, the type and coordinate accuracy of channels and the naming quality of latent variables. At 50 and 75\% masking on the three robot columns, the NRR@1 of our method is 2.2 to 5.1 times that of the strongest baseline. Its identification rate is 1.6 to 2.9 times and its MRR 1.4 to 2.1 times that of the strongest baseline.
\begin{table}[t]\centering\scriptsize\setlength{\tabcolsep}{2.4pt}
\caption{No structure supplied, \ours{} on the graph it discovers. Bold is the best value in a column within a set of variables and underline the second best. A dash is a method that returned no name at that level. Appendix~\ref{apx:tables} carries every metric each set of variables reports.}
\label{tab:st1-main}
\fittable{llcccc:cccc:cccc}{%
\toprule
 &  & \multicolumn{4}{c:}{NRR@1} & \multicolumn{4}{c:}{MRR} & \multicolumn{4}{c}{identification} \\
variables & method & 20 & 50 & 75 & 90 & 20 & 50 & 75 & 90 & 20 & 50 & 75 & 90 \\
\midrule
simulated arms & \cellcolor{oursrow}\ours{} & \cellcolor{oursrow}\textbf{.083} & \cellcolor{oursrow}\textbf{.140} & \cellcolor{oursrow}\textbf{.137} & \cellcolor{oursrow}\textbf{.046} & \cellcolor{oursrow}\textbf{.224} & \cellcolor{oursrow}\textbf{.286} & \cellcolor{oursrow}\textbf{.278} & \cellcolor{oursrow}\textbf{.194} & \cellcolor{oursrow}\textbf{.106} & \cellcolor{oursrow}\textbf{.203} & \cellcolor{oursrow}\textbf{.208} & \cellcolor{oursrow}\textbf{.116} \\
 & CLIP-Dissect-e5 & \underline{.083} & \underline{.058} & \underline{.027} & .027 & \underline{.185} & \underline{.186} & .116 & \underline{.120} & \underline{.062} & .050 & .033 & \underline{.060} \\
 & AutoInterp & .000 & .008 & .011 & .005 & .078 & .095 & .095 & .086 & .000 & .033 & .039 & .033 \\
 & Delphi & .000 & .025 & .016 & .005 & .079 & .120 & .119 & .083 & \underline{.062} & .034 & \underline{.072} & .028 \\
 & SASC-style & .021 & \underline{.058} & \underline{.027} & .018 & .112 & .162 & .130 & .118 & .042 & \underline{.074} & .060 & .028 \\
 & DeViSE-style & .021 & .033 & .022 & \underline{.032} & .143 & .136 & \underline{.131} & .116 & .000 & .025 & .016 & .005 \\
\midrule
real logs & \cellcolor{oursrow}\ours{} & \cellcolor{oursrow}\textbf{.417} & \cellcolor{oursrow}\textbf{.296} & \cellcolor{oursrow}\textbf{.195} & \cellcolor{oursrow}\textbf{.120} & \cellcolor{oursrow}\textbf{.485} & \cellcolor{oursrow}\textbf{.385} & \cellcolor{oursrow}\textbf{.314} & \cellcolor{oursrow}\textbf{.248} & \cellcolor{oursrow}\textbf{.417} & \cellcolor{oursrow}\textbf{.296} & \cellcolor{oursrow}\textbf{.244} & \cellcolor{oursrow}\textbf{.180} \\
 & CLIP-Dissect-e5 & \underline{.167} & .074 & .049 & .080 & \underline{.305} & \underline{.258} & .205 & .226 & .000 & .037 & \underline{.100} & .080 \\
 & AutoInterp & .083 & .037 & .049 & .020 & .265 & .200 & .178 & .152 & \underline{.250} & .111 & .098 & .041 \\
 & Delphi & .000 & .037 & \underline{.073} & \underline{.120} & .163 & .212 & \underline{.215} & \underline{.228} & \underline{.250} & \underline{.120} & .024 & .100 \\
 & SASC-style & .000 & \underline{.111} & .024 & .060 & .145 & .223 & .138 & .184 & .083 & .115 & .073 & .080 \\
 & DeViSE-style & .083 & .074 & .024 & .000 & .216 & .188 & .178 & .084 & .083 & .074 & .098 & \underline{.120} \\
\midrule
vehicle CAN & \cellcolor{oursrow}\ours{} & \cellcolor{oursrow}\textbf{.800} & \cellcolor{oursrow}\textbf{.367} & \cellcolor{oursrow}\textbf{.244} & \cellcolor{oursrow}\textbf{.218} & \cellcolor{oursrow}\textbf{.859} & \cellcolor{oursrow}\textbf{.570} & \cellcolor{oursrow}\textbf{.412} & \cellcolor{oursrow}\textbf{.397} & \cellcolor{oursrow}\textbf{.800} & \cellcolor{oursrow}\textbf{.400} & \cellcolor{oursrow}\textbf{.300} & \cellcolor{oursrow}\textbf{.282} \\
 & CLIP-Dissect-e5 & .000 & .067 & \underline{.078} & .091 & .135 & .229 & \underline{.298} & \underline{.286} & .000 & .050 & .089 & .118 \\
 & AutoInterp & .000 & .000 & .000 & .000 & .184 & .245 & .171 & .191 & \underline{.050} & .017 & .000 & .009 \\
 & Delphi & .000 & .017 & .067 & \underline{.118} & .195 & .292 & .241 & .281 & .000 & .067 & .033 & .055 \\
 & SASC-style & .000 & \underline{.167} & .044 & .109 & \underline{.276} & \underline{.366} & .223 & .268 & .000 & \underline{.250} & \underline{.146} & \underline{.191} \\
 & DeViSE-style & .000 & .000 & .000 & .000 & .180 & .173 & .141 & .083 & \underline{.050} & .067 & .022 & -- \\
\midrule
questionnaire items & \cellcolor{oursrow}\ours{} & \cellcolor{oursrow}\textbf{.121} & \cellcolor{oursrow}\textbf{.098} & \cellcolor{oursrow}\textbf{.059} & \cellcolor{oursrow}\textbf{.028} & \cellcolor{oursrow}\textbf{.253} & \cellcolor{oursrow}\textbf{.216} & \cellcolor{oursrow}\textbf{.164} & \cellcolor{oursrow}\textbf{.105} & \cellcolor{oursrow}\textbf{.356} & \cellcolor{oursrow}\textbf{.364} & \cellcolor{oursrow}\underline{.283} & \cellcolor{oursrow}\underline{.217} \\
 & CLIP-Dissect-e5 & .028 & .023 & .022 & .014 & .096 & .095 & .099 & .071 & .163 & .179 & .183 & .153 \\
 & AutoInterp & .010 & .012 & .008 & .007 & .058 & .055 & .054 & .054 & .085 & .103 & .101 & .099 \\
 & Delphi & .008 & .010 & .013 & .011 & .060 & .055 & .060 & .063 & .141 & .122 & .133 & .113 \\
 & SASC-style & \underline{.079} & \underline{.062} & \underline{.035} & \underline{.021} & \underline{.211} & \underline{.175} & \underline{.128} & \underline{.087} & \underline{.356} & \underline{.342} & \textbf{.303} & \textbf{.219} \\
 & DeViSE-style & .033 & .023 & .014 & .008 & .131 & .110 & .087 & .064 & .238 & .227 & .193 & .164 \\
\midrule
latent variables & \cellcolor{oursrow}\ours{} & \cellcolor{oursrow}\textbf{.540} & \cellcolor{oursrow}\textbf{.486} & \cellcolor{oursrow}\textbf{.386} & \cellcolor{oursrow}\textbf{.240} & \cellcolor{oursrow}\textbf{.660} & \cellcolor{oursrow}\textbf{.623} & \cellcolor{oursrow}\textbf{.560} & \cellcolor{oursrow}\textbf{.424} & \cellcolor{oursrow}\textbf{.531} & \cellcolor{oursrow}\textbf{.511} & \cellcolor{oursrow}\textbf{.441} & \cellcolor{oursrow}\textbf{.254} \\
 & CLIP-Dissect-e5 & .073 & .119 & .127 & \underline{.142} & .257 & .288 & .298 & \underline{.292} & .128 & .120 & .151 & .121 \\
 & AutoInterp & .096 & .100 & .108 & .088 & .258 & .259 & .270 & .252 & .160 & .140 & .170 & .139 \\
 & Delphi & .123 & .108 & .085 & .069 & .266 & .269 & .251 & .235 & .138 & .108 & .100 & .115 \\
 & SASC-style & \underline{.369} & \underline{.315} & \underline{.204} & .115 & \underline{.523} & \underline{.479} & \underline{.393} & .291 & \underline{.502} & \underline{.424} & \underline{.319} & \underline{.204} \\
 & DeViSE-style & .138 & .138 & .108 & .112 & .291 & .281 & .260 & .266 & .151 & .154 & .123 & .112 \\
\bottomrule}
\end{table}

\begin{table}[!htb]\centering\scriptsize\setlength{\tabcolsep}{2.4pt}
\caption{Robot channels with no structure supplied, \ours{} on the graph it discovers. The table gives every metric a robot channel reports, which are the three of Table~\ref{tab:st1-main}, sibling discrimination, type accuracy and coordinate accuracy (Appendix~\ref{apx:protocol}).}
\label{tab:st1-robot}
\fittable{llcccc:cccc:cccc}{%
\toprule
 &  & \multicolumn{4}{c:}{NRR@1} & \multicolumn{4}{c:}{MRR} & \multicolumn{4}{c}{SDA} \\
variables & method & 20 & 50 & 75 & 90 & 20 & 50 & 75 & 90 & 20 & 50 & 75 & 90 \\
\midrule
simulated arms & \ours{} & \textbf{.083} & \textbf{.140} & \textbf{.137} & \textbf{.046} & \textbf{.224} & \textbf{.286} & \textbf{.278} & \textbf{.194} & .128 & .190 & .153 & .164 \\
 & CLIP-Dissect-e5 & \underline{.083} & \underline{.058} & \underline{.027} & .027 & \underline{.185} & \underline{.186} & .116 & \underline{.120} & \underline{.234} & .155 & .142 & .169 \\
 & AutoInterp & .000 & .008 & .011 & .005 & .078 & .095 & .095 & .086 & \textbf{.255} & \underline{.207} & .165 & .155 \\
 & Delphi & .000 & .025 & .016 & .005 & .079 & .120 & .119 & .083 & .043 & \textbf{.241} & \textbf{.193} & \textbf{.197} \\
 & SASC-style & .021 & \underline{.058} & \underline{.027} & .018 & .112 & .162 & .130 & .118 & .170 & .147 & .165 & \underline{.183} \\
 & DeViSE-style & .021 & .033 & .022 & \underline{.032} & .143 & .136 & \underline{.131} & .116 & .170 & .198 & \underline{.170} & .178 \\
\midrule
real logs & \ours{} & \textbf{.417} & \textbf{.296} & \textbf{.195} & \textbf{.120} & \textbf{.485} & \textbf{.385} & \textbf{.314} & \textbf{.248} & \textbf{.455} & \textbf{.360} & \textbf{.359} & \underline{.255} \\
 & CLIP-Dissect-e5 & \underline{.167} & .074 & .049 & .080 & \underline{.305} & \underline{.258} & .205 & .226 & .273 & .160 & .128 & .234 \\
 & AutoInterp & .083 & .037 & .049 & .020 & .265 & .200 & .178 & .152 & \underline{.364} & \underline{.240} & \underline{.205} & .170 \\
 & Delphi & .000 & .037 & \underline{.073} & \underline{.120} & .163 & .212 & \underline{.215} & \underline{.228} & .182 & .200 & .179 & \textbf{.298} \\
 & SASC-style & .000 & \underline{.111} & .024 & .060 & .145 & .223 & .138 & .184 & .091 & \underline{.240} & \underline{.205} & .234 \\
 & DeViSE-style & .083 & .074 & .024 & .000 & .216 & .188 & .178 & .084 & .273 & .200 & .179 & .064 \\
\midrule
vehicle CAN & \ours{} & \textbf{.800} & \textbf{.367} & \textbf{.244} & \textbf{.218} & \textbf{.859} & \textbf{.570} & \textbf{.412} & \textbf{.397} & \textbf{.800} & \underline{.367} & \textbf{.500} & \textbf{.367} \\
 & CLIP-Dissect-e5 & .000 & .067 & \underline{.078} & .091 & .135 & .229 & \underline{.298} & \underline{.286} & \underline{.300} & .167 & .150 & .233 \\
 & AutoInterp & .000 & .000 & .000 & .000 & .184 & .245 & .171 & .191 & .100 & .100 & .100 & .000 \\
 & Delphi & .000 & .017 & .067 & \underline{.118} & .195 & .292 & .241 & .281 & .100 & .133 & .300 & .000 \\
 & SASC-style & .000 & \underline{.167} & .044 & .109 & \underline{.276} & \underline{.366} & .223 & .268 & .000 & \textbf{.433} & \underline{.500} & \underline{.333} \\
 & DeViSE-style & .000 & .000 & .000 & .000 & .180 & .173 & .141 & .083 & .000 & .333 & .000 & .000 \\
\bottomrule}
\par\vspace{3pt}
\fittable{llcccc:cccc:cccc}{%
\toprule
 &  & \multicolumn{4}{c:}{identification} & \multicolumn{4}{c:}{type acc} & \multicolumn{4}{c}{coord acc} \\
variables & method & 20 & 50 & 75 & 90 & 20 & 50 & 75 & 90 & 20 & 50 & 75 & 90 \\
\midrule
simulated arms & \ours{} & \textbf{.106} & \textbf{.203} & \textbf{.208} & \textbf{.116} & \textbf{.396} & \textbf{.537} & \textbf{.577} & \textbf{.420} & \textbf{.106} & \textbf{.181} & \textbf{.159} & \textbf{.145} \\
 & CLIP-Dissect-e5 & \underline{.062} & .050 & .033 & \underline{.060} & .271 & .215 & .137 & .105 & \underline{.085} & .026 & .023 & .009 \\
 & AutoInterp & .000 & .033 & .039 & .033 & .000 & .107 & .126 & .114 & .000 & .009 & .034 & .000 \\
 & Delphi & \underline{.062} & .034 & \underline{.072} & .028 & .104 & .141 & .170 & .132 & .021 & .052 & \underline{.051} & \underline{.014} \\
 & SASC-style & .042 & \underline{.074} & .060 & .028 & .208 & \underline{.281} & .203 & \underline{.173} & .043 & \underline{.060} & .040 & \underline{.014} \\
 & DeViSE-style & .000 & .025 & .016 & .005 & \underline{.312} & .273 & \underline{.275} & \underline{.173} & .000 & .000 & .000 & .000 \\
\midrule
real logs & \ours{} & \textbf{.417} & \textbf{.296} & \textbf{.244} & \textbf{.180} & \textbf{.667} & \textbf{.518} & \textbf{.512} & \textbf{.460} & \textbf{.750} & \textbf{.545} & \textbf{.606} & \textbf{.579} \\
 & CLIP-Dissect-e5 & .000 & .037 & \underline{.100} & .080 & .083 & .000 & .073 & .060 & .000 & .045 & .061 & .026 \\
 & AutoInterp & \underline{.250} & .111 & .098 & .041 & \underline{.250} & .185 & \underline{.244} & \underline{.260} & .250 & .182 & .212 & .368 \\
 & Delphi & \underline{.250} & \underline{.120} & .024 & .100 & .167 & .074 & .098 & .240 & .000 & .136 & .151 & .026 \\
 & SASC-style & .083 & .115 & .073 & .080 & .083 & \underline{.222} & .146 & .180 & .250 & \underline{.409} & \underline{.333} & \underline{.395} \\
 & DeViSE-style & .083 & .074 & .098 & \underline{.120} & .083 & .111 & .195 & .000 & \underline{.500} & .227 & .303 & .316 \\
\midrule
vehicle CAN & \ours{} & \textbf{.800} & \textbf{.400} & \textbf{.300} & \textbf{.282} & \textbf{.900} & \textbf{.683} & \textbf{.356} & \textbf{.391} & \textbf{.800} & \textbf{.367} & \textbf{.500} & \textbf{.367} \\
 & CLIP-Dissect-e5 & .000 & .050 & .089 & .118 & .000 & .050 & .033 & .027 & .000 & .000 & .000 & .000 \\
 & AutoInterp & \underline{.050} & .017 & .000 & .009 & .000 & \underline{.400} & \underline{.200} & .273 & .000 & .000 & .000 & .000 \\
 & Delphi & .000 & .067 & .033 & .055 & .000 & .183 & .156 & .273 & .000 & .000 & .000 & .000 \\
 & SASC-style & .000 & \underline{.250} & \underline{.146} & \underline{.191} & \underline{.200} & .383 & .067 & \underline{.300} & .000 & \underline{.200} & .000 & \underline{.333} \\
 & DeViSE-style & \underline{.050} & .067 & .022 & -- & .000 & .000 & .000 & .000 & .000 & .000 & .000 & .000 \\
\bottomrule}
\end{table}

\begin{table}[!htb]\centering\scriptsize\setlength{\tabcolsep}{2.4pt}
\caption{Questionnaire items with no structure supplied, \ours{} on the graph it discovers, with sibling discrimination (Appendix~\ref{apx:protocol}).}
\label{tab:st1-items}
\fittable{llcccc:cccc:cccc:cccc}{%
\toprule
 &  & \multicolumn{4}{c:}{NRR@1} & \multicolumn{4}{c:}{MRR} & \multicolumn{4}{c:}{SDA} & \multicolumn{4}{c}{identification} \\
variables & method & 20 & 50 & 75 & 90 & 20 & 50 & 75 & 90 & 20 & 50 & 75 & 90 & 20 & 50 & 75 & 90 \\
\midrule
questionnaire items & \ours{} & \textbf{.121} & \textbf{.098} & \textbf{.059} & \textbf{.028} & \textbf{.253} & \textbf{.216} & \textbf{.164} & \textbf{.105} & \textbf{.226} & \textbf{.189} & \textbf{.164} & \textbf{.147} & \textbf{.356} & \textbf{.364} & \underline{.283} & \underline{.217} \\
 & CLIP-Dissect-e5 & .028 & .023 & .022 & .014 & .096 & .095 & .099 & .071 & .097 & .112 & .112 & \underline{.137} & .163 & .179 & .183 & .153 \\
 & AutoInterp & .010 & .012 & .008 & .007 & .058 & .055 & .054 & .054 & .118 & .124 & .109 & .109 & .085 & .103 & .101 & .099 \\
 & Delphi & .008 & .010 & .013 & .011 & .060 & .055 & .060 & .063 & .108 & .113 & .121 & .110 & .141 & .122 & .133 & .113 \\
 & SASC-style & \underline{.079} & \underline{.062} & \underline{.035} & \underline{.021} & \underline{.211} & \underline{.175} & \underline{.128} & \underline{.087} & \underline{.169} & \underline{.164} & \underline{.164} & .136 & \underline{.356} & \underline{.342} & \textbf{.303} & \textbf{.219} \\
 & DeViSE-style & .033 & .023 & .014 & .008 & .131 & .110 & .087 & .064 & .154 & .162 & .155 & .133 & .238 & .227 & .193 & .164 \\
\bottomrule}
\end{table}
\begin{table}[!htb]\centering\scriptsize\setlength{\tabcolsep}{2.4pt}
\caption{Latent variables, same setting as Table~\ref{tab:st1-items}, with naming quality. The baselines name the factors that a shared factor analysis extracts. \ours{} names the latent variables its own discovery step finds.}
\label{tab:st1-latents}
\fittable{llcccc:cccc:cccc:cccc}{%
\toprule
 &  & \multicolumn{4}{c:}{NRR@1} & \multicolumn{4}{c:}{MRR} & \multicolumn{4}{c:}{identification} & \multicolumn{4}{c}{naming quality} \\
variables & method & 20 & 50 & 75 & 90 & 20 & 50 & 75 & 90 & 20 & 50 & 75 & 90 & 20 & 50 & 75 & 90 \\
\midrule
latent variables & \ours{} & \textbf{.540} & \textbf{.486} & \textbf{.386} & \textbf{.240} & \textbf{.660} & \textbf{.623} & \textbf{.560} & \textbf{.424} & \textbf{.531} & \textbf{.511} & \textbf{.441} & \textbf{.254} & \underline{.446} & \textbf{.426} & \textbf{.390} & \textbf{.217} \\
 & CLIP-Dissect-e5 & .073 & .119 & .127 & \underline{.142} & .257 & .288 & .298 & \underline{.292} & .128 & .120 & .151 & .121 & .116 & .112 & .135 & .105 \\
 & AutoInterp & .096 & .100 & .108 & .088 & .258 & .259 & .270 & .252 & .160 & .140 & .170 & .139 & .160 & .140 & .170 & .139 \\
 & Delphi & .123 & .108 & .085 & .069 & .266 & .269 & .251 & .235 & .138 & .108 & .100 & .115 & .138 & .108 & .100 & .115 \\
 & SASC-style & \underline{.369} & \underline{.315} & \underline{.204} & .115 & \underline{.523} & \underline{.479} & \underline{.393} & .291 & \underline{.502} & \underline{.424} & \underline{.319} & \underline{.204} & \textbf{.502} & \underline{.424} & \underline{.319} & \underline{.204} \\
 & DeViSE-style & .138 & .138 & .108 & .112 & .291 & .281 & .260 & .266 & .151 & .154 & .123 & .112 & .062 & .065 & .069 & .054 \\
\bottomrule}
\end{table}

\textit{A channel label gives a language prior almost no semantic information.} On the three robot columns the baselines remain low at every masking level. The neighbors of a channel are labels that describe no signal, and a sampling method receives only numbers. On the vehicle CAN they even rise a little as names are masked, Delphi from .000 to .118. A sampling method writes the same few names for every channel, and such a name earns credit only once that channel is unnamed. At 90\% masking SASC-style
writes rear right wheel speed for 55 of the 110 targets and front right
wheel speed for 19. At that level one to three names remain per
real log and one per drive, so the anchors are nearly gone and the lead closes. On the real logs \ours{} and Delphi both retrieve .120.

\textit{A channel name has two parts, and \ours{} leads on both.} The type of a channel says what is measured, such as the angle of a joint. Its coordinate says which one, by an index, an axis or a side. Type
accuracy scores the first part and coordinate accuracy the second
\mbox{(Appendix~\ref{apx:protocol} and Figure~\ref{fig:name-parts})}. The type is the part the structure determines most. Our method leads type accuracy on all three robot columns at every masking level. On the simulated arms its type accuracy is 1.3 to 2.4 times that of the strongest baseline, .537 against .281 at 50\% masking and .577 against .275 at 75\%. On the real logs it is 1.8 to 2.7 times that of the strongest baseline.
On the simulated arms its type accuracy rises from .396 to .577 between 20 and 75\% masking, so the type stays recoverable as names are masked.
The coordinate is the harder part, and on the simulated arms it is where the gap to the baselines opens as names are masked. At 90\% masking the coordinate accuracy of our method on the simulated arms is .145 against .014 for the strongest baseline. On the real logs it leads at every level, .579 against .395 at 90\%. On the vehicle CAN it leads at every
level, .367 against .200 at 50\%. Once the names of the siblings are masked, the position in the graph is what still distinguishes joint 3 from joint 5.

\begin{figure}[!htb]\centering
\includegraphics[width=\textwidth]{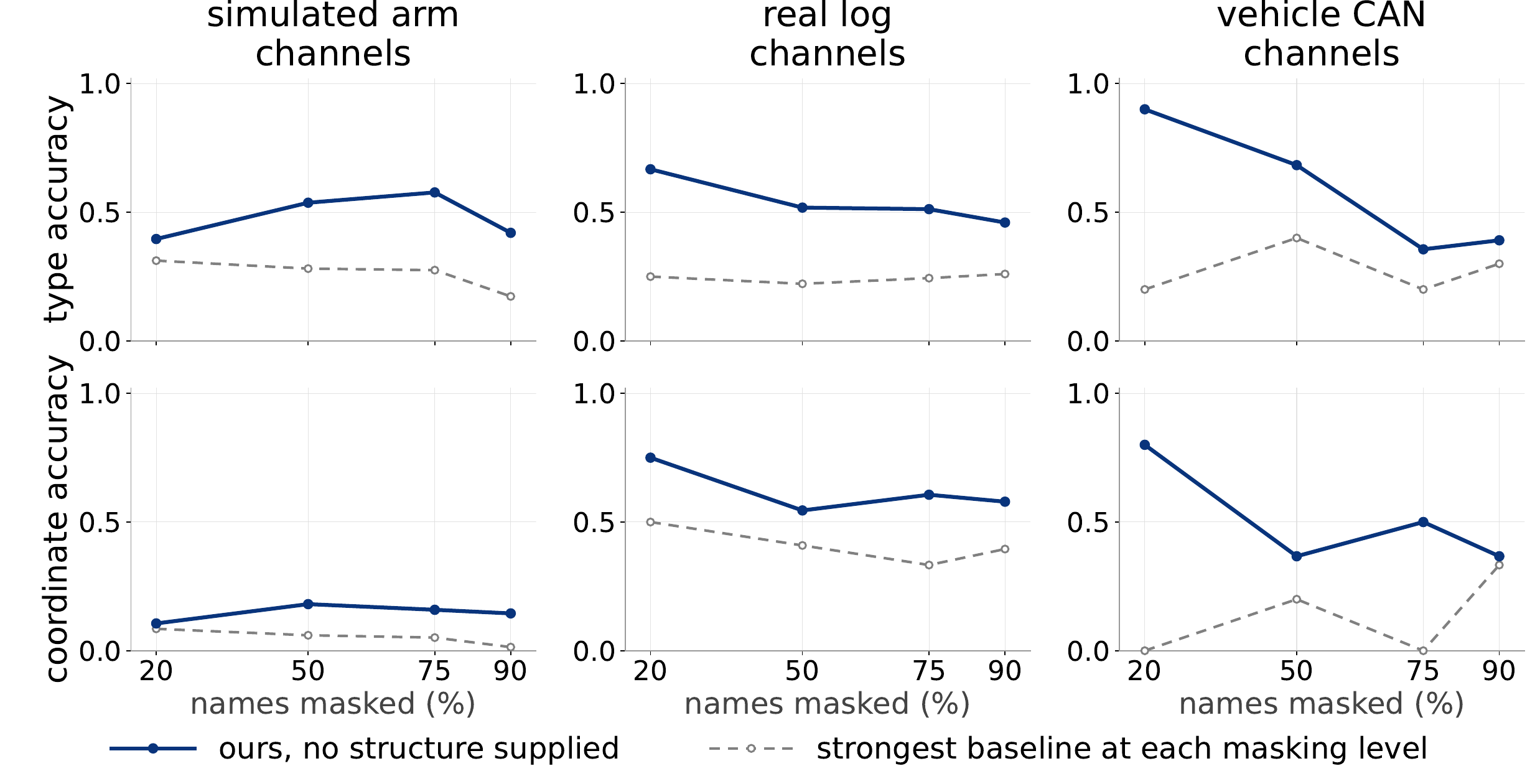}
\caption{The two parts of a channel name on the three robot columns, with
no structure supplied. Top, type accuracy scores what is measured. Bottom,
coordinate accuracy scores which one of that type. The solid line is
\ours{} and the dashed line the strongest baseline at each masking level.}
\label{fig:name-parts}
\end{figure}

\textit{Through the structure, \ours{} names the latent variables that no activation profile can name.} A latent variable has no data column and no text of its own. The four methods that read an activation profile therefore stay between .07 and .14 in NRR@1 at every masking level. SASC-style reads the item texts and falls from .369 to .115 as they are masked. Our method uses the relations of the latent variable to its items and falls from .540 to .240. The ratio between the two therefore grows from 1.5 to 1.9 through 75\% masking. Naming quality rejects a quoted statement and a copied item text. Under naming quality our method ranks first from 50\% masking on, .390 against .319 for SASC-style at 75\%. On questionnaire items, a named sibling indicates the topic of the unnamed item. There the ratio of \ours{} to SASC-style in NRR@1 is 1.5, 1.6, 1.7 and 1.3 at the four levels.

\textit{Two panels depart from this pattern, and the named siblings are the reason in both.} On the simulated arms NRR@1 peaks at 50\% masking, since at light masking \ours{} copies the name of a named sibling that differs from the target by one index. A copied sibling name lies nearer to that sibling than to the target. Sibling discrimination on the simulated arms is therefore the one auxiliary metric on which some baselines lie above our method. On questionnaire items the identification curve of SASC-style overlaps that of our method. The judge picks one of five sentences from the same scale, so it credits any name that reaches the right latent variable. Any method that reads the neighboring sentences reaches that latent variable (Table~\ref{tab:copy-refs}). NRR@1 requires the exact item, and there our method leads at every level.

\subsection{RQ1: the discovered graph and the supplied graph}
\label{apx:j-graphs}

The second finding of RQ1 is that \ours{} discovers the structure it needs and that its lead is wider when none is supplied.
Tables~\ref{tab:st2-robot} to~\ref{tab:st2-latents} give every method under the discovered graph and under the supplied graph. They carry the numbers behind the lower row of Figure~\ref{fig:main}.
\begin{table}[!htb]\centering\scriptsize\setlength{\tabcolsep}{2.4pt}
\caption{Robot channels under the supplied graph and under the graph \ours{} discovers, every method given the same graph. A dagger marks a method whose output does not change with the structure it is supplied, and a dash a method that returned no name for any channel of that set at that level.}
\label{tab:st2-robot}
\fittable{lllcccc:cccc:cccc}{%
\toprule
 &  &  & \multicolumn{4}{c:}{NRR@1} & \multicolumn{4}{c:}{MRR} & \multicolumn{4}{c}{SDA} \\
structure & variables & method & 20 & 50 & 75 & 90 & 20 & 50 & 75 & 90 & 20 & 50 & 75 & 90 \\
\midrule
supplied graph & simulated arms & \ours{} & \textbf{.146} & \textbf{.198} & \textbf{.181} & \textbf{.055} & \textbf{.317} & \textbf{.341} & \textbf{.354} & \textbf{.225} & .170 & \underline{.259} & \textbf{.222} & .117 \\
 &  & CLIP-Dissect-e5$^{\dagger}$ & \underline{.083} & .058 & .027 & .027 & \underline{.185} & .186 & .116 & .120 & \underline{.234} & .155 & .142 & .169 \\
 &  & AutoInterp$^{\dagger}$ & .000 & .008 & .011 & .005 & .078 & .095 & .095 & .086 & \textbf{.255} & .207 & .165 & .155 \\
 &  & Delphi$^{\dagger}$ & .000 & .025 & .016 & .005 & .079 & .120 & .119 & .083 & .043 & .241 & .193 & \underline{.197} \\
 &  & SASC-style$^{\dagger}$ & .021 & .058 & .027 & .018 & .112 & .162 & .130 & .118 & .170 & .147 & .165 & .183 \\
 &  & DeViSE-style$^{\dagger}$ & .021 & .033 & .022 & \underline{.032} & .143 & .136 & .131 & .116 & .170 & .198 & .170 & .178 \\
 &  & base model given graph & .021 & \underline{.083} & \underline{.033} & .023 & .164 & \underline{.193} & \underline{.135} & \underline{.148} & \textbf{.255} & \textbf{.293} & \underline{.205} & \textbf{.211} \\
 &  & GraphMAE & .000 & .000 & .000 & .000 & .162 & .134 & .086 & .082 & .128 & .147 & .097 & .174 \\
\midrule
supplied graph & real logs & \ours{} & \textbf{.500} & \textbf{.333} & \textbf{.244} & \textbf{.160} & \textbf{.566} & \textbf{.451} & \textbf{.369} & \textbf{.276} & \underline{.545} & \textbf{.440} & \textbf{.333} & \textbf{.298} \\
 &  & CLIP-Dissect-e5$^{\dagger}$ & \underline{.167} & .074 & .049 & .080 & .305 & .258 & .205 & .226 & .273 & .160 & .128 & .234 \\
 &  & AutoInterp$^{\dagger}$ & .083 & .037 & .049 & .020 & .265 & .200 & .178 & .152 & .364 & .240 & .205 & .170 \\
 &  & Delphi$^{\dagger}$ & .000 & .037 & .073 & \underline{.120} & .163 & .212 & .215 & \underline{.228} & .182 & .200 & .179 & \underline{.298} \\
 &  & SASC-style$^{\dagger}$ & .000 & .111 & .024 & .060 & .145 & .223 & .138 & .184 & .091 & .240 & .205 & .234 \\
 &  & DeViSE-style$^{\dagger}$ & .083 & .074 & .024 & .000 & .216 & .188 & .178 & .084 & .273 & .200 & .179 & .064 \\
 &  & base model given graph & \underline{.167} & \underline{.185} & \underline{.146} & .020 & \underline{.434} & \underline{.387} & \underline{.278} & .176 & \textbf{.727} & \underline{.440} & \underline{.333} & .213 \\
 &  & GraphMAE & .000 & .000 & .000 & .000 & .205 & .150 & .131 & .136 & .000 & .200 & .103 & .149 \\
\midrule
supplied graph & vehicle CAN & \ours{} & \textbf{.750} & \textbf{.300} & \textbf{.156} & \textbf{.145} & \textbf{.775} & \textbf{.553} & \textbf{.317} & \textbf{.378} & \textbf{1.000} & .300 & \underline{.400} & .133 \\
 &  & CLIP-Dissect-e5$^{\dagger}$ & .000 & .067 & \underline{.078} & .091 & .135 & .229 & \underline{.298} & .286 & \underline{.300} & .167 & .150 & \underline{.233} \\
 &  & AutoInterp$^{\dagger}$ & .000 & .000 & .000 & .000 & .184 & .245 & .171 & .191 & .100 & .100 & .100 & .000 \\
 &  & Delphi$^{\dagger}$ & .000 & .017 & .067 & \underline{.118} & .195 & .292 & .241 & .281 & .100 & .133 & .300 & .000 \\
 &  & SASC-style$^{\dagger}$ & .000 & \underline{.167} & .044 & .109 & \underline{.276} & \underline{.366} & .223 & .268 & .000 & \textbf{.433} & \textbf{.500} & \textbf{.333} \\
 &  & DeViSE-style$^{\dagger}$ & .000 & .000 & .000 & .000 & .180 & .173 & .141 & .083 & .000 & \underline{.333} & .000 & .000 \\
 &  & base model given graph & .000 & .083 & .067 & .091 & .269 & .322 & .269 & \underline{.297} & .000 & .100 & \textbf{.500} & .000 \\
 &  & GraphMAE & .000 & .000 & .000 & .000 & .250 & .199 & .156 & .191 & .000 & \underline{.333} & .000 & .000 \\
\midrule
discovered graph & simulated arms & \ours{} & \textbf{.083} & \textbf{.140} & \textbf{.137} & \textbf{.046} & \textbf{.224} & \textbf{.286} & \textbf{.278} & \textbf{.194} & .128 & .190 & .153 & .164 \\
 &  & CLIP-Dissect-e5$^{\dagger}$ & \underline{.083} & \underline{.058} & \underline{.027} & .027 & \underline{.185} & \underline{.186} & .116 & .120 & \underline{.234} & .155 & .142 & .169 \\
 &  & AutoInterp$^{\dagger}$ & .000 & .008 & .011 & .005 & .078 & .095 & .095 & .086 & \textbf{.255} & \underline{.207} & .165 & .155 \\
 &  & Delphi$^{\dagger}$ & .000 & .025 & .016 & .005 & .079 & .120 & .119 & .083 & .043 & \textbf{.241} & \textbf{.193} & \underline{.197} \\
 &  & SASC-style$^{\dagger}$ & .021 & \underline{.058} & \underline{.027} & .018 & .112 & .162 & .130 & .118 & .170 & .147 & .165 & .183 \\
 &  & DeViSE-style$^{\dagger}$ & .021 & .033 & .022 & \underline{.032} & .143 & .136 & \underline{.131} & .116 & .170 & .198 & .170 & .178 \\
 &  & base model given graph & .000 & .050 & \underline{.027} & .023 & .147 & .176 & .128 & \underline{.131} & .191 & .103 & \underline{.182} & \textbf{.202} \\
 &  & GraphMAE & .000 & .000 & .000 & .000 & .154 & .097 & .082 & .080 & .149 & .198 & .148 & .160 \\
\midrule
discovered graph & real logs & \ours{} & \textbf{.417} & \textbf{.296} & \textbf{.195} & \textbf{.120} & \textbf{.485} & \textbf{.385} & \textbf{.314} & \textbf{.248} & \textbf{.455} & \textbf{.360} & \textbf{.359} & \underline{.255} \\
 &  & CLIP-Dissect-e5$^{\dagger}$ & \underline{.167} & .074 & .049 & .080 & \underline{.305} & .258 & .205 & .226 & .273 & .160 & .128 & .234 \\
 &  & AutoInterp$^{\dagger}$ & .083 & .037 & .049 & .020 & .265 & .200 & .178 & .152 & .364 & .240 & .205 & .170 \\
 &  & Delphi$^{\dagger}$ & .000 & .037 & .073 & \underline{.120} & .163 & .212 & .215 & \underline{.228} & .182 & .200 & .179 & \textbf{.298} \\
 &  & SASC-style$^{\dagger}$ & .000 & .111 & .024 & .060 & .145 & .223 & .138 & .184 & .091 & .240 & .205 & .234 \\
 &  & DeViSE-style$^{\dagger}$ & .083 & .074 & .024 & .000 & .216 & .188 & .178 & .084 & .273 & .200 & .179 & .064 \\
 &  & base model given graph & .000 & \underline{.148} & \underline{.122} & .020 & .297 & \underline{.334} & \underline{.284} & .181 & \underline{.455} & \underline{.320} & \underline{.282} & .128 \\
 &  & GraphMAE & .000 & .000 & .000 & .000 & .156 & .127 & .135 & .134 & .000 & .120 & .103 & .128 \\
\midrule
discovered graph & vehicle CAN & \ours{} & \textbf{.800} & \textbf{.367} & \textbf{.244} & \textbf{.218} & \textbf{.859} & \textbf{.570} & \textbf{.412} & \textbf{.397} & \textbf{.800} & \underline{.367} & \textbf{.500} & \textbf{.367} \\
 &  & CLIP-Dissect-e5$^{\dagger}$ & .000 & .067 & \underline{.078} & .091 & .135 & .229 & \underline{.298} & \underline{.286} & \underline{.300} & .167 & .150 & .233 \\
 &  & AutoInterp$^{\dagger}$ & .000 & .000 & .000 & .000 & .184 & .245 & .171 & .191 & .100 & .100 & .100 & .000 \\
 &  & Delphi$^{\dagger}$ & .000 & .017 & .067 & \underline{.118} & .195 & .292 & .241 & .281 & .100 & .133 & .300 & .000 \\
 &  & SASC-style$^{\dagger}$ & .000 & \underline{.167} & .044 & .109 & .276 & \underline{.366} & .223 & .268 & .000 & \textbf{.433} & \underline{.500} & \underline{.333} \\
 &  & DeViSE-style$^{\dagger}$ & .000 & .000 & .000 & .000 & .180 & .173 & .141 & .083 & .000 & .333 & .000 & .000 \\
 &  & base model given graph & \underline{.100} & .050 & .022 & .073 & \underline{.343} & .279 & .236 & .265 & .000 & .100 & \underline{.500} & .000 \\
 &  & GraphMAE & .000 & .000 & .000 & .000 & .147 & .179 & .169 & .191 & .000 & .267 & .200 & .000 \\
\bottomrule}
\end{table}

\begin{table}[!htb]\centering\scriptsize\setlength{\tabcolsep}{2.4pt}
\addtocounter{table}{-1}
\caption{Continued with the same rows, for identification, type accuracy and coordinate accuracy.}
\fittable{lllcccc:cccc:cccc}{%
\toprule
 &  &  & \multicolumn{4}{c:}{identification} & \multicolumn{4}{c:}{type acc} & \multicolumn{4}{c}{coord acc} \\
structure & variables & method & 20 & 50 & 75 & 90 & 20 & 50 & 75 & 90 & 20 & 50 & 75 & 90 \\
\midrule
supplied graph & simulated arms & \ours{} & \textbf{.167} & \textbf{.252} & \textbf{.242} & \textbf{.106} & \textbf{.542} & \textbf{.612} & \textbf{.703} & \textbf{.525} & \textbf{.149} & \textbf{.233} & \textbf{.216} & \textbf{.132} \\
 &  & CLIP-Dissect-e5$^{\dagger}$ & .062 & .050 & .033 & .060 & .271 & .215 & .137 & .105 & \underline{.085} & .026 & .023 & .009 \\
 &  & AutoInterp$^{\dagger}$ & .000 & .033 & .039 & .033 & .000 & .107 & .126 & .114 & .000 & .009 & .034 & .000 \\
 &  & Delphi$^{\dagger}$ & .062 & .034 & .072 & .028 & .104 & .141 & .170 & .132 & .021 & .052 & .051 & .014 \\
 &  & SASC-style$^{\dagger}$ & .042 & .074 & .060 & .028 & .208 & \underline{.281} & .203 & .173 & .043 & .060 & .040 & .014 \\
 &  & DeViSE-style$^{\dagger}$ & .000 & .025 & .016 & .005 & .312 & .273 & \underline{.275} & .173 & .000 & .000 & .000 & .000 \\
 &  & base model given graph & \underline{.109} & \underline{.190} & \underline{.078} & \underline{.084} & \underline{.375} & \underline{.281} & .214 & \underline{.288} & .043 & \underline{.172} & \underline{.074} & .061 \\
 &  & GraphMAE & .064 & .076 & .041 & .048 & \underline{.375} & .273 & .115 & .101 & .021 & .060 & .040 & \underline{.085} \\
\midrule
supplied graph & real logs & \ours{} & \textbf{.500} & \textbf{.407} & \textbf{.268} & \textbf{.208} & \textbf{.750} & \textbf{.593} & \textbf{.610} & \textbf{.500} & \textbf{1.000} & \textbf{.682} & \underline{.455} & \textbf{.553} \\
 &  & CLIP-Dissect-e5$^{\dagger}$ & .000 & .037 & .100 & .080 & .083 & .000 & .073 & .060 & .000 & .045 & .061 & .026 \\
 &  & AutoInterp$^{\dagger}$ & .250 & .111 & .098 & .041 & .250 & .185 & \underline{.244} & .260 & .250 & .182 & .212 & .368 \\
 &  & Delphi$^{\dagger}$ & .250 & .120 & .024 & .100 & .167 & .074 & .098 & .240 & .000 & .136 & .151 & .026 \\
 &  & SASC-style$^{\dagger}$ & .083 & .115 & .073 & .080 & .083 & \underline{.222} & .146 & .180 & .250 & .409 & .333 & .395 \\
 &  & DeViSE-style$^{\dagger}$ & .083 & .074 & .098 & \underline{.120} & .083 & .111 & .195 & .000 & .500 & .227 & .303 & .316 \\
 &  & base model given graph & \underline{.333} & \underline{.333} & \underline{.128} & \underline{.120} & .167 & \underline{.222} & .195 & .260 & \underline{.875} & \underline{.591} & \textbf{.485} & \underline{.500} \\
 &  & GraphMAE & .083 & .000 & .000 & .061 & \underline{.583} & \underline{.222} & .220 & \underline{.380} & .250 & .500 & .242 & .447 \\
\midrule
supplied graph & vehicle CAN & \ours{} & \textbf{.750} & \textbf{.317} & \textbf{.216} & \textbf{.200} & \textbf{.750} & \textbf{.650} & \textbf{.289} & \textbf{.382} & \textbf{1.000} & \textbf{.300} & \textbf{.400} & \underline{.133} \\
 &  & CLIP-Dissect-e5$^{\dagger}$ & .000 & .050 & .089 & .118 & .000 & .050 & .033 & .027 & .000 & .000 & .000 & .000 \\
 &  & AutoInterp$^{\dagger}$ & .050 & .017 & .000 & .009 & .000 & .400 & .200 & .273 & .000 & .000 & .000 & .000 \\
 &  & Delphi$^{\dagger}$ & .000 & .067 & .033 & .055 & .000 & .183 & .156 & .273 & .000 & .000 & .000 & .000 \\
 &  & SASC-style$^{\dagger}$ & .000 & \underline{.250} & \underline{.146} & \underline{.191} & .200 & .383 & .067 & \underline{.300} & .000 & \underline{.200} & .000 & \textbf{.333} \\
 &  & DeViSE-style$^{\dagger}$ & .050 & .067 & .022 & -- & .000 & .000 & .000 & .000 & .000 & .000 & .000 & .000 \\
 &  & base model given graph & \underline{.105} & .133 & .091 & .074 & .250 & \underline{.483} & .000 & .264 & .000 & .000 & .000 & .000 \\
 &  & GraphMAE & .000 & .017 & .000 & .009 & \underline{.500} & .000 & \underline{.222} & .273 & .000 & .000 & .000 & .000 \\
\midrule
discovered graph & simulated arms & \ours{} & \textbf{.106} & \textbf{.203} & \textbf{.208} & \textbf{.116} & \underline{.396} & \textbf{.537} & \textbf{.577} & \textbf{.420} & \textbf{.106} & \textbf{.181} & \textbf{.159} & \textbf{.145} \\
 &  & CLIP-Dissect-e5$^{\dagger}$ & \underline{.062} & .050 & .033 & .060 & .271 & .215 & .137 & .105 & \underline{.085} & .026 & .023 & .009 \\
 &  & AutoInterp$^{\dagger}$ & .000 & .033 & .039 & .033 & .000 & .107 & .126 & .114 & .000 & .009 & .034 & .000 \\
 &  & Delphi$^{\dagger}$ & \underline{.062} & .034 & .072 & .028 & .104 & .141 & .170 & .132 & .021 & .052 & .051 & .014 \\
 &  & SASC-style$^{\dagger}$ & .042 & .074 & .060 & .028 & .208 & .281 & .203 & .173 & .043 & .060 & .040 & .014 \\
 &  & DeViSE-style$^{\dagger}$ & .000 & .025 & .016 & .005 & .312 & .273 & \underline{.275} & .173 & .000 & .000 & .000 & .000 \\
 &  & base model given graph & .021 & .076 & \underline{.096} & \underline{.089} & \textbf{.417} & \underline{.364} & .225 & \underline{.265} & .021 & .035 & .040 & .052 \\
 &  & GraphMAE & .045 & \underline{.078} & .045 & .029 & .312 & .132 & .099 & .101 & .064 & \underline{.095} & \underline{.085} & \underline{.066} \\
\midrule
discovered graph & real logs & \ours{} & \textbf{.417} & \underline{.296} & \textbf{.244} & \textbf{.180} & \textbf{.667} & \textbf{.518} & \textbf{.512} & \textbf{.460} & \textbf{.750} & \textbf{.545} & \textbf{.606} & \textbf{.579} \\
 &  & CLIP-Dissect-e5$^{\dagger}$ & .000 & .037 & .100 & .080 & .083 & .000 & .073 & .060 & .000 & .045 & .061 & .026 \\
 &  & AutoInterp$^{\dagger}$ & \underline{.250} & .111 & .098 & .041 & \underline{.250} & .185 & .244 & .260 & .250 & .182 & .212 & .368 \\
 &  & Delphi$^{\dagger}$ & \underline{.250} & .120 & .024 & .100 & .167 & .074 & .098 & .240 & .000 & .136 & .151 & .026 \\
 &  & SASC-style$^{\dagger}$ & .083 & .115 & .073 & .080 & .083 & .222 & .146 & .180 & .250 & .409 & .333 & .395 \\
 &  & DeViSE-style$^{\dagger}$ & .083 & .074 & .098 & \underline{.120} & .083 & .111 & .195 & .000 & .500 & .227 & .303 & .316 \\
 &  & base model given graph & .083 & \textbf{.370} & \underline{.171} & .020 & .083 & \underline{.259} & \underline{.293} & .280 & \underline{.625} & \underline{.500} & \underline{.485} & \underline{.474} \\
 &  & GraphMAE & .000 & .000 & .122 & .040 & \underline{.250} & \underline{.259} & .220 & \underline{.380} & .250 & .227 & .303 & .421 \\
\midrule
discovered graph & vehicle CAN & \ours{} & \textbf{.800} & \textbf{.400} & \textbf{.300} & \textbf{.282} & \textbf{.900} & \textbf{.683} & \textbf{.356} & \textbf{.391} & \textbf{.800} & \textbf{.367} & \textbf{.500} & \textbf{.367} \\
 &  & CLIP-Dissect-e5$^{\dagger}$ & .000 & .050 & .089 & .118 & .000 & .050 & .033 & .027 & .000 & .000 & .000 & .000 \\
 &  & AutoInterp$^{\dagger}$ & \underline{.050} & .017 & .000 & .009 & .000 & \underline{.400} & \underline{.200} & .273 & .000 & .000 & .000 & .000 \\
 &  & Delphi$^{\dagger}$ & .000 & .067 & .033 & .055 & .000 & .183 & .156 & .273 & .000 & .000 & .000 & .000 \\
 &  & SASC-style$^{\dagger}$ & .000 & \underline{.250} & \underline{.146} & \underline{.191} & \underline{.200} & .383 & .067 & \underline{.300} & .000 & \underline{.200} & .000 & \underline{.333} \\
 &  & DeViSE-style$^{\dagger}$ & \underline{.050} & .067 & .022 & -- & .000 & .000 & .000 & .000 & .000 & .000 & .000 & .000 \\
 &  & base model given graph & \underline{.050} & .133 & .023 & .135 & .050 & .367 & .000 & .009 & .000 & .000 & .000 & .000 \\
 &  & GraphMAE & .000 & .067 & .000 & .009 & .100 & .150 & .122 & .273 & .000 & .000 & .000 & .000 \\
\bottomrule}
\end{table}

\begin{table}[!htb]\centering\scriptsize\setlength{\tabcolsep}{2.4pt}
\caption{Questionnaire items under both structure sources, every method given the same graph. A dagger marks a method whose output does not change with the structure it is supplied.}
\label{tab:st2-items}
\fittable{llcccc:cccc:cccc:cccc}{%
\toprule
 &  & \multicolumn{4}{c:}{NRR@1} & \multicolumn{4}{c:}{MRR} & \multicolumn{4}{c:}{SDA} & \multicolumn{4}{c}{identification} \\
structure & method & 20 & 50 & 75 & 90 & 20 & 50 & 75 & 90 & 20 & 50 & 75 & 90 & 20 & 50 & 75 & 90 \\
\midrule
supplied graph & \ours{} & \textbf{.126} & \textbf{.097} & \textbf{.064} & \textbf{.038} & \textbf{.254} & \textbf{.220} & \textbf{.179} & \textbf{.121} & \textbf{.200} & \textbf{.176} & .147 & .129 & \underline{.351} & \underline{.314} & \textbf{.304} & \underline{.227} \\
 & CLIP-Dissect-e5$^{\dagger}$ & .028 & .023 & .022 & .014 & .096 & .095 & .099 & .071 & .097 & .112 & .112 & \textbf{.137} & .163 & .179 & .183 & .153 \\
 & AutoInterp$^{\dagger}$ & .010 & .012 & .008 & .007 & .058 & .055 & .054 & .054 & .118 & .124 & .109 & .109 & .085 & .103 & .101 & .099 \\
 & Delphi$^{\dagger}$ & .008 & .010 & .013 & .011 & .060 & .055 & .060 & .063 & .108 & .113 & .121 & .110 & .141 & .122 & .133 & .113 \\
 & SASC-style$^{\dagger}$ & \underline{.079} & \underline{.062} & .035 & \underline{.021} & .211 & .175 & .128 & .087 & \underline{.169} & \underline{.164} & \textbf{.164} & \underline{.136} & \textbf{.356} & \textbf{.342} & \underline{.303} & .219 \\
 & DeViSE-style$^{\dagger}$ & .033 & .023 & .014 & .008 & .131 & .110 & .087 & .064 & .154 & .162 & \underline{.155} & .133 & .238 & .227 & .193 & .164 \\
 & base model given graph & .074 & .051 & \underline{.051} & .017 & \underline{.219} & \underline{.182} & \underline{.166} & \underline{.100} & .126 & .102 & .119 & .109 & .338 & .309 & .280 & \textbf{.228} \\
 & GraphMAE & .000 & .000 & .000 & .000 & .086 & .080 & .075 & .059 & .077 & .080 & .075 & .096 & .144 & .139 & .118 & .108 \\
\midrule
discovered graph & \ours{} & \textbf{.121} & \textbf{.098} & \textbf{.059} & \textbf{.028} & \textbf{.253} & \textbf{.216} & \textbf{.164} & \textbf{.105} & \textbf{.226} & \textbf{.189} & \textbf{.164} & \textbf{.147} & \underline{.356} & \underline{.364} & .283 & .217 \\
 & CLIP-Dissect-e5$^{\dagger}$ & .028 & .023 & .022 & .014 & .096 & .095 & .099 & .071 & .097 & .112 & .112 & \underline{.137} & .163 & .179 & .183 & .153 \\
 & AutoInterp$^{\dagger}$ & .010 & .012 & .008 & .007 & .058 & .055 & .054 & .054 & .118 & .124 & .109 & .109 & .085 & .103 & .101 & .099 \\
 & Delphi$^{\dagger}$ & .008 & .010 & .013 & .011 & .060 & .055 & .060 & .063 & .108 & .113 & .121 & .110 & .141 & .122 & .133 & .113 \\
 & SASC-style$^{\dagger}$ & .079 & \underline{.062} & .035 & \underline{.021} & .211 & .175 & .128 & .087 & .169 & \underline{.164} & \underline{.164} & .136 & .356 & .342 & \underline{.303} & \underline{.219} \\
 & DeViSE-style$^{\dagger}$ & .033 & .023 & .014 & .008 & .131 & .110 & .087 & .064 & .154 & .162 & .155 & .133 & .238 & .227 & .193 & .164 \\
 & base model given graph & \underline{.085} & .054 & \underline{.046} & .017 & \underline{.227} & \underline{.188} & \underline{.162} & \underline{.094} & \underline{.174} & .163 & .154 & .125 & \textbf{.364} & \textbf{.384} & \textbf{.309} & \textbf{.229} \\
 & GraphMAE & .000 & .000 & .000 & .000 & .069 & .067 & .062 & .054 & .103 & .103 & .096 & .109 & .159 & .135 & .123 & .115 \\
\bottomrule}
\end{table}
\begin{table}[!htb]\centering\scriptsize\setlength{\tabcolsep}{2.4pt}
\caption{Latent variables under both structure sources, every method given the same graph.}
\label{tab:st2-latents}
\fittable{llcccc:cccc:cccc:cccc}{%
\toprule
 &  & \multicolumn{4}{c:}{NRR@1} & \multicolumn{4}{c:}{MRR} & \multicolumn{4}{c:}{identification} & \multicolumn{4}{c}{naming quality} \\
structure & method & 20 & 50 & 75 & 90 & 20 & 50 & 75 & 90 & 20 & 50 & 75 & 90 & 20 & 50 & 75 & 90 \\
\midrule
supplied graph & \ours{} & \textbf{.485} & \underline{.415} & \textbf{.415} & \textbf{.238} & \textbf{.624} & \textbf{.560} & \textbf{.551} & \textbf{.400} & \textbf{.627} & \textbf{.537} & \textbf{.535} & \textbf{.351} & \textbf{.596} & \textbf{.518} & \textbf{.500} & \textbf{.332} \\
 & CLIP-Dissect-e5 & .158 & .146 & .135 & .181 & .333 & .314 & .315 & .337 & .243 & .203 & .188 & .186 & .243 & .203 & .188 & .186 \\
 & AutoInterp & .108 & .108 & .112 & .096 & .262 & .267 & .275 & .260 & .139 & .202 & .135 & .142 & .139 & .202 & .135 & .142 \\
 & Delphi & .088 & .123 & .135 & .088 & .250 & .287 & .283 & .257 & .169 & .174 & .192 & .173 & .169 & .174 & .192 & .173 \\
 & SASC-style$^{\dagger}$ & .369 & .315 & .204 & .115 & .523 & .479 & .393 & .291 & .502 & .424 & .319 & .204 & .502 & .424 & .319 & .204 \\
 & DeViSE-style$^{\dagger}$ & .138 & .138 & .108 & .112 & .291 & .281 & .260 & .266 & .151 & .154 & .123 & .112 & .062 & .065 & .069 & .054 \\
 & base model given graph & \underline{.477} & \textbf{.419} & \underline{.331} & \underline{.212} & \underline{.614} & \underline{.560} & \underline{.480} & \underline{.376} & \underline{.512} & \underline{.492} & \underline{.454} & \underline{.285} & \underline{.512} & \underline{.492} & \underline{.450} & \underline{.277} \\
 & GraphMAE & .242 & .181 & .158 & .115 & .400 & .347 & .320 & .276 & .350 & .315 & .208 & .124 & .000 & .000 & .000 & .000 \\
\midrule
discovered graph & \ours{} & \textbf{.540} & \textbf{.486} & \underline{.386} & \textbf{.240} & \textbf{.660} & \textbf{.623} & \underline{.560} & \textbf{.424} & \textbf{.531} & \textbf{.511} & \underline{.441} & \textbf{.254} & .446 & \underline{.426} & \underline{.390} & \underline{.217} \\
 & CLIP-Dissect-e5 & .137 & .203 & .191 & .146 & .347 & .403 & .383 & .351 & .160 & .188 & .177 & .095 & .160 & .188 & .177 & .095 \\
 & AutoInterp & .131 & .126 & .163 & .126 & .334 & .335 & .367 & .330 & .154 & .151 & .157 & .143 & .154 & .151 & .157 & .143 \\
 & Delphi & .177 & .137 & .174 & .126 & .372 & .344 & .366 & .345 & .191 & .157 & .200 & .166 & .183 & .157 & .197 & .166 \\
 & SASC-style$^{\dagger}$ & .369 & .315 & .204 & .115 & .523 & .479 & .393 & .291 & .502 & .424 & .319 & .204 & \underline{.502} & .424 & .319 & .204 \\
 & DeViSE-style$^{\dagger}$ & .138 & .138 & .108 & .112 & .291 & .281 & .260 & .266 & .151 & .154 & .123 & .112 & .062 & .065 & .069 & .054 \\
 & base model given graph & \underline{.394} & \underline{.389} & \textbf{.403} & \underline{.186} & \underline{.576} & \underline{.568} & \textbf{.579} & \underline{.401} & \underline{.531} & \underline{.480} & \textbf{.446} & \underline{.226} & \textbf{.520} & \textbf{.480} & \textbf{.437} & \textbf{.223} \\
 & GraphMAE & .203 & .197 & .200 & .180 & .383 & .380 & .389 & .385 & .294 & .280 & .251 & .163 & .000 & .000 & .000 & .000 \\
\bottomrule}
\end{table}

\textit{\ours{} does as well on the graph it discovers as on the supplied graph.} The dashed line of the lower row is our method on its discovered graph, and it follows the solid line in every panel. On the vehicle CAN it is ahead at every
masking level. On the simulated arms and the real logs the supplied graph describes the mechanism that produced the measurements, and it is ahead there. On latent variables the two graphs define different targets, and the two lines sit at a similar level. Type accuracy agrees
with NRR@1 on all three columns. On the vehicle CAN the discovered graph is ahead at every masking level, .900 and .683 against .750 and .650 at 20 and 50\%. The supplied graph is ahead at every level on the simulated arms and on the real logs.

\textit{On latent variables the two lines of the lower row are two tasks with different targets.} Under the supplied graph the targets are the 52 latent variables of the five keys counted with their parent level. Under the discovered graph they are the 70 latent variables our method discovers, each referred to the key by the rule of Appendix~\ref{apx:latent-ref}. A discovered latent variable holds 4.8 items on average, counting the items beneath its latent children, against 12.5 for a latent variable of the key. At 75\% masking it therefore keeps 1.2 named items against 3.1.
Table~\ref{tab:perscale-latents} gives both sources scale by scale.

\textit{A supplied graph lifts the baselines that can use it.} The five methods that name a variable from its data alone do not change with the structure they are given, which the dagger in Tables~\ref{tab:st2-robot} and~\ref{tab:st2-items} marks. On latent variables the same holds for the two that keep the factor-analysis grouping. The baselines that use the graph gain from it, and the gain depends on what the graph adds to their input. On the real logs, the base model that reads the graph as text reaches .185 and .146 in NRR@1 at 50 and 75\% masking. The strongest baseline with no graph reaches .111 and .073. On the
simulated arms its identification rate is above every such baseline at
all four levels, .190 against .074 at 50\%. On both the gain in NRR@1
is gone at 90\% masking, where few neighbors keep a name. On questionnaire items its MRR is above every baseline with no graph at all four levels, .219 against .211 at 20\%. The gain is small, since an item's named neighbors are sentences that those baselines read too. On latent variables five baselines use the graph.
CLIP-Dissect-e5, AutoInterp and Delphi take their grouping from it. CLIP-Dissect-e5 climbs from .073 to .158 in NRR@1 at 20\% masking and from .128 to .243 in identification. GraphMAE reaches .242 and the base model .477, against .369 for the strongest baseline with no structure supplied. The base model gains since a latent variable of the key enters the text as its list of named items. GraphMAE never places a true name first on a
channel or an item, since it returns a known name there. On the vehicle CAN the base model stays below the strongest baseline with no graph from 50\% masking on, .083 against .167 at 50\%. The documented edges follow the kinematics, while the recorded speed is computed from the wheel speeds.

\textit{Only \ours{} can discover a graph, so its lead over the strongest baseline is wider with no structure supplied.} With no structure supplied, none of the baselines has a graph to use. A baseline takes from a graph what text can carry. The optimization and the prefix turn the structure into a vector, so our method keeps the lead under both sources. On the vehicle CAN, the lead in NRR@1 at the four masking levels is .800, .200, .167 and .100 with no structure supplied. Under the supplied graph it is .750, .133, .078 and .027.
From 50\% masking on the ratio is 1.8 to 3.1 against 1.2 to 2.0.
On the real logs the lead is 2.7 times at 50 and 75\% masking against
1.8 and 1.7 times. The same holds on latent variables at every masking level, where the targets of our method are the two sets described above. There the difference between the two rows is largest through 50\% masking, .163 and .175 against .098 and .072 at 75 and 90\%. Many anchors let the graph constrain an unnamed variable tightly, and the difference closes where few names are left in either row. On the
vehicle CAN the difference stays between .050 and .088 at every level.

\textit{A physical law supplied as a prior raises NRR@1 at 20 and 50\% masking on the simulated arms and on the vehicle CAN.} Table~\ref{tab:physics-prior} compares the full method on the graphs discovered with the prior of Appendix~\ref{apx:discovery} and on the graphs discovered from the time series alone. On the simulated arms the prior raises NRR@1 from .083 to .104 at 20\% masking and from .140 to .157 at 50\%.  On the vehicle CAN it raises NRR@1 from .800 to .950 at 20\% and from .367 to .483 at 50\%.
\begin{table}[!htb]\centering\scriptsize\setlength{\tabcolsep}{2.4pt}
\caption{The main method on the graph discovered from the time series alone, the setting of every other table, against the same method when the discovery step also admits the instantaneous edges that the physics of the platform orients, on an arm from the joints to the end-effector pose and the gripper, and on a drive from the steering wheel angle and the wheel speeds to the vehicle speed and from those to the accelerations and the rates. No structure is supplied in either row. Columns are the four masking levels.}
\label{tab:physics-prior}
\fittable{llcccc:cccc:cccc:cccc:cccc}{%
\toprule
 &  & \multicolumn{4}{c:}{NRR@1} & \multicolumn{4}{c:}{MRR} & \multicolumn{4}{c:}{identification} & \multicolumn{4}{c:}{type acc} & \multicolumn{4}{c}{coord acc} \\
variables & discovery & 20 & 50 & 75 & 90 & 20 & 50 & 75 & 90 & 20 & 50 & 75 & 90 & 20 & 50 & 75 & 90 & 20 & 50 & 75 & 90 \\
\midrule
simulated arms & time only & .083 & .140 & .137 & .046 & .224 & .286 & .278 & .194 & .106 & .203 & .208 & .116 & .396 & .537 & .577 & .420 & .106 & .181 & .159 & .145 \\
 & with the physical law & .104 & .157 & .132 & .068 & .286 & .316 & .299 & .213 & .167 & .229 & .187 & .144 & .583 & .587 & .654 & .466 & .106 & .198 & .153 & .174 \\
\midrule
vehicle CAN & time only & .800 & .367 & .244 & .218 & .859 & .570 & .412 & .397 & .800 & .400 & .300 & .282 & .900 & .683 & .356 & .391 & .800 & .367 & .500 & .367 \\
 & with the physical law & .950 & .483 & .222 & .209 & .955 & .632 & .379 & .394 & .950 & .500 & .278 & .282 & .950 & .800 & .333 & .427 & 1.000 & .333 & .500 & .200 \\
\bottomrule}
\end{table}

\subsection{RQ1: the lead replicate by replicate, against copied names and scale by scale}
\label{apx:j-checks}

Three checks test whether the lead of RQ1 depends on a few replicates, on the way a name is scored or on the pooling.

\paragraph{Paired differences to the strongest baseline.} \textit{The lead holds on the same replicates.} Tables~\ref{tab:paired-st1} and~\ref{tab:paired-st2} put every cell of the main comparison on the same replicates as its strongest baseline and report the mean paired difference with its standard error over replicates. The replicates are folds on the questionnaire columns and arms, logs or drives on the robot columns. The
strongest baseline of a cell is the baseline with the highest pooled
value in that cell of Tables~\ref{tab:st1-robot}
to~\ref{tab:st1-latents}, or of Tables~\ref{tab:st2-robot}
to~\ref{tab:st2-latents} under the supplied graph. Without a supplied structure, \ours{} is ahead in 58 of the 64 cells and ahead by more than two standard errors in 38. It is behind by more than two standard errors in no cell. The lead lies in the retrieval metrics NRR@1 and MRR, where a name has to reach the right variable. There the difference is positive in 38 of the 40
cells and exceeds two standard errors in 29. On the judged metrics, our method is ahead by more than two standard errors in 9 of the 24 cells and level with the strongest baseline within two standard errors in 14. The judge credits any name that reaches the right latent variable or channel type. Under the supplied graph the three counts are 60, 33 and 1 of 64. The one cell behind by more than two standard errors is item identification at 50\% masking. There SASC-style reads the neighboring sentences, and the judge credits the latent variable they name. The cell for real-log identification at 90\% rests on a single replicate and has no standard error. A count below the number of replicates means that the baseline returned no name on the other replicates.

\begin{table}[!htbp]\centering\scriptsize\setlength{\tabcolsep}{3pt}
\caption{Paired differences to the strongest baseline, no structure supplied. Each cell is our value minus the strongest baseline's, as the mean over replicates, its standard error, the number of replicates in which our method is ahead, and the baseline, C CLIP-Dissect-e5, A AutoInterp, D Delphi, S SASC-style, V DeViSE-style. Bold is a difference above two standard errors.}
\label{tab:paired-st1}
\fittable{llllll}{
\toprule
variables & metric & 20 & 50 & 75 & 90 \\
\midrule
simulated arms & NRR@1 & $-$.019 $\pm$ .067 (2/6) C & +.079 $\pm$ .040 (4/6) S & \textbf{+.109 $\pm$ .022} (6/6) S & +.014 $\pm$ .017 (3/6) V \\
 & MRR & +.026 $\pm$ .058 (5/6) C & \textbf{+.097 $\pm$ .030} (6/6) C & \textbf{+.145 $\pm$ .026} (6/6) V & \textbf{+.069 $\pm$ .017} (6/6) C \\
 & identification & +.042 $\pm$ .024 (3/6) D & \textbf{+.130 $\pm$ .022} (6/6) S & \textbf{+.135 $\pm$ .031} (6/6) D & \textbf{+.054 $\pm$ .023} (4/6) C \\
\midrule
real logs & NRR@1 & +.167 $\pm$ .167 (1/3) C & +.238 $\pm$ .133 (3/3) S & \textbf{+.132 $\pm$ .034} (3/3) D & $-$.002 $\pm$ .046 (1/3) D \\
 & MRR & +.090 $\pm$ .182 (1/3) C & +.177 $\pm$ .090 (3/3) C & \textbf{+.104 $\pm$ .008} (3/3) D & +.021 $\pm$ .062 (2/3) D \\
 & identification & +.111 $\pm$ .111 (1/3) D & +.234 $\pm$ .135 (3/3) D & +.163 $\pm$ .121 (2/3) C & +.000 $\pm$ -- (0/1) V \\
\midrule
vehicle CAN & NRR@1 & \textbf{+.800 $\pm$ .111} (9/10) S & \textbf{+.200 $\pm$ .033} (9/10) S & \textbf{+.167 $\pm$ .025} (10/10) C & \textbf{+.100 $\pm$ .042} (5/10) D \\
 & MRR & \textbf{+.583 $\pm$ .089} (10/10) S & \textbf{+.204 $\pm$ .036} (10/10) S & \textbf{+.114 $\pm$ .024} (10/10) C & \textbf{+.111 $\pm$ .041} (7/10) C \\
 & identification & \textbf{+.750 $\pm$ .112} (9/10) V & \textbf{+.150 $\pm$ .046} (6/10) S & \textbf{+.154 $\pm$ .067} (7/10) S & \textbf{+.091 $\pm$ .038} (6/10) S \\
\midrule
questionnaire items & NRR@1 & \textbf{+.041 $\pm$ .016} (4/5) S & \textbf{+.037 $\pm$ .009} (5/5) S & \textbf{+.025 $\pm$ .010} (5/5) S & +.008 $\pm$ .005 (3/5) S \\
 & MRR & \textbf{+.041 $\pm$ .019} (4/5) S & \textbf{+.041 $\pm$ .005} (5/5) S & \textbf{+.037 $\pm$ .009} (5/5) S & \textbf{+.018 $\pm$ .003} (5/5) S \\
 & identification & +.001 $\pm$ .034 (2/5) S & +.022 $\pm$ .011 (4/5) S & $-$.020 $\pm$ .017 (1/5) S & $-$.003 $\pm$ .007 (2/5) S \\
\midrule
latent variables & NRR@1 & \textbf{+.171 $\pm$ .025} (5/5) S & \textbf{+.170 $\pm$ .019} (5/5) S & \textbf{+.182 $\pm$ .028} (5/5) S & \textbf{+.098 $\pm$ .040} (4/5) C \\
 & MRR & \textbf{+.137 $\pm$ .018} (5/5) S & \textbf{+.144 $\pm$ .016} (5/5) S & \textbf{+.168 $\pm$ .024} (5/5) S & \textbf{+.132 $\pm$ .029} (5/5) C \\
 & identification & +.029 $\pm$ .020 (4/5) S & \textbf{+.087 $\pm$ .018} (5/5) S & \textbf{+.122 $\pm$ .040} (4/5) S & +.050 $\pm$ .039 (4/5) S \\
 & naming quality & $-$.056 $\pm$ .031 (1/5) S & +.002 $\pm$ .012 (2/5) S & +.070 $\pm$ .039 (4/5) S & +.013 $\pm$ .035 (3/5) S \\
\bottomrule}\end{table}

\begin{table}[!htbp]\centering\scriptsize\setlength{\tabcolsep}{3pt}
\caption{Paired differences to the strongest baseline under the supplied graph, which every method that accepts a graph as input receives, with the structure-free rows as reference. Each cell is our value minus the strongest baseline's, as the mean over replicates, its standard error, the number of replicates in which our method is ahead, and the baseline, C CLIP-Dissect-e5, A AutoInterp, D Delphi, S SASC-style, V DeViSE-style, B base model given graph, G GraphMAE. Bold is a difference above two standard errors.}
\label{tab:paired-st2}
\fittable{llllll}{
\toprule
variables & metric & 20 & 50 & 75 & 90 \\
\midrule
simulated arms & NRR@1 & +.037 $\pm$ .084 (3/6) C & \textbf{+.121 $\pm$ .033} (6/6) B & \textbf{+.149 $\pm$ .016} (6/6) B & +.027 $\pm$ .021 (3/6) V \\
 & MRR & +.107 $\pm$ .093 (5/6) C & \textbf{+.151 $\pm$ .029} (6/6) B & \textbf{+.220 $\pm$ .020} (6/6) B & \textbf{+.083 $\pm$ .031} (6/6) B \\
 & identification & +.041 $\pm$ .055 (3/6) B & +.058 $\pm$ .038 (4/6) B & \textbf{+.164 $\pm$ .024} (6/6) B & +.022 $\pm$ .019 (3/6) B \\
\midrule
real logs & NRR@1 & +.278 $\pm$ .147 (2/3) B & +.190 $\pm$ .156 (2/3) B & \textbf{+.098 $\pm$ .002} (3/3) B & +.041 $\pm$ .024 (2/3) D \\
 & MRR & +.115 $\pm$ .133 (2/3) B & +.098 $\pm$ .133 (2/3) B & \textbf{+.095 $\pm$ .007} (3/3) B & \textbf{+.057 $\pm$ .018} (3/3) D \\
 & identification & +.167 $\pm$ .096 (2/3) B & +.032 $\pm$ .159 (2/3) B & \textbf{+.143 $\pm$ .030} (3/3) B & +.115 $\pm$ .071 (3/3) B \\
\midrule
vehicle CAN & NRR@1 & \textbf{+.750 $\pm$ .083} (10/10) B & \textbf{+.133 $\pm$ .042} (8/10) S & +.078 $\pm$ .044 (6/10) C & +.027 $\pm$ .030 (4/10) D \\
 & MRR & \textbf{+.499 $\pm$ .078} (10/10) S & \textbf{+.188 $\pm$ .033} (9/10) S & +.019 $\pm$ .042 (5/10) C & \textbf{+.080 $\pm$ .018} (10/10) B \\
 & identification & \textbf{+.650 $\pm$ .130} (8/10) B & +.067 $\pm$ .057 (4/10) S & +.069 $\pm$ .065 (5/10) S & +.009 $\pm$ .025 (4/10) S \\
\midrule
questionnaire items & NRR@1 & \textbf{+.046 $\pm$ .014} (5/5) S & \textbf{+.036 $\pm$ .014} (4/5) S & \textbf{+.013 $\pm$ .006} (3/5) B & \textbf{+.017 $\pm$ .007} (5/5) S \\
 & MRR & \textbf{+.035 $\pm$ .013} (5/5) B & \textbf{+.038 $\pm$ .012} (4/5) B & \textbf{+.012 $\pm$ .005} (5/5) B & \textbf{+.021 $\pm$ .008} (5/5) B \\
 & identification & $-$.005 $\pm$ .019 (2/5) S & \textbf{$-$.028 $\pm$ .009} (0/5) S & +.001 $\pm$ .011 (2/5) S & $-$.001 $\pm$ .018 (2/5) B \\
\midrule
latent variables & NRR@1 & +.008 $\pm$ .019 (2/5) B & $-$.004 $\pm$ .034 (1/5) B & \textbf{+.085 $\pm$ .013} (5/5) B & \textbf{+.027 $\pm$ .005} (5/5) B \\
 & MRR & +.009 $\pm$ .016 (2/5) B & +.000 $\pm$ .028 (2/5) B & \textbf{+.071 $\pm$ .009} (5/5) B & \textbf{+.025 $\pm$ .003} (5/5) B \\
 & identification & \textbf{+.115 $\pm$ .029} (5/5) B & +.045 $\pm$ .026 (4/5) B & \textbf{+.081 $\pm$ .029} (4/5) B & \textbf{+.067 $\pm$ .018} (5/5) B \\
 & naming quality & \textbf{+.085 $\pm$ .023} (5/5) B & +.025 $\pm$ .022 (4/5) B & +.050 $\pm$ .028 (4/5) B & \textbf{+.055 $\pm$ .013} (5/5) B \\
\bottomrule}\end{table}

\paragraph{Name-copy references.}
\textit{A copied name reaches at most the right latent variable or channel type.} Tables~\ref{tab:copy-refs} and~\ref{tab:copy-refs-latents} score the three references of Appendix~\ref{apx:roster}. Where the output is a
known name, on every robot column and for Correlation kNN everywhere,
NRR@1 never places the true name first. Identification still credits such an output, since the judge often confuses a returned neighbor of the same type with the target among five candidates. This is why
the task is scored on NRR@1 and naming quality beside
identification. On questionnaire items, where Feature Propagation decodes to the concept bank, it reaches at most .067 on NRR@1. On latent variables, the two references that run behind the factor-analysis front end approach \ours{} in NRR@1. In the scoring space, the text of an item lies near the name of its own latent variable. EFA loading naming has .500 against .540 at 20\% masking and .254 against .240 at 90\%. Naming quality separates both from our method. EFA loading naming returns an item of the scale, and the verbatim check rejects a copied item text in every case. Its identification rate of .608 at 20\% masking therefore becomes a naming quality of .000. Feature Propagation decodes the mean of a factor's items to the concept bank. Its naming quality lies below that of our method at every masking level, .408 against .446 at 20\%.

\begin{table}[!htbp]\centering\scriptsize\setlength{\tabcolsep}{2.4pt}
\caption{Name-copy references with no structure supplied. Correlation kNN returns the name of the most correlated named variable, and Feature Propagation decodes the diffused known names to the nearest known name on robot channels and to the concept bank on questionnaires.}
\label{tab:copy-refs}
\fittable{llcccc:cccc:cccc}{%
\toprule
 &  & \multicolumn{4}{c:}{NRR@1} & \multicolumn{4}{c:}{MRR} & \multicolumn{4}{c}{identification} \\
variables & method & 20 & 50 & 75 & 90 & 20 & 50 & 75 & 90 & 20 & 50 & 75 & 90 \\
\midrule
simulated arms & Correlation kNN & .000 & .000 & .000 & .000 & {.181} & {.173} & {.159} & {.117} & {.106} & {.067} & {.068} & {.051} \\
 & Feature Propagation & .000 & .000 & .000 & .000 & {.189} & {.157} & {.168} & {.118} & {.128} & {.067} & {.074} & {.052} \\
\midrule
real logs & Correlation kNN & .000 & .000 & .000 & .000 & {.262} & {.218} & {.179} & {.142} & {.083} & {.222} & {.098} & {.040} \\
 & Feature Propagation & .000 & .000 & .000 & .000 & {.210} & {.209} & {.174} & {.136} & {.167} & {.259} & {.122} & {.020} \\
\midrule
vehicle CAN & Correlation kNN & .000 & .000 & .000 & .000 & {.228} & {.251} & {.181} & {.191} & .000 & {.083} & {.012} & .000 \\
 & Feature Propagation & .000 & .000 & .000 & .000 & {.174} & {.248} & {.159} & {.191} & .000 & {.067} & {.023} & .000 \\
\midrule
questionnaire items & Correlation kNN & .000 & .000 & .000 & .000 & {.261} & {.227} & {.169} & {.115} & {.433} & {.390} & {.325} & {.263} \\
 & Feature Propagation & {.067} & {.045} & {.031} & {.013} & {.222} & {.199} & {.159} & {.103} & {.272} & {.281} & {.244} & {.174} \\
\bottomrule}
\end{table}
\begin{table}[!htbp]\centering\scriptsize\setlength{\tabcolsep}{2.4pt}
\caption{The two references on latent variables. EFA loading naming returns the text of the item that loads highest on the factor, and Feature Propagation decodes the mean member embedding of the factor to the concept bank.}
\label{tab:copy-refs-latents}
\fittable{lcccc:cccc:cccc:cccc}{%
\toprule
 & \multicolumn{4}{c:}{NRR@1} & \multicolumn{4}{c:}{MRR} & \multicolumn{4}{c:}{identification} & \multicolumn{4}{c}{naming quality} \\
method & 20 & 50 & 75 & 90 & 20 & 50 & 75 & 90 & 20 & 50 & 75 & 90 & 20 & 50 & 75 & 90 \\
\midrule
EFA loading naming & {.500} & {.458} & {.381} & {.254} & {.643} & {.611} & {.543} & {.431} & {.608} & {.600} & {.550} & {.402} & .000 & .000 & .000 & .000 \\
Feature Propagation + EFA & {.454} & {.400} & {.358} & {.235} & {.616} & {.570} & {.519} & {.413} & {.631} & {.585} & {.519} & {.342} & {.408} & {.373} & {.308} & {.196} \\
\bottomrule}
\end{table}

\paragraph{Scale by scale, arm by arm, log by log.}
\textit{The lead holds scale by scale.} In NRR@1 on items \ours{} is first in 15 of the 20 combinations of scale and masking level. The main text pools the five
evaluation scales by sample and the arms, logs or drives of a robot
column by sample. Tables~\ref{tab:scale-dass} to~\ref{tab:scale-wvs} give every method on the items and the latent variables of each scale. They report the three metrics of Table~\ref{tab:st1-main} at every masking level and naming quality on latent variables. Table~\ref{tab:perbody} gives identification on each simulated arm and real log. Pooled by sample, the cells of the five scale tables
reproduce the main tables. The scales differ in how much information their items give a frozen model.
HEXACO has short everyday statements, so it is the easiest scale for
identification. Table~\ref{tab:perscale-latents} gives our method on each scale under both structure sources. The key groups the items by the theory of the scale, and the measurements group them by how respondents answer. The two sources therefore define different sets of latent variables on every scale. On the World Values Survey, the discovered latent variables differ most from the key (Appendix~\ref{apx:discovery}).

\begin{table}[!htbp]\centering\scriptsize\setlength{\tabcolsep}{2.6pt}
\caption{Latent variables scale by scale, our method under both structure sources.}
\label{tab:perscale-latents}
\fittable{llcccc:cccc:cccc:cccc}{
\toprule
 & & \multicolumn{4}{c}{NRR@1} & \multicolumn{4}{c}{MRR} & \multicolumn{4}{c}{identification} & \multicolumn{4}{c}{naming quality} \\
structure & scale & 20 & 50 & 75 & 90 & 20 & 50 & 75 & 90 & 20 & 50 & 75 & 90 & 20 & 50 & 75 & 90 \\
\midrule
discovered graph & DASS & .756 & .800 & .622 & .400 & .859 & .881 & .789 & .648 & .689 & .689 & .511 & .222 & .667 & .622 & .511 & .156 \\
 & HEXACO & .408 & .320 & .304 & .104 & .516 & .462 & .461 & .247 & .512 & .496 & .504 & .248 & .384 & .432 & .456 & .192 \\
 & KIMS & .691 & .655 & .418 & .400 & .814 & .803 & .633 & .605 & .382 & .382 & .273 & .164 & .382 & .345 & .255 & .164 \\
 & RIASEC & .692 & .585 & .508 & .262 & .793 & .726 & .655 & .473 & .800 & .662 & .562 & .369 & .615 & .446 & .406 & .323 \\
 & WVS & .350 & .333 & .217 & .233 & .524 & .488 & .427 & .409 & .300 & .367 & .283 & .250 & .283 & .317 & .267 & .250 \\
\midrule
supplied graph & DASS & .733 & .667 & .733 & .533 & .867 & .822 & .867 & .733 & .667 & .467 & .467 & .200 & .667 & .467 & .467 & .200 \\
 & HEXACO & .387 & .300 & .373 & .140 & .508 & .442 & .495 & .298 & .560 & .500 & .587 & .407 & .547 & .500 & .587 & .393 \\
 & KIMS & .600 & .750 & .450 & .550 & .792 & .858 & .671 & .721 & .500 & .500 & .450 & .300 & .500 & .500 & .450 & .300 \\
 & RIASEC & .667 & .633 & .600 & .400 & .822 & .786 & .750 & .573 & .933 & .828 & .633 & .433 & .867 & .690 & .367 & .367 \\
 & WVS & .556 & .422 & .311 & .222 & .720 & .584 & .450 & .374 & .689 & .512 & .356 & .182 & .600 & .488 & .333 & .159 \\
\bottomrule}\end{table}

\begin{table}[!htb]\centering\scriptsize\setlength{\tabcolsep}{2.4pt}
\caption{DASS, no structure supplied. Items above and latent variables below. Bold is the best value per column and underline the second best.}
\label{tab:scale-dass}
\fittable{lcccc:cccc:cccc}{
\toprule
 & \multicolumn{4}{c:}{NRR@1} & \multicolumn{4}{c:}{MRR} & \multicolumn{4}{c}{identification} \\
items & 20 & 50 & 75 & 90 & 20 & 50 & 75 & 90 & 20 & 50 & 75 & 90 \\
\midrule
\ours{} & \textbf{.071} & \textbf{.067} & \textbf{.056} & \textbf{.021} & \underline{.192} & \textbf{.184} & \textbf{.173} & \underline{.101} & \underline{.286} & \textbf{.219} & \textbf{.176} & \underline{.105} \\
CLIP-Dissect-e5 & .000 & .010 & .013 & .011 & .078 & .093 & .115 & .094 & .029 & .107 & .088 & .079 \\
AutoInterp & .024 & .019 & .019 & \underline{.016} & .103 & .095 & .103 & \textbf{.104} & .048 & .105 & .113 & \textbf{.111} \\
Delphi & .000 & .029 & .013 & .005 & .092 & .093 & .093 & .089 & .143 & .086 & \underline{.157} & .074 \\
SASC-style & \underline{.071} & \underline{.048} & \underline{.025} & \underline{.016} & \textbf{.218} & \underline{.162} & \underline{.130} & .091 & \textbf{.310} & \underline{.143} & .125 & .053 \\
DeViSE-style & .000 & .019 & .006 & .011 & .129 & .133 & .088 & .085 & .048 & \underline{.143} & .056 & .095 \\
\bottomrule}
\par\vspace{6pt}
\fittable{lcccc:cccc:cccc:cccc}{
\toprule
 & \multicolumn{4}{c:}{NRR@1} & \multicolumn{4}{c:}{MRR} & \multicolumn{4}{c:}{identification} & \multicolumn{4}{c}{naming quality} \\
latent variables & 20 & 50 & 75 & 90 & 20 & 50 & 75 & 90 & 20 & 50 & 75 & 90 & 20 & 50 & 75 & 90 \\
\midrule
\ours{} & \underline{.756} & \textbf{.800} & \textbf{.622} & \textbf{.400} & \underline{.859} & \textbf{.881} & \textbf{.789} & \underline{.648} & \textbf{.689} & \textbf{.689} & \textbf{.511} & \underline{.222} & \textbf{.667} & \underline{.622} & \textbf{.511} & \underline{.156} \\
CLIP-Dissect-e5 & .267 & .467 & .267 & \underline{.400} & .611 & .678 & .567 & \textbf{.667} & .200 & .067 & .133 & .067 & .133 & .067 & .067 & .067 \\
AutoInterp & .333 & .400 & \underline{.400} & .200 & .611 & .644 & \underline{.644} & .544 & .267 & .267 & \underline{.333} & .133 & .267 & .267 & \underline{.333} & .133 \\
Delphi & .333 & .267 & .067 & .267 & .622 & .578 & .456 & .578 & .000 & .133 & .067 & .133 & .000 & .133 & .067 & .133 \\
SASC-style & \textbf{.933} & \underline{.733} & .267 & .333 & \textbf{.967} & \underline{.856} & .556 & .611 & \underline{.600} & \underline{.667} & .133 & .200 & \underline{.600} & \textbf{.667} & .133 & \textbf{.200} \\
DeViSE-style & .467 & .533 & .200 & \underline{.400} & .678 & .722 & .533 & \textbf{.667} & .267 & .267 & .200 & \textbf{.267} & .000 & .067 & .067 & .000 \\
\bottomrule}
\end{table}

\begin{table}[!htb]\centering\scriptsize\setlength{\tabcolsep}{2.4pt}
\caption{HEXACO, no structure supplied. Items above and latent variables below. Bold is the best value per column and underline the second best.}
\label{tab:scale-hexaco}
\fittable{lcccc:cccc:cccc}{
\toprule
 & \multicolumn{4}{c:}{NRR@1} & \multicolumn{4}{c:}{MRR} & \multicolumn{4}{c}{identification} \\
items & 20 & 50 & 75 & 90 & 20 & 50 & 75 & 90 & 20 & 50 & 75 & 90 \\
\midrule
\ours{} & \textbf{.138} & \textbf{.113} & \textbf{.059} & \textbf{.026} & \textbf{.249} & \textbf{.209} & \textbf{.141} & \textbf{.080} & \textbf{.442} & \textbf{.438} & \underline{.348} & \underline{.279} \\
CLIP-Dissect-e5 & .025 & .022 & .017 & .013 & .068 & .074 & .068 & .051 & .204 & .209 & .204 & .190 \\
AutoInterp & .004 & .005 & .004 & .000 & .027 & .027 & .027 & .020 & .083 & .114 & .103 & .108 \\
Delphi & .000 & .003 & .004 & .004 & .023 & .026 & .027 & .030 & .167 & .149 & .146 & .133 \\
SASC-style & \underline{.087} & \underline{.060} & \underline{.039} & \underline{.016} & \underline{.199} & \underline{.160} & \underline{.117} & \underline{.063} & \underline{.421} & \underline{.438} & \textbf{.402} & \textbf{.290} \\
DeViSE-style & .033 & .022 & .014 & .003 & .104 & .082 & .072 & .043 & .312 & .290 & .259 & .209 \\
\bottomrule}
\par\vspace{6pt}
\fittable{lcccc:cccc:cccc:cccc}{
\toprule
 & \multicolumn{4}{c:}{NRR@1} & \multicolumn{4}{c:}{MRR} & \multicolumn{4}{c:}{identification} & \multicolumn{4}{c}{naming quality} \\
latent variables & 20 & 50 & 75 & 90 & 20 & 50 & 75 & 90 & 20 & 50 & 75 & 90 & 20 & 50 & 75 & 90 \\
\midrule
\ours{} & \textbf{.408} & \textbf{.320} & \textbf{.304} & \textbf{.104} & \textbf{.516} & \textbf{.462} & \textbf{.461} & \textbf{.247} & \textbf{.512} & \textbf{.496} & \textbf{.504} & \textbf{.248} & \underline{.384} & \textbf{.432} & \textbf{.456} & \underline{.192} \\
CLIP-Dissect-e5 & .027 & .073 & .093 & \underline{.087} & .157 & .190 & .221 & .184 & .160 & .140 & .187 & .127 & .153 & .127 & .167 & .107 \\
AutoInterp & .053 & .027 & .027 & .053 & .158 & .128 & .140 & .154 & .133 & .127 & .147 & .160 & .133 & .127 & .147 & .160 \\
Delphi & .020 & .067 & .027 & .033 & .110 & .166 & .143 & .126 & .113 & .107 & .081 & .133 & .113 & .107 & .081 & .133 \\
SASC-style & \underline{.267} & \underline{.247} & \underline{.160} & .067 & \underline{.413} & \underline{.391} & \underline{.322} & \underline{.196} & \underline{.460} & \underline{.413} & \underline{.360} & \underline{.220} & \textbf{.460} & \underline{.413} & \underline{.360} & \textbf{.220} \\
DeViSE-style & .067 & .053 & .047 & .033 & .165 & .143 & .145 & .134 & .173 & .160 & .133 & .107 & .047 & .040 & .060 & .047 \\
\bottomrule}
\end{table}

\begin{table}[!htb]\centering\scriptsize\setlength{\tabcolsep}{2.4pt}
\caption{KIMS, no structure supplied. Items above and latent variables below. Bold is the best value per column and underline the second best.}
\label{tab:scale-kims}
\fittable{lcccc:cccc:cccc}{
\toprule
 & \multicolumn{4}{c:}{NRR@1} & \multicolumn{4}{c:}{MRR} & \multicolumn{4}{c}{identification} \\
items & 20 & 50 & 75 & 90 & 20 & 50 & 75 & 90 & 20 & 50 & 75 & 90 \\
\midrule
\ours{} & \textbf{.179} & .050 & \textbf{.069} & \textbf{.074} & \textbf{.370} & \underline{.237} & \textbf{.236} & \textbf{.198} & .179 & \textbf{.310} & \textbf{.214} & \underline{.143} \\
CLIP-Dissect-e5 & .026 & .010 & .041 & .017 & .147 & .127 & .165 & .118 & .184 & .140 & .148 & .093 \\
AutoInterp & .026 & .010 & .014 & .011 & .096 & .092 & .099 & .100 & .077 & .051 & .083 & .069 \\
Delphi & .026 & .040 & .041 & .034 & .117 & .126 & .136 & .124 & .154 & .090 & .083 & .069 \\
SASC-style & \underline{.077} & \textbf{.120} & \underline{.055} & \underline{.057} & \underline{.273} & \textbf{.270} & \underline{.203} & \underline{.175} & \textbf{.333} & \underline{.242} & \underline{.193} & \textbf{.190} \\
DeViSE-style & .051 & \underline{.060} & .034 & .029 & .170 & .174 & .132 & .114 & \underline{.231} & .121 & .131 & .137 \\
\bottomrule}
\par\vspace{6pt}
\fittable{lcccc:cccc:cccc:cccc}{
\toprule
 & \multicolumn{4}{c:}{NRR@1} & \multicolumn{4}{c:}{MRR} & \multicolumn{4}{c:}{identification} & \multicolumn{4}{c}{naming quality} \\
latent variables & 20 & 50 & 75 & 90 & 20 & 50 & 75 & 90 & 20 & 50 & 75 & 90 & 20 & 50 & 75 & 90 \\
\midrule
\ours{} & \textbf{.691} & \textbf{.655} & \underline{.418} & \textbf{.400} & \textbf{.814} & \textbf{.803} & \underline{.633} & \textbf{.605} & \underline{.382} & \textbf{.382} & \underline{.273} & \textbf{.164} & \underline{.382} & \underline{.345} & \underline{.255} & \textbf{.164} \\
CLIP-Dissect-e5 & .300 & .250 & .250 & .350 & .562 & .517 & .525 & .588 & .050 & .050 & .050 & \underline{.150} & .000 & .050 & .050 & \underline{.150} \\
AutoInterp & .100 & .350 & .200 & .200 & .421 & .588 & .487 & .467 & .211 & .158 & .158 & .050 & .211 & .158 & .158 & .050 \\
Delphi & \underline{.500} & .300 & .300 & .150 & \underline{.671} & .537 & .550 & .487 & .250 & .050 & .150 & \underline{.150} & .250 & .050 & .150 & \underline{.150} \\
SASC-style & .300 & \underline{.500} & \textbf{.500} & .200 & .575 & \underline{.688} & \textbf{.700} & .521 & \textbf{.500} & \underline{.350} & \textbf{.350} & \underline{.150} & \textbf{.500} & \textbf{.350} & \textbf{.350} & \underline{.150} \\
DeViSE-style & .250 & .400 & .350 & \underline{.400} & .558 & .633 & .592 & \underline{.604} & .150 & .150 & .050 & .050 & .150 & .150 & .050 & .000 \\
\bottomrule}
\end{table}

\begin{table}[!htb]\centering\scriptsize\setlength{\tabcolsep}{2.4pt}
\caption{RIASEC, no structure supplied. Items above and latent variables below. Bold is the best value per column and underline the second best.}
\label{tab:scale-riasec}
\fittable{lcccc:cccc:cccc}{
\toprule
 & \multicolumn{4}{c:}{NRR@1} & \multicolumn{4}{c:}{MRR} & \multicolumn{4}{c}{identification} \\
items & 20 & 50 & 75 & 90 & 20 & 50 & 75 & 90 & 20 & 50 & 75 & 90 \\
\midrule
\ours{} & \textbf{.062} & \underline{.042} & \textbf{.050} & .009 & \textbf{.220} & \textbf{.198} & \textbf{.194} & \textbf{.124} & \textbf{.271} & \textbf{.258} & \textbf{.200} & \textbf{.116} \\
CLIP-Dissect-e5 & .042 & \textbf{.050} & \underline{.028} & \underline{.014} & .148 & .153 & \underline{.147} & .074 & .083 & .167 & \underline{.194} & .079 \\
AutoInterp & .021 & .025 & .011 & \textbf{.023} & .108 & .085 & .075 & .102 & .125 & .117 & .133 & .088 \\
Delphi & .021 & .008 & .022 & \textbf{.023} & .122 & .085 & .089 & \underline{.105} & .042 & .100 & .106 & .102 \\
SASC-style & \underline{.062} & \textbf{.050} & .022 & \underline{.014} & \underline{.201} & \underline{.179} & .119 & .102 & \underline{.167} & \underline{.217} & .172 & \underline{.112} \\
DeViSE-style & .021 & .008 & .006 & .009 & .180 & .150 & .108 & .079 & .125 & .158 & .100 & .079 \\
\bottomrule}
\par\vspace{6pt}
\fittable{lcccc:cccc:cccc:cccc}{
\toprule
 & \multicolumn{4}{c:}{NRR@1} & \multicolumn{4}{c:}{MRR} & \multicolumn{4}{c:}{identification} & \multicolumn{4}{c}{naming quality} \\
latent variables & 20 & 50 & 75 & 90 & 20 & 50 & 75 & 90 & 20 & 50 & 75 & 90 & 20 & 50 & 75 & 90 \\
\midrule
\ours{} & \textbf{.692} & \textbf{.585} & \textbf{.508} & \textbf{.262} & \textbf{.793} & \textbf{.726} & \textbf{.655} & \textbf{.473} & \underline{.800} & \underline{.662} & \textbf{.562} & \textbf{.369} & \underline{.615} & \underline{.446} & \textbf{.406} & \textbf{.323} \\
CLIP-Dissect-e5 & .100 & .133 & .200 & \underline{.233} & .362 & .377 & .416 & \underline{.453} & .069 & .100 & .167 & .167 & .069 & .100 & .167 & .133 \\
AutoInterp & .167 & .133 & \underline{.233} & .133 & .392 & .401 & .474 & .383 & .233 & .167 & \underline{.333} & .133 & .233 & .167 & \underline{.333} & .133 \\
Delphi & .200 & .067 & \underline{.233} & .100 & .442 & .350 & .446 & .354 & .267 & .067 & .200 & .033 & .267 & .067 & .200 & .033 \\
SASC-style & \underline{.633} & \underline{.433} & \underline{.233} & .133 & \underline{.776} & \underline{.628} & \underline{.487} & .387 & \textbf{.867} & \textbf{.667} & .267 & \underline{.200} & \textbf{.867} & \textbf{.667} & .267 & \underline{.200} \\
DeViSE-style & .233 & .200 & .200 & .167 & .438 & .423 & .412 & .422 & .172 & .200 & .100 & .167 & .172 & .133 & .100 & .133 \\
\bottomrule}
\end{table}

\begin{table}[!htb]\centering\scriptsize\setlength{\tabcolsep}{2.4pt}
\caption{WVS, no structure supplied. Items above and latent variables below. Bold is the best value per column and underline the second best.}
\label{tab:scale-wvs}
\fittable{lcccc:cccc:cccc}{
\toprule
 & \multicolumn{4}{c:}{NRR@1} & \multicolumn{4}{c:}{MRR} & \multicolumn{4}{c}{identification} \\
items & 20 & 50 & 75 & 90 & 20 & 50 & 75 & 90 & 20 & 50 & 75 & 90 \\
\midrule
\ours{} & \underline{.048} & \textbf{.220} & \textbf{.075} & \underline{.032} & \textbf{.269} & \textbf{.372} & \textbf{.210} & \textbf{.184} & .048 & \textbf{.140} & .075 & \underline{.095} \\
CLIP-Dissect-e5 & \textbf{.095} & .020 & \underline{.050} & .032 & .234 & .140 & \underline{.188} & .165 & .048 & .060 & \textbf{.138} & \textbf{.129} \\
AutoInterp & .000 & \underline{.060} & .013 & .032 & .139 & .171 & .127 & .144 & \underline{.095} & .040 & .025 & .063 \\
Delphi & .048 & .000 & .037 & \textbf{.042} & .165 & .116 & .162 & \underline{.171} & .048 & .000 & \underline{.087} & .063 \\
SASC-style & .048 & .020 & .000 & .032 & .246 & \underline{.181} & .127 & .164 & \textbf{.190} & \underline{.100} & .037 & .042 \\
DeViSE-style & \textbf{.095} & .000 & .000 & .021 & \underline{.254} & .174 & .124 & .131 & .048 & .020 & .050 & .032 \\
\bottomrule}
\par\vspace{6pt}
\fittable{lcccc:cccc:cccc:cccc}{
\toprule
 & \multicolumn{4}{c:}{NRR@1} & \multicolumn{4}{c:}{MRR} & \multicolumn{4}{c:}{identification} & \multicolumn{4}{c}{naming quality} \\
latent variables & 20 & 50 & 75 & 90 & 20 & 50 & 75 & 90 & 20 & 50 & 75 & 90 & 20 & 50 & 75 & 90 \\
\midrule
\ours{} & \underline{.350} & \textbf{.333} & \textbf{.217} & \textbf{.233} & \underline{.524} & \textbf{.488} & \textbf{.427} & \textbf{.409} & \underline{.300} & \textbf{.367} & \textbf{.283} & \textbf{.250} & \underline{.283} & \textbf{.317} & \textbf{.267} & \textbf{.250} \\
CLIP-Dissect-e5 & .044 & .089 & .089 & .089 & .270 & .327 & .288 & .291 & .068 & .114 & .068 & .071 & .068 & .114 & .068 & .071 \\
AutoInterp & .111 & .111 & .156 & .089 & .310 & .328 & .343 & .295 & .143 & .116 & .089 & .114 & .143 & .116 & .089 & .114 \\
Delphi & .178 & .133 & .089 & .067 & .368 & .335 & .281 & .290 & .133 & .156 & .089 & .089 & .133 & .156 & .089 & .089 \\
SASC-style & \textbf{.378} & \underline{.244} & \underline{.178} & \underline{.156} & \textbf{.550} & \underline{.451} & \underline{.377} & \underline{.332} & \textbf{.350} & \underline{.238} & \underline{.267} & \underline{.178} & \textbf{.350} & \underline{.238} & \underline{.267} & \underline{.178} \\
DeViSE-style & .156 & .133 & .111 & .111 & .369 & .345 & .302 & .316 & .022 & .067 & .111 & .067 & .022 & .067 & .089 & .067 \\
\bottomrule}
\end{table}

\begin{table}[!htbp]\centering\small\setlength{\tabcolsep}{4pt}
\caption{Identification arm by arm and log by log, our method under both structure sources, with the number of judged cases in parentheses. Vehicle drives are pooled in Table~\ref{tab:st2-robot}.}
\label{tab:perbody}
\fittable{lcccc:cccc}{
\toprule
 & \multicolumn{4}{c:}{discovered graph} & \multicolumn{4}{c}{supplied graph} \\
\cmidrule(lr){2-5}\cmidrule(lr){6-9}
arm or log & 20 & 50 & 75 & 90 & 20 & 50 & 75 & 90 \\
\midrule
Panda & .167 (6) & .250 (16) & .227 (22) & .138 (29) & .000 (6) & .188 (16) & .208 (24) & .207 (29) \\
Sawyer & .143 (7) & .167 (18) & .154 (26) & .125 (32) & .286 (7) & .278 (18) & .333 (27) & .161 (31) \\
IIWA & .000 (9) & .238 (21) & .152 (33) & .125 (40) & .222 (9) & .300 (20) & .242 (33) & .075 (40) \\
UR5e & .125 (8) & .190 (21) & .258 (31) & .079 (38) & .125 (8) & .190 (21) & .188 (32) & .079 (38) \\
Kinova3 & .111 (9) & .227 (22) & .242 (33) & .128 (39) & .111 (9) & .273 (22) & .182 (33) & .075 (40) \\
xArm7 & .125 (8) & .150 (20) & .212 (33) & .105 (38) & .222 (9) & .273 (22) & .303 (33) & .075 (40) \\
UTokyo xArm & .000 (3) & .667 (6) & .400 (10) & .250 (12) & .333 (3) & .833 (6) & .400 (10) & .333 (12) \\
Stanford HYDRA & .333 (3) & .286 (7) & .100 (10) & .231 (13) & .000 (3) & .000 (7) & .100 (10) & .167 (12) \\
ALOHA cup & .667 (6) & .143 (14) & .238 (21) & .120 (25) & .833 (6) & .429 (14) & .286 (21) & .167 (24) \\
\bottomrule
}
\end{table}

\subsection{RQ2: component by component}
\label{apx:j-components}

\begin{figure}[!htb]\centering
\includegraphics[width=\textwidth]{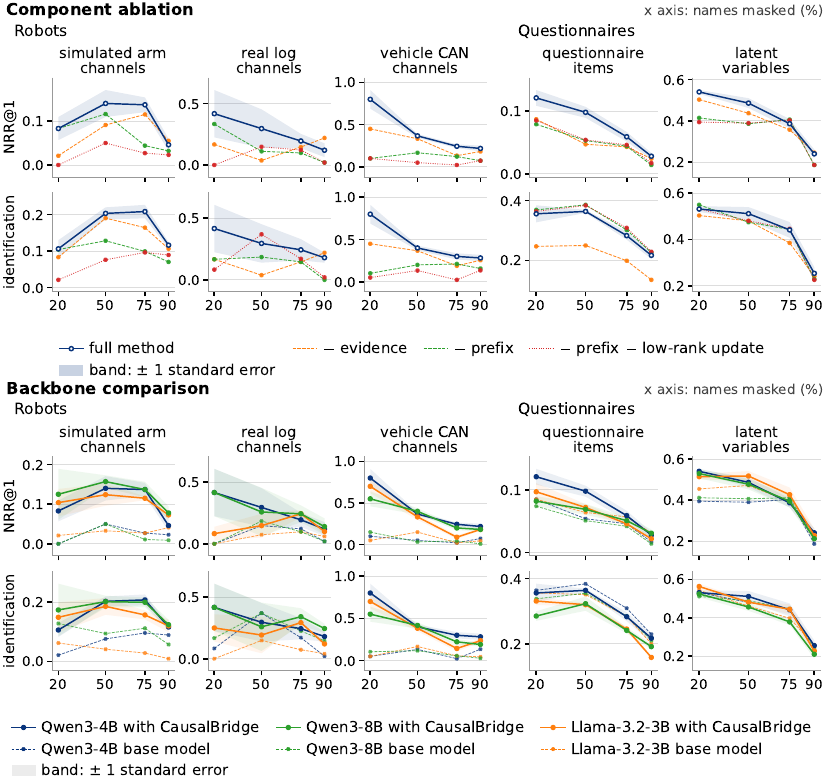}
\caption{One component removed at a time (top) and one recipe on three
backbones (bottom), under the discovered graph. Rows are NRR@1 and
identification. In the bottom block, a solid line is a backbone with \ours{} and the dashed line of the same color is that backbone alone.}
\label{fig:comp}
\end{figure}

The finding of RQ2 is that every component contributes and that the prefix contributes the most. Figure~\ref{fig:comp} (top) shows the variants with one component removed at a time under the discovered graph.

\textit{The prefix carries the structure.} Removing the prefix costs the most. Without it, the NRR@1 curve falls below the curve without the evidence line on the vehicle CAN at every level and on every robot column from 75\% masking on. The variant without the prefix still receives the vector $\hat e_v$
in words, since its evidence line keeps the nearest phrases decoded from
that vector (Appendix~\ref{apx:plugin}). These phrases keep the rough
meaning of the vector and lose how the variable relates to each named
variable, so a few decoded phrases cannot replace the vector. On the vehicle CAN removing the prefix loses 50 to 88\% of the NRR@1 of the full method. On
the simulated arms and the real logs it loses the most at deep masking,
68\% at 75\% on the arms and 83\% at 90\% on the logs. On questionnaire items it loses 25 to 50\% of the NRR@1 of the full method, while identification does not fall. A recognized concept therefore no longer reaches the right variable.

\textit{The evidence line supplies the vocabulary of the name.} On the vehicle CAN the variant without the evidence line lies between the full method and the variant without the prefix at every masking level. On the identification of items it is the lowest curve, 30 to 38\% below
the full method at every masking level. The model phrases the vector $\hat e_v$ better when the evidence line also gives the named variables most correlated with the unnamed variable. On the simulated arms and the real logs, removing the evidence line costs more than removing the prefix through 50\% masking. On the real logs at 20\% masking, the variant without the evidence line reaches .167 and the variant without the prefix .333, with the full method at .417. The evidence line lists the named channels most correlated with an unnamed channel, and their names carry the platform's own vocabulary. From 75\% on removing the prefix costs more.

\textit{The prefix and the evidence line converge on latent variables and at 90\% masking.} On latent variables the prefix and the evidence line both describe the variable through the same items. The four
identification curves therefore lie within .07 of one another at every
level, and the four naming quality values within .10
(Table~\ref{tab:abl-latents}). The full method leads NRR@1 and MRR at 20
and 50\% masking, and the NRR@1 curves cross at 75\%. At 90\% masking the variant without the evidence line meets or passes the full method on the simulated arms, the real logs and the latent variables as well. The evidence line lists the most correlated named variables, and at that level almost none are left to list.

\textit{The prefix determines the type of a channel on every robot column, and the evidence line helps determine its coordinate.} Type and coordinate accuracy (Appendix~\ref{apx:protocol}) separate the contributions of the prefix and the evidence line on all three robot columns (Table~\ref{tab:abl-robot}). The type depends on the prefix everywhere.
At 50\% masking, type accuracy falls on every column without the prefix. It falls from .537 to .347 on the simulated arms, from .518 to .259 on the real logs and from .683 to .300 on the vehicle CAN. The vector $\hat e_v$ therefore carries the type of a channel. Without
the evidence line the type accuracy stays within .07 of the full method
on the simulated arms and the vehicle CAN. On the real logs it falls at
20 and 50\% masking, .370 against .518 at 50\%, since the names in the
evidence line carry the platform's own vocabulary. The coordinate depends
on the evidence line where the names of the most correlated named
channels carry it. On the vehicle CAN the four wheel speeds differ only by a side word. Without the evidence line, coordinate accuracy falls from .800 to .000 at 20\% masking and from .500 to .050 at 75\%.
At 50 and 90\% it stays within .04. On the real logs it falls at every level, from .750 to .625
at 20\% masking and from .579 to .474 at 90\%. On the simulated arms the
coordinate depends on the prefix from 50\% masking on, .181 against .078
at 50\% and .145 against .042 at 90\% without it. The index of a joint follows from its position among the other joints, and no neighbor's name states it. The evidence line is worth
.04 to .05 there through 75\%.

\textit{The low-rank update adds little on its own.} Without the prefix, the model with the low-rank update coincides with the base model alone on both questionnaire columns. It keeps part of the gain on the simulated arms, on the real logs at 20\% masking and on the vehicle CAN at 50 and 75\%.

\textit{Under a supplied graph, the prefix is worth less where the structure can be written into the prompt.} A latent variable of the key can be written into the text
as its item list, while a discovered latent variable exists only as a position in the embedding space. On latent variables the prefix is therefore worth .13 and .10 in NRR@1 at 20 and 50\% masking under the discovered
graph and almost nothing under the supplied graph
(Table~\ref{tab:abl-latents}). On the simulated arms at light masking the text alone even does better, since under the supplied graph the prompt lists a joint's named neighbor with its index.

\subsection{RQ3: across backbones}
\label{apx:j-backbones}

The finding of RQ3 is that one recipe carries across backbones.
Tables~\ref{tab:bb-items}, \ref{tab:bb-latents} and~\ref{tab:bb-robot} carry the numbers of Figure~\ref{fig:comp} (bottom), which shows one recipe on Qwen3-4B, Qwen3-8B and Llama-3.2-3B, the order in which every triple of numbers below is listed.

\textit{One recipe serves every backbone.} With \ours{}, every backbone lies far above itself alone in NRR@1 on the channels, whose labels give a language prior little semantic information. On the vehicle CAN at 20\% masking the three backbones rise from .100, .150 and .050 alone to .800, .550 and .700 with our method. On the real logs they rise from .000 alone to .417, .417 and .083. Type accuracy shows the same gain. Alone, the backbones name the type of .050, .450 and .000 of the unnamed channels of the vehicle CAN at 20\% masking, and with our method .900, .950 and 1.000. From 50\% masking on, the type accuracy of each backbone with our method lies above that backbone alone on every robot column. The gain shrinks where the text already carries more of the name, to a few hundredths on questionnaire items. The item text already indicates the latent variable, and the prefix adds the exact item. 

\textit{The backbones differ in their starting level and in the columns where they perform best.} The two Qwen models lie within .10 of
each other on the simulated arms and the real logs at every level and
on the vehicle CAN from 50\% masking on. The 3B Llama starts lowest on the real logs and reaches the other two by 75\% masking. On latent variables, where the text carries part of the structure, it is the best of the three at 50 and 75\% masking. At 90\% masking the three lie within .10 of one another in NRR@1 on every column. At deep masking the structure therefore sets the level, and the choice of backbone has little effect. The same pattern holds under the supplied graph, where the gain on the vehicle CAN at 20\% masking is .750, .400 and .600 for the three backbones.

\subsection{Tables}
\label{apx:j-tables}

\paragraph{Component and backbone tables.}
These tables carry the numbers behind the component ablation and the
backbone comparison of the main text, under both structure sources, with every
metric each set of variables reports. Bold marks the best value in a column within a set of variables and underline the second best. A zero that a method returns by construction is never marked.

\begingroup
\raggedbottom
\makeatletter
\renewenvironment{table}[1][]{\par\addvspace{\intextsep}\begin{center}\begin{minipage}{\linewidth}\def\@captype{table}}{\end{minipage}\end{center}\addvspace{\intextsep}\par}
\makeatother
\begin{table}[!htbp]\centering\scriptsize\renewcommand{\arraystretch}{0.96}\setlength{\tabcolsep}{2.4pt}
\caption{Questionnaire items, one component removed at a time, under both structure sources.}
\label{tab:abl-items}
\fittable{llcccc:cccc:cccc:cccc}{%
\toprule
 &  & \multicolumn{4}{c:}{NRR@1} & \multicolumn{4}{c:}{MRR} & \multicolumn{4}{c:}{SDA} & \multicolumn{4}{c}{identification} \\
structure & variant & 20 & 50 & 75 & 90 & 20 & 50 & 75 & 90 & 20 & 50 & 75 & 90 & 20 & 50 & 75 & 90 \\
\midrule
supplied graph & \ours{} & \textbf{.126} & \textbf{.097} & \textbf{.064} & \textbf{.038} & \textbf{.254} & \textbf{.220} & \textbf{.179} & \textbf{.121} & \textbf{.200} & \textbf{.176} & \textbf{.147} & \textbf{.129} & \textbf{.351} & \textbf{.314} & \textbf{.304} & \underline{.227} \\
 & $-$ evidence & .067 & \underline{.056} & .042 & \underline{.025} & .177 & .161 & .130 & .092 & .092 & .092 & .081 & .087 & .190 & .218 & .174 & .131 \\
 & $-$ prefix & \underline{.074} & .053 & .050 & .018 & .217 & \underline{.184} & .165 & \underline{.101} & .123 & \underline{.103} & \underline{.122} & \underline{.109} & \underline{.344} & \underline{.314} & \underline{.286} & .227 \\
 & $-$ prefix $-$ low-rank update & \underline{.074} & .051 & \underline{.051} & .017 & \underline{.219} & .182 & \underline{.166} & .100 & \underline{.126} & .102 & .119 & \underline{.109} & .338 & .309 & .280 & \textbf{.228} \\
\midrule
discovered graph & \ours{} & \textbf{.121} & \textbf{.098} & \textbf{.059} & \textbf{.028} & \textbf{.253} & \textbf{.216} & \textbf{.164} & \textbf{.105} & \textbf{.226} & \textbf{.189} & \textbf{.164} & \textbf{.147} & .356 & .364 & .283 & .217 \\
 & $-$ evidence & \underline{.087} & .047 & .043 & \underline{.024} & .196 & .143 & .117 & .084 & .172 & .156 & .141 & .123 & .247 & .250 & .199 & .135 \\
 & $-$ prefix & .079 & .053 & .044 & .014 & \underline{.229} & \underline{.190} & .158 & .091 & .169 & .157 & .152 & .123 & \textbf{.369} & \textbf{.386} & \underline{.302} & \underline{.226} \\
 & $-$ prefix $-$ low-rank update & .085 & \underline{.054} & \underline{.046} & .017 & .227 & .188 & \underline{.162} & \underline{.094} & \underline{.174} & \underline{.163} & \underline{.154} & \underline{.125} & \underline{.364} & \underline{.384} & \textbf{.309} & \textbf{.229} \\
\bottomrule}
\end{table}
\begin{table}[!htb]\centering\scriptsize\setlength{\tabcolsep}{2.4pt}
\caption{Latent variables, one component removed at a time, under both structure sources.}
\label{tab:abl-latents}
\fittable{llcccc:cccc:cccc:cccc}{%
\toprule
 &  & \multicolumn{4}{c:}{NRR@1} & \multicolumn{4}{c:}{MRR} & \multicolumn{4}{c:}{identification} & \multicolumn{4}{c}{naming quality} \\
structure & variant & 20 & 50 & 75 & 90 & 20 & 50 & 75 & 90 & 20 & 50 & 75 & 90 & 20 & 50 & 75 & 90 \\
\midrule
supplied graph & \ours{} & \textbf{.485} & .415 & \textbf{.415} & \underline{.238} & \textbf{.624} & \underline{.560} & \textbf{.551} & \textbf{.400} & \textbf{.627} & \textbf{.537} & \textbf{.535} & \textbf{.351} & \textbf{.596} & \textbf{.518} & \textbf{.500} & \textbf{.332} \\
 & $-$ evidence & .446 & \textbf{.450} & \underline{.381} & \textbf{.262} & .579 & \textbf{.566} & \underline{.507} & \underline{.396} & \underline{.539} & .479 & \underline{.496} & \underline{.290} & \underline{.535} & .475 & \underline{.485} & \underline{.278} \\
 & $-$ prefix & .462 & .404 & .319 & .208 & .600 & .550 & .475 & .371 & .496 & .465 & .465 & .281 & .496 & .465 & .462 & .273 \\
 & $-$ prefix $-$ low-rank update & \underline{.477} & \underline{.419} & .331 & .212 & \underline{.614} & .560 & .480 & .376 & .512 & \underline{.492} & .454 & .285 & .512 & \underline{.492} & .450 & .277 \\
\midrule
discovered graph & \ours{} & \textbf{.540} & \textbf{.486} & .386 & \underline{.240} & \textbf{.660} & \textbf{.623} & .560 & \underline{.424} & \underline{.531} & \textbf{.511} & .441 & \textbf{.254} & .446 & .426 & \underline{.390} & .217 \\
 & $-$ evidence & \underline{.503} & \underline{.437} & .357 & \textbf{.246} & \underline{.637} & \underline{.582} & .530 & \textbf{.428} & .503 & \underline{.480} & .385 & \underline{.241} & .491 & .471 & .382 & \textbf{.229} \\
 & $-$ prefix & .414 & .386 & \textbf{.406} & .186 & .590 & .565 & \textbf{.581} & .404 & \textbf{.549} & .474 & \underline{.443} & .229 & \textbf{.537} & \underline{.474} & \textbf{.437} & \underline{.226} \\
 & $-$ prefix $-$ low-rank update & .394 & .389 & \underline{.403} & .186 & .576 & .568 & \underline{.579} & .401 & .531 & \underline{.480} & \textbf{.446} & .226 & \underline{.520} & \textbf{.480} & \textbf{.437} & .223 \\
\bottomrule}
\end{table}

\begin{table}[!htbp]\centering\scriptsize\renewcommand{\arraystretch}{0.96}\setlength{\tabcolsep}{2.4pt}
\caption{Questionnaire items on three language models, each with \ours{} and as the base model, under both structure sources.}
\label{tab:bb-items}
\fittable{llcccc:cccc:cccc:cccc}{%
\toprule
 &  & \multicolumn{4}{c:}{NRR@1} & \multicolumn{4}{c:}{MRR} & \multicolumn{4}{c:}{SDA} & \multicolumn{4}{c}{identification} \\
structure & model & 20 & 50 & 75 & 90 & 20 & 50 & 75 & 90 & 20 & 50 & 75 & 90 & 20 & 50 & 75 & 90 \\
\midrule
supplied graph & Qwen3-4B & \textbf{.126} & \textbf{.097} & \textbf{.064} & \textbf{.038} & \textbf{.254} & \textbf{.220} & \textbf{.179} & \textbf{.121} & \textbf{.200} & \textbf{.176} & \textbf{.147} & \underline{.129} & \textbf{.351} & \textbf{.314} & \textbf{.304} & \underline{.227} \\
 & Qwen3-4B base model & .074 & .051 & .051 & .017 & .219 & .182 & \underline{.166} & .100 & .126 & .102 & .119 & .109 & .338 & \underline{.309} & \underline{.280} & \textbf{.228} \\
 & Qwen3-8B & \underline{.105} & \underline{.079} & .056 & .032 & \underline{.222} & \underline{.199} & .158 & \underline{.113} & \underline{.182} & \underline{.151} & .137 & \textbf{.132} & \underline{.346} & .294 & .260 & .214 \\
 & Qwen3-8B base model & .069 & .053 & .036 & .012 & .206 & .180 & .146 & .095 & .136 & .107 & .119 & .115 & .321 & .296 & .264 & .223 \\
 & Llama-3.2-3B & .067 & .071 & \underline{.062} & \underline{.033} & .188 & .181 & .162 & .107 & .113 & .116 & .119 & .120 & .282 & .254 & .258 & .180 \\
 & Llama-3.2-3B base model & .077 & .068 & .057 & .032 & .192 & .176 & .157 & .106 & .128 & .123 & \underline{.145} & \underline{.129} & .277 & .270 & .256 & .209 \\
\midrule
discovered graph & Qwen3-4B & \textbf{.121} & \textbf{.098} & \textbf{.059} & \underline{.028} & \textbf{.253} & \textbf{.216} & \textbf{.164} & \textbf{.105} & \textbf{.226} & \textbf{.189} & \underline{.164} & \textbf{.147} & \underline{.356} & \underline{.364} & .283 & \underline{.217} \\
 & Qwen3-4B base model & .085 & .054 & .046 & .017 & \underline{.227} & \underline{.188} & \underline{.162} & .094 & .174 & .163 & .154 & .125 & \textbf{.364} & \textbf{.384} & \textbf{.309} & \textbf{.229} \\
 & Qwen3-8B & .082 & .069 & .051 & \textbf{.031} & .199 & .181 & .140 & .099 & .162 & .152 & .147 & \underline{.138} & .285 & .323 & .241 & .191 \\
 & Qwen3-8B base model & .074 & .051 & .042 & .014 & .221 & .187 & .156 & .089 & .167 & .167 & .161 & .128 & .338 & .355 & \underline{.285} & .210 \\
 & Llama-3.2-3B & \underline{.097} & \underline{.072} & .046 & .022 & .226 & .181 & .139 & .092 & \underline{.208} & \underline{.173} & .141 & .129 & .331 & .320 & .245 & .158 \\
 & Llama-3.2-3B base model & .085 & .064 & \underline{.057} & \textbf{.031} & .208 & .183 & .155 & \underline{.100} & .190 & .150 & \textbf{.165} & .135 & \underline{.356} & .351 & .284 & .200 \\
\bottomrule}
\end{table}

\begin{table}[!htbp]\centering\scriptsize\renewcommand{\arraystretch}{0.96}\setlength{\tabcolsep}{2.4pt}
\caption{Robot channels, one component removed at a time, under both structure sources.}
\label{tab:abl-robot}
\fittable{lllcccc:cccc:cccc}{%
\toprule
 &  &  & \multicolumn{4}{c:}{NRR@1} & \multicolumn{4}{c:}{MRR} & \multicolumn{4}{c}{SDA} \\
structure & variables & variant & 20 & 50 & 75 & 90 & 20 & 50 & 75 & 90 & 20 & 50 & 75 & 90 \\
\midrule
supplied graph & simulated arms & \ours{} & \underline{.146} & \textbf{.198} & \textbf{.181} & \textbf{.055} & \underline{.317} & \underline{.341} & \textbf{.354} & \textbf{.225} & .170 & \underline{.259} & \underline{.222} & .117 \\
 &  & $-$ evidence & .104 & .140 & \underline{.099} & \underline{.055} & .240 & \textbf{.350} & \underline{.277} & \underline{.216} & .149 & .164 & .131 & .150 \\
 &  & $-$ prefix & \textbf{.333} & \underline{.198} & .088 & .018 & \textbf{.402} & .288 & .206 & .122 & \textbf{.447} & \textbf{.293} & \textbf{.250} & \underline{.183} \\
 &  & $-$ prefix $-$ low-rank update & .021 & .083 & .033 & .023 & .164 & .193 & .135 & .148 & \underline{.255} & \textbf{.293} & .205 & \textbf{.211} \\
\midrule
supplied graph & real logs & \ours{} & \textbf{.500} & \textbf{.333} & \underline{.244} & \underline{.160} & \underline{.566} & \textbf{.451} & \underline{.369} & \underline{.276} & \underline{.545} & \textbf{.440} & \underline{.333} & \underline{.298} \\
 &  & $-$ evidence & \underline{.500} & \underline{.259} & \textbf{.268} & \textbf{.200} & \textbf{.567} & .369 & \textbf{.385} & \textbf{.351} & .545 & .360 & \textbf{.359} & \textbf{.319} \\
 &  & $-$ prefix & .250 & .148 & .171 & .060 & .427 & .308 & .310 & .194 & .545 & .400 & .308 & .277 \\
 &  & $-$ prefix $-$ low-rank update & .167 & .185 & .146 & .020 & .434 & \underline{.387} & .278 & .176 & \textbf{.727} & \underline{.440} & .333 & .213 \\
\midrule
supplied graph & vehicle CAN & \ours{} & \textbf{.750} & \textbf{.300} & \textbf{.156} & \textbf{.145} & \textbf{.775} & \textbf{.553} & \textbf{.317} & \textbf{.378} & \textbf{1.000} & \underline{.300} & \underline{.400} & \textbf{.133} \\
 &  & $-$ evidence & \underline{.250} & \underline{.250} & .044 & .009 & \underline{.442} & \underline{.449} & .259 & .295 & .000 & \textbf{.333} & .000 & .000 \\
 &  & $-$ prefix & .000 & .100 & \underline{.111} & .064 & .226 & .327 & \underline{.312} & .281 & .000 & .300 & \textbf{.500} & .000 \\
 &  & $-$ prefix $-$ low-rank update & .000 & .083 & .067 & \underline{.091} & .269 & .322 & .269 & \underline{.297} & .000 & .100 & \textbf{.500} & .000 \\
\midrule
discovered graph & simulated arms & \ours{} & \textbf{.083} & \textbf{.140} & \textbf{.137} & \underline{.046} & \textbf{.224} & \textbf{.286} & \textbf{.278} & \textbf{.194} & .128 & \underline{.190} & .153 & .164 \\
 &  & $-$ evidence & .021 & .091 & \underline{.115} & \textbf{.055} & .156 & \underline{.256} & \underline{.251} & \underline{.189} & .085 & .147 & .136 & .150 \\
 &  & $-$ prefix & \underline{.083} & \underline{.116} & .044 & .032 & \underline{.220} & .223 & .148 & .125 & \textbf{.277} & \textbf{.233} & \textbf{.199} & \underline{.169} \\
 &  & $-$ prefix $-$ low-rank update & .000 & .050 & .027 & .023 & .147 & .176 & .128 & .131 & \underline{.191} & .103 & \underline{.182} & \textbf{.202} \\
\midrule
discovered graph & real logs & \ours{} & \textbf{.417} & \textbf{.296} & \textbf{.195} & \underline{.120} & \underline{.485} & \textbf{.385} & \textbf{.314} & \underline{.248} & \underline{.455} & \textbf{.360} & \textbf{.359} & \underline{.255} \\
 &  & $-$ evidence & .167 & .037 & \underline{.146} & \textbf{.220} & .286 & .190 & \underline{.296} & \textbf{.347} & .182 & .120 & .179 & \textbf{.319} \\
 &  & $-$ prefix & \underline{.333} & .111 & .098 & .020 & \textbf{.533} & .265 & .229 & .154 & \textbf{.636} & \underline{.320} & \underline{.282} & .149 \\
 &  & $-$ prefix $-$ low-rank update & .000 & \underline{.148} & .122 & .020 & .297 & \underline{.334} & .284 & .181 & .455 & \underline{.320} & \underline{.282} & .128 \\
\midrule
discovered graph & vehicle CAN & \ours{} & \textbf{.800} & \textbf{.367} & \textbf{.244} & \textbf{.218} & \textbf{.859} & \textbf{.570} & \textbf{.412} & \textbf{.397} & \textbf{.800} & \textbf{.367} & \textbf{.500} & \textbf{.367} \\
 &  & $-$ evidence & \underline{.450} & \underline{.333} & \underline{.133} & \underline{.182} & \underline{.633} & \underline{.528} & \underline{.365} & \underline{.370} & .000 & \underline{.333} & .050 & \underline{.333} \\
 &  & $-$ prefix & .100 & .167 & .122 & .073 & .320 & .395 & .324 & .297 & \underline{.100} & .200 & \underline{.500} & .000 \\
 &  & $-$ prefix $-$ low-rank update & .100 & .050 & .022 & .073 & .343 & .279 & .236 & .265 & .000 & .100 & \underline{.500} & .000 \\
\bottomrule}
\par\vspace{3pt}
\fittable{lllcccc:cccc:cccc}{%
\toprule
 &  &  & \multicolumn{4}{c:}{identification} & \multicolumn{4}{c:}{type acc} & \multicolumn{4}{c}{coord acc} \\
structure & variables & variant & 20 & 50 & 75 & 90 & 20 & 50 & 75 & 90 & 20 & 50 & 75 & 90 \\
\midrule
supplied graph & simulated arms & \ours{} & \underline{.167} & \textbf{.252} & \textbf{.242} & \textbf{.106} & \textbf{.542} & \underline{.612} & \textbf{.703} & \textbf{.525} & \underline{.149} & \textbf{.233} & \textbf{.216} & \underline{.132} \\
 &  & $-$ evidence & .130 & .215 & .155 & \underline{.105} & .438 & \textbf{.719} & \underline{.626} & \underline{.484} & .106 & .190 & .125 & \textbf{.136} \\
 &  & $-$ prefix & \textbf{.354} & \underline{.229} & \underline{.162} & .069 & \underline{.458} & .347 & .297 & .173 & \textbf{.340} & \underline{.224} & \underline{.153} & .052 \\
 &  & $-$ prefix $-$ low-rank update & .109 & .190 & .078 & .084 & .375 & .281 & .214 & .288 & .043 & .172 & .074 & .061 \\
\midrule
supplied graph & real logs & \ours{} & \textbf{.500} & \textbf{.407} & \underline{.268} & \underline{.208} & \textbf{.750} & \textbf{.593} & \underline{.610} & \underline{.500} & \textbf{1.000} & \textbf{.682} & .455 & \textbf{.553} \\
 &  & $-$ evidence & \underline{.500} & .259 & \textbf{.293} & \textbf{.260} & \underline{.667} & \underline{.556} & \textbf{.658} & \textbf{.640} & \underline{.875} & \underline{.636} & .455 & .447 \\
 &  & $-$ prefix & .273 & .231 & .250 & .149 & .333 & .333 & .366 & .240 & .750 & .500 & \textbf{.606} & \underline{.526} \\
 &  & $-$ prefix $-$ low-rank update & .333 & \underline{.333} & .128 & .120 & .167 & .222 & .195 & .260 & \underline{.875} & .591 & \underline{.485} & .500 \\
\midrule
supplied graph & vehicle CAN & \ours{} & \textbf{.750} & \textbf{.317} & \textbf{.216} & \textbf{.200} & \textbf{.750} & \textbf{.650} & \textbf{.289} & \textbf{.382} & \textbf{1.000} & \textbf{.300} & \textbf{.400} & \textbf{.133} \\
 &  & $-$ evidence & \underline{.250} & \underline{.317} & .115 & .100 & \underline{.750} & \underline{.567} & \underline{.267} & .273 & .000 & \underline{.267} & .000 & .000 \\
 &  & $-$ prefix & .000 & .100 & \underline{.213} & \underline{.136} & .200 & .217 & .111 & \underline{.291} & .000 & .000 & .000 & .000 \\
 &  & $-$ prefix $-$ low-rank update & .105 & .133 & .091 & .074 & .250 & .483 & .000 & .264 & .000 & .000 & .000 & .000 \\
\midrule
discovered graph & simulated arms & \ours{} & \textbf{.106} & \textbf{.203} & \textbf{.208} & \textbf{.116} & .396 & \underline{.537} & \textbf{.577} & \textbf{.420} & \underline{.106} & \textbf{.181} & \textbf{.159} & \textbf{.145} \\
 &  & $-$ evidence & .083 & \underline{.190} & \underline{.164} & \underline{.106} & .375 & \textbf{.579} & \underline{.522} & \underline{.415} & .064 & \underline{.129} & \underline{.119} & \underline{.145} \\
 &  & $-$ prefix & \underline{.104} & .128 & .099 & .070 & \textbf{.438} & .347 & .286 & .256 & \textbf{.128} & .078 & .091 & .042 \\
 &  & $-$ prefix $-$ low-rank update & .021 & .076 & .096 & .089 & \underline{.417} & .364 & .225 & .265 & .021 & .035 & .040 & .052 \\
\midrule
discovered graph & real logs & \ours{} & \textbf{.417} & \underline{.296} & \textbf{.244} & \underline{.180} & \textbf{.667} & \textbf{.518} & \underline{.512} & \underline{.460} & \underline{.750} & \textbf{.545} & \textbf{.606} & \textbf{.579} \\
 &  & $-$ evidence & \underline{.167} & .038 & .146 & \textbf{.220} & \underline{.583} & \underline{.370} & \textbf{.585} & \textbf{.600} & .625 & .409 & .485 & \underline{.474} \\
 &  & $-$ prefix & \underline{.167} & .185 & .146 & .000 & .333 & .259 & .342 & .220 & \textbf{1.000} & .455 & \underline{.606} & .447 \\
 &  & $-$ prefix $-$ low-rank update & .083 & \textbf{.370} & \underline{.171} & .020 & .083 & .259 & .293 & .280 & .625 & \underline{.500} & .485 & \underline{.474} \\
\midrule
discovered graph & vehicle CAN & \ours{} & \textbf{.800} & \textbf{.400} & \textbf{.300} & \textbf{.282} & \underline{.900} & \textbf{.683} & \textbf{.356} & \textbf{.391} & \textbf{.800} & \textbf{.367} & \textbf{.500} & \textbf{.367} \\
 &  & $-$ evidence & \underline{.450} & \underline{.367} & .189 & \underline{.255} & \textbf{.950} & \underline{.667} & \underline{.344} & \underline{.364} & .000 & \underline{.333} & \underline{.050} & \underline{.333} \\
 &  & $-$ prefix & .100 & .200 & \underline{.211} & .157 & .150 & .300 & .133 & .318 & \underline{.100} & .000 & .000 & .000 \\
 &  & $-$ prefix $-$ low-rank update & .050 & .133 & .023 & .135 & .050 & .367 & .000 & .009 & .000 & .000 & .000 & .000 \\
\bottomrule}
\end{table}

\begin{table}[!htb]\centering\scriptsize\setlength{\tabcolsep}{2.4pt}
\caption{Latent variables on three language models, each with \ours{} and as the base model, under both structure sources.}
\label{tab:bb-latents}
\fittable{llcccc:cccc:cccc:cccc}{%
\toprule
 &  & \multicolumn{4}{c:}{NRR@1} & \multicolumn{4}{c:}{MRR} & \multicolumn{4}{c:}{identification} & \multicolumn{4}{c}{naming quality} \\
structure & model & 20 & 50 & 75 & 90 & 20 & 50 & 75 & 90 & 20 & 50 & 75 & 90 & 20 & 50 & 75 & 90 \\
\midrule
supplied graph & Qwen3-4B & \underline{.485} & .415 & \textbf{.415} & \underline{.238} & \underline{.624} & \underline{.560} & \textbf{.551} & .400 & \textbf{.627} & \underline{.537} & \textbf{.535} & \underline{.351} & \textbf{.596} & .518 & \underline{.500} & .332 \\
 & Qwen3-4B base model & .477 & \underline{.419} & .331 & .212 & .614 & \underline{.560} & .480 & .376 & .512 & .492 & .454 & .285 & .512 & .492 & .450 & .277 \\
 & Qwen3-8B & .454 & .404 & .346 & .235 & .582 & .552 & .512 & \underline{.405} & \underline{.573} & \underline{.537} & .504 & \textbf{.377} & \underline{.569} & \underline{.537} & .492 & \textbf{.369} \\
 & Qwen3-8B base model & .385 & .323 & .338 & .219 & .539 & .490 & .488 & .381 & .486 & .469 & .462 & .279 & .486 & .454 & .435 & .244 \\
 & Llama-3.2-3B & \textbf{.535} & \textbf{.454} & \underline{.373} & \textbf{.273} & \textbf{.646} & \textbf{.588} & \underline{.528} & \textbf{.423} & .558 & \textbf{.560} & \underline{.523} & .350 & .558 & \textbf{.560} & \textbf{.515} & \underline{.350} \\
 & Llama-3.2-3B base model & .369 & .350 & .304 & .185 & .521 & .506 & .467 & .363 & .485 & .473 & .462 & .292 & .485 & .473 & .462 & .292 \\
\midrule
discovered graph & Qwen3-4B & \textbf{.540} & \underline{.486} & .386 & \textbf{.240} & \textbf{.660} & \underline{.623} & .560 & \textbf{.424} & .531 & \textbf{.511} & .441 & \underline{.254} & .446 & .426 & .390 & .217 \\
 & Qwen3-4B base model & .394 & .389 & .403 & .186 & .576 & .568 & .579 & .401 & .531 & .480 & \underline{.446} & .226 & .520 & \textbf{.480} & \underline{.437} & .223 \\
 & Qwen3-8B & \underline{.529} & .477 & .394 & .211 & .646 & .613 & .569 & .417 & .524 & .456 & .378 & .209 & .513 & .427 & .361 & .203 \\
 & Qwen3-8B base model & .411 & .406 & \underline{.411} & .214 & .583 & .574 & \underline{.580} & .413 & .511 & \underline{.489} & \textbf{.447} & .239 & .494 & .457 & .390 & \underline{.225} \\
 & Llama-3.2-3B & .514 & \textbf{.517} & \textbf{.426} & \underline{.226} & \underline{.650} & \textbf{.651} & \textbf{.583} & \textbf{.424} & \textbf{.563} & .481 & .444 & .226 & \textbf{.563} & \underline{.479} & \textbf{.438} & .223 \\
 & Llama-3.2-3B base model & .454 & .471 & .391 & .211 & .609 & .621 & .573 & \underline{.423} & \underline{.543} & .461 & .397 & \textbf{.256} & \underline{.543} & .461 & .391 & \textbf{.256} \\
\bottomrule}
\end{table}

\begin{table}[!htbp]\centering\scriptsize\renewcommand{\arraystretch}{0.96}\setlength{\tabcolsep}{2.4pt}
\caption{Robot channels on three language models, each with \ours{} and as the base model, under both structure sources.}
\label{tab:bb-robot}
\fittable{lllcccc:cccc:cccc}{%
\toprule
 &  &  & \multicolumn{4}{c:}{NRR@1} & \multicolumn{4}{c:}{MRR} & \multicolumn{4}{c}{SDA} \\
structure & variables & model & 20 & 50 & 75 & 90 & 20 & 50 & 75 & 90 & 20 & 50 & 75 & 90 \\
\midrule
supplied graph & simulated arms & Qwen3-4B & .146 & \textbf{.198} & \textbf{.181} & .055 & .317 & \textbf{.341} & \textbf{.354} & .225 & .170 & \underline{.259} & \textbf{.222} & .117 \\
 &  & Qwen3-4B base model & .021 & .083 & .033 & .023 & .164 & .193 & .135 & .148 & \underline{.255} & \textbf{.293} & .205 & \underline{.211} \\
 &  & Qwen3-8B & \underline{.167} & \underline{.182} & \underline{.170} & \textbf{.119} & \underline{.326} & \underline{.329} & \underline{.348} & \textbf{.265} & .213 & \underline{.259} & \underline{.216} & .197 \\
 &  & Qwen3-8B base model & .083 & .058 & .027 & .046 & .221 & .168 & .141 & .165 & \underline{.255} & .233 & .210 & .207 \\
 &  & Llama-3.2-3B & \textbf{.271} & .157 & .154 & \underline{.105} & \textbf{.393} & .300 & .304 & \underline{.251} & \textbf{.319} & \underline{.259} & .182 & \textbf{.225} \\
 &  & Llama-3.2-3B base model & .021 & .017 & .044 & .032 & .137 & .113 & .141 & .126 & .213 & .207 & .148 & .188 \\
\midrule
supplied graph & real logs & Qwen3-4B & \textbf{.500} & \textbf{.333} & \underline{.244} & \underline{.160} & \textbf{.566} & \underline{.451} & \underline{.369} & \underline{.276} & \underline{.545} & \underline{.440} & \underline{.333} & \underline{.298} \\
 &  & Qwen3-4B base model & .167 & .185 & .146 & .020 & \underline{.434} & .387 & .278 & .176 & \textbf{.727} & \underline{.440} & \underline{.333} & .213 \\
 &  & Qwen3-8B & .250 & \textbf{.333} & \textbf{.317} & \textbf{.200} & .395 & \textbf{.466} & \textbf{.456} & \textbf{.326} & .364 & .320 & \underline{.333} & \textbf{.319} \\
 &  & Qwen3-8B base model & .083 & .074 & .049 & .020 & .308 & .286 & .207 & .145 & \underline{.545} & \textbf{.520} & .256 & .234 \\
 &  & Llama-3.2-3B & \underline{.333} & \underline{.259} & \underline{.244} & .120 & .430 & .367 & .349 & .253 & .364 & \underline{.440} & \textbf{.385} & \textbf{.319} \\
 &  & Llama-3.2-3B base model & .000 & .074 & .098 & .080 & .177 & .266 & .228 & .195 & .000 & .200 & .179 & .213 \\
\midrule
supplied graph & vehicle CAN & Qwen3-4B & \textbf{.750} & \underline{.300} & .156 & .145 & \underline{.775} & \underline{.553} & .317 & \underline{.378} & \textbf{1.000} & \underline{.300} & \underline{.400} & .133 \\
 &  & Qwen3-4B base model & .000 & .083 & .067 & .091 & .269 & .322 & .269 & .297 & .000 & .100 & \textbf{.500} & .000 \\
 &  & Qwen3-8B & .450 & \textbf{.417} & \textbf{.256} & \underline{.154} & .585 & \textbf{.584} & \underline{.401} & .327 & .100 & \textbf{.333} & \textbf{.500} & \underline{.300} \\
 &  & Qwen3-8B base model & .050 & .000 & .044 & .027 & .332 & .249 & .253 & .253 & .000 & .000 & \textbf{.500} & .000 \\
 &  & Llama-3.2-3B & \underline{.700} & \underline{.300} & \underline{.244} & \textbf{.182} & \textbf{.776} & .528 & \textbf{.418} & \textbf{.398} & \underline{.600} & \textbf{.333} & \textbf{.500} & \textbf{.333} \\
 &  & Llama-3.2-3B base model & .100 & .167 & .078 & .145 & .281 & .383 & .247 & .331 & .000 & .133 & .000 & .267 \\
\midrule
discovered graph & simulated arms & Qwen3-4B & .083 & \underline{.140} & \textbf{.137} & .046 & .224 & \underline{.286} & \underline{.278} & \underline{.194} & .128 & .190 & .153 & .164 \\
 &  & Qwen3-4B base model & .000 & .050 & .027 & .023 & .147 & .176 & .128 & .131 & \underline{.191} & .103 & \underline{.182} & \underline{.202} \\
 &  & Qwen3-8B & \textbf{.125} & \textbf{.157} & \textbf{.137} & \textbf{.078} & \underline{.255} & \textbf{.288} & \textbf{.279} & \textbf{.198} & \underline{.191} & \underline{.198} & .176 & \textbf{.207} \\
 &  & Qwen3-8B base model & .000 & .050 & .011 & .009 & .155 & .184 & .121 & .113 & .085 & .121 & \textbf{.188} & .164 \\
 &  & Llama-3.2-3B & \underline{.104} & .124 & \underline{.115} & \underline{.073} & \textbf{.256} & .261 & .251 & .193 & .170 & \textbf{.216} & .142 & .197 \\
 &  & Llama-3.2-3B base model & .021 & .033 & .027 & .041 & .162 & .135 & .126 & .137 & \textbf{.213} & .190 & .165 & .169 \\
\midrule
discovered graph & real logs & Qwen3-4B & \textbf{.417} & \textbf{.296} & \underline{.195} & \underline{.120} & \underline{.485} & \underline{.385} & .314 & \underline{.248} & \textbf{.455} & \underline{.360} & \textbf{.359} & .255 \\
 &  & Qwen3-4B base model & .000 & .148 & .122 & .020 & .297 & .334 & .284 & .181 & \textbf{.455} & .320 & \underline{.282} & .128 \\
 &  & Qwen3-8B & \textbf{.417} & \underline{.259} & \textbf{.244} & \textbf{.140} & \textbf{.508} & \textbf{.393} & \textbf{.386} & \textbf{.268} & \textbf{.455} & .320 & \textbf{.359} & \underline{.277} \\
 &  & Qwen3-8B base model & .000 & .185 & .098 & .020 & .231 & .353 & .231 & .158 & \underline{.273} & \textbf{.400} & .256 & .191 \\
 &  & Llama-3.2-3B & \underline{.083} & .148 & \textbf{.244} & .100 & .233 & .308 & \underline{.355} & .231 & .091 & .280 & \textbf{.359} & \textbf{.383} \\
 &  & Llama-3.2-3B base model & .000 & .074 & .098 & .060 & .155 & .228 & .218 & .212 & .091 & .160 & .205 & .170 \\
\midrule
discovered graph & vehicle CAN & Qwen3-4B & \textbf{.800} & \underline{.367} & \textbf{.244} & \textbf{.218} & \textbf{.859} & \underline{.570} & \textbf{.412} & \textbf{.397} & \textbf{.800} & \textbf{.367} & \textbf{.500} & \textbf{.367} \\
 &  & Qwen3-4B base model & .100 & .050 & .022 & .073 & .343 & .279 & .236 & .265 & .000 & .100 & \textbf{.500} & .000 \\
 &  & Qwen3-8B & .550 & \textbf{.400} & \underline{.200} & \underline{.182} & .679 & \textbf{.582} & \underline{.375} & .353 & .200 & \textbf{.367} & \textbf{.500} & \underline{.333} \\
 &  & Qwen3-8B base model & .150 & .033 & .044 & .009 & .320 & .287 & .239 & .233 & .000 & .000 & \underline{.150} & .000 \\
 &  & Llama-3.2-3B & \underline{.700} & .333 & .089 & \underline{.182} & \underline{.800} & .534 & .309 & \underline{.376} & \underline{.400} & \underline{.300} & .000 & \underline{.333} \\
 &  & Llama-3.2-3B base model & .050 & .150 & .022 & .045 & .245 & .329 & .195 & .204 & .000 & .100 & .000 & .033 \\
\bottomrule}
\end{table}

\begin{table}[!htbp]\centering\scriptsize\renewcommand{\arraystretch}{0.96}\setlength{\tabcolsep}{2.4pt}
\addtocounter{table}{-1}
\caption{Continued with the same rows, for identification, type accuracy and coordinate accuracy.}
\fittable{lllcccc:cccc:cccc}{%
\toprule
 &  &  & \multicolumn{4}{c:}{identification} & \multicolumn{4}{c:}{type acc} & \multicolumn{4}{c}{coord acc} \\
structure & variables & model & 20 & 50 & 75 & 90 & 20 & 50 & 75 & 90 & 20 & 50 & 75 & 90 \\
\midrule
supplied graph & simulated arms & Qwen3-4B & .167 & .252 & \textbf{.242} & .106 & .542 & \textbf{.612} & \textbf{.703} & \textbf{.525} & .149 & .233 & \underline{.216} & .132 \\
 &  & Qwen3-4B base model & .109 & .190 & .078 & .084 & .375 & .281 & .214 & .288 & .043 & .172 & .074 & .061 \\
 &  & Qwen3-8B & \underline{.229} & \textbf{.267} & .222 & \textbf{.188} & \textbf{.604} & \textbf{.612} & \underline{.692} & \underline{.516} & \underline{.170} & \textbf{.267} & \textbf{.222} & \textbf{.183} \\
 &  & Qwen3-8B base model & .085 & .137 & .084 & .083 & .458 & .281 & .220 & .251 & .043 & .103 & .074 & .070 \\
 &  & Llama-3.2-3B & \textbf{.312} & \underline{.264} & \underline{.228} & \underline{.165} & \underline{.583} & \underline{.488} & .588 & .498 & \textbf{.298} & \underline{.241} & .188 & \underline{.174} \\
 &  & Llama-3.2-3B base model & .000 & .033 & .034 & .028 & .312 & .215 & .203 & .192 & .000 & .000 & .017 & .009 \\
\midrule
supplied graph & real logs & Qwen3-4B & \textbf{.500} & \textbf{.407} & \underline{.268} & .208 & \textbf{.750} & \underline{.593} & \underline{.610} & \textbf{.500} & \textbf{1.000} & \textbf{.682} & .455 & \underline{.553} \\
 &  & Qwen3-4B base model & .333 & \underline{.333} & .128 & .120 & .167 & .222 & .195 & .260 & \underline{.875} & .591 & .485 & .500 \\
 &  & Qwen3-8B & .333 & \underline{.333} & \textbf{.308} & \textbf{.260} & \underline{.500} & \textbf{.630} & \textbf{.707} & \underline{.440} & .500 & .500 & \underline{.545} & .474 \\
 &  & Qwen3-8B base model & .333 & .296 & .171 & .080 & .167 & .259 & .415 & .220 & \underline{.875} & \underline{.636} & \textbf{.576} & .500 \\
 &  & Llama-3.2-3B & \underline{.364} & .259 & \underline{.268} & \underline{.220} & \underline{.500} & .444 & .488 & \underline{.440} & \underline{.875} & .591 & \underline{.545} & \textbf{.605} \\
 &  & Llama-3.2-3B base model & .167 & .111 & .103 & .080 & .000 & .185 & .195 & .120 & .000 & .136 & .030 & .053 \\
\midrule
supplied graph & vehicle CAN & Qwen3-4B & \textbf{.750} & .317 & .216 & .200 & \underline{.750} & \underline{.650} & .289 & \textbf{.382} & \textbf{1.000} & \underline{.300} & \underline{.400} & .133 \\
 &  & Qwen3-4B base model & .105 & .133 & .091 & .074 & .250 & .483 & .000 & .264 & .000 & .000 & .000 & .000 \\
 &  & Qwen3-8B & .450 & \textbf{.433} & \textbf{.333} & \underline{.218} & \textbf{.900} & \textbf{.750} & \textbf{.367} & .345 & .100 & \textbf{.333} & \textbf{.500} & \underline{.300} \\
 &  & Qwen3-8B base model & .100 & .000 & .079 & .065 & .300 & .500 & .011 & .273 & .000 & .000 & .000 & .000 \\
 &  & Llama-3.2-3B & \underline{.700} & \underline{.333} & \underline{.311} & \textbf{.227} & \textbf{.900} & .633 & \underline{.344} & \underline{.364} & \underline{.600} & \textbf{.333} & \textbf{.500} & \textbf{.333} \\
 &  & Llama-3.2-3B base model & .100 & .200 & .125 & .182 & .200 & .233 & .033 & .109 & .000 & .000 & .000 & .000 \\
\midrule
discovered graph & simulated arms & Qwen3-4B & .106 & \textbf{.203} & \textbf{.208} & .116 & .396 & \textbf{.537} & \textbf{.577} & \textbf{.420} & .106 & .181 & \textbf{.159} & \underline{.145} \\
 &  & Qwen3-4B base model & .021 & .076 & .096 & .089 & .417 & .364 & .225 & .265 & .021 & .035 & .040 & .052 \\
 &  & Qwen3-8B & \textbf{.174} & \underline{.202} & \underline{.200} & \textbf{.124} & \underline{.479} & \textbf{.537} & \underline{.555} & \underline{.406} & \textbf{.170} & \underline{.190} & \textbf{.159} & \underline{.145} \\
 &  & Qwen3-8B base model & .128 & .094 & .112 & .057 & \underline{.479} & .380 & .253 & .210 & .021 & .052 & .034 & .042 \\
 &  & Llama-3.2-3B & \underline{.149} & .186 & .157 & \underline{.120} & \textbf{.500} & \underline{.471} & .483 & .361 & \underline{.149} & \textbf{.215} & \underline{.148} & \textbf{.150} \\
 &  & Llama-3.2-3B base model & .062 & .041 & .028 & .009 & .375 & .248 & .236 & .201 & .043 & .009 & .011 & .005 \\
\midrule
discovered graph & real logs & Qwen3-4B & \textbf{.417} & \underline{.296} & .244 & \underline{.180} & \textbf{.667} & \underline{.518} & \underline{.512} & \textbf{.460} & \textbf{.750} & \underline{.545} & \underline{.606} & \underline{.579} \\
 &  & Qwen3-4B base model & .083 & \textbf{.370} & .171 & .020 & .083 & .259 & .293 & .280 & \underline{.625} & .500 & .485 & .474 \\
 &  & Qwen3-8B & \textbf{.417} & .259 & \textbf{.341} & \textbf{.245} & \underline{.583} & \textbf{.556} & \textbf{.537} & \underline{.380} & \textbf{.750} & .500 & .545 & .474 \\
 &  & Qwen3-8B base model & .167 & \textbf{.370} & .225 & .146 & .250 & .333 & .390 & .220 & .500 & \textbf{.591} & \textbf{.667} & .447 \\
 &  & Llama-3.2-3B & \underline{.250} & .192 & \underline{.293} & .120 & .250 & .481 & \underline{.512} & .340 & .375 & .500 & \underline{.606} & \textbf{.658} \\
 &  & Llama-3.2-3B base model & .000 & .148 & .073 & .040 & .000 & .222 & .073 & .120 & .000 & .091 & .091 & .026 \\
\midrule
discovered graph & vehicle CAN & Qwen3-4B & \textbf{.800} & \underline{.400} & \textbf{.300} & \textbf{.282} & .900 & \underline{.683} & \textbf{.356} & \textbf{.391} & \textbf{.800} & \textbf{.367} & \textbf{.500} & \textbf{.367} \\
 &  & Qwen3-4B base model & .050 & .133 & .023 & .135 & .050 & .367 & .000 & .009 & .000 & .000 & .000 & .000 \\
 &  & Qwen3-8B & .550 & \textbf{.417} & \underline{.222} & .191 & \underline{.950} & \textbf{.717} & \underline{.300} & \underline{.364} & .200 & \textbf{.367} & \textbf{.500} & \underline{.333} \\
 &  & Qwen3-8B base model & .105 & .117 & .057 & .028 & .450 & .533 & .211 & .273 & .000 & .000 & .000 & .000 \\
 &  & Llama-3.2-3B & \underline{.700} & .383 & .144 & \underline{.236} & \textbf{1.000} & \underline{.683} & .222 & .318 & \underline{.400} & \underline{.300} & .000 & .300 \\
 &  & Llama-3.2-3B base model & .050 & .167 & .056 & .045 & .000 & .083 & .022 & .036 & .000 & .000 & .000 & .000 \\
\bottomrule}
\end{table}

\endgroup

\section{Results with Conditional Mutual Information as the Dependence Measure}
\label{apx:cmi}

This appendix compares \ours{} with the form of its objective that takes the CMI as the dependence measure (Appendices~\ref{apx:cmi-vectors} and~\ref{apx:cmi-objective}). Tables~\ref{tab:cmi-quest} and~\ref{tab:cmi-robot} report every metric of \ours{} and of \ours{} with CMI, and Table~\ref{tab:cmi-diff} reports their differences. Both versions name the variables with the adapter, the backbone and the judge of the main results.

\textit{Replacing the distance correlation with the CMI leaves the names essentially unchanged.} The two versions give close results on the questionnaires and on the robot channels, under both structure sources and across the masking levels. Neither is consistently ahead of the other. The CMI turns the independence and dependence lines of Eq.~\ref{eq:realization} into exact equalities between conditional mutual informations, and the main method with the distance correlation writes about the same names. The distance correlation is therefore a sound practical choice for the main method.

\begingroup
\raggedbottom
\makeatletter
\renewenvironment{table}[1][]{\par\addvspace{\intextsep}\begin{center}\begin{minipage}{\linewidth}\def\@captype{table}}{\end{minipage}\end{center}\addvspace{\intextsep}\par}
\makeatother

\begin{table}[!htbp]\centering\scriptsize\setlength{\tabcolsep}{2.4pt}
\caption{Questionnaire items (top) and latent variables (bottom), \ours{} and \ours{} with CMI, under both structure sources.}
\label{tab:cmi-quest}
\fittable{llcccc:cccc:cccc:cccc}{%
\toprule
 &  & \multicolumn{4}{c:}{NRR@1} & \multicolumn{4}{c:}{MRR} & \multicolumn{4}{c:}{SDA} & \multicolumn{4}{c}{identification} \\
structure & method & 20 & 50 & 75 & 90 & 20 & 50 & 75 & 90 & 20 & 50 & 75 & 90 & 20 & 50 & 75 & 90 \\
\midrule
supplied graph & \ours{} & .126 & .097 & .064 & .038 & .254 & .220 & .179 & .121 & .200 & .176 & .147 & .129 & .351 & .314 & .304 & .227 \\
 & \ours{} with CMI & .123 & .101 & .072 & .039 & .256 & .223 & .183 & .119 & .195 & .178 & .154 & .138 & .362 & .324 & .309 & .208 \\
\midrule
discovered graph & \ours{} & .121 & .098 & .059 & .028 & .253 & .216 & .164 & .105 & .226 & .189 & .164 & .147 & .356 & .364 & .283 & .217 \\
 & \ours{} with CMI & .133 & .094 & .057 & .036 & .265 & .218 & .166 & .113 & .231 & .193 & .164 & .153 & .367 & .359 & .301 & .209 \\
\bottomrule}
\par\vspace{3pt}
\fittable{llcccc:cccc:cccc:cccc}{%
\toprule
 &  & \multicolumn{4}{c:}{NRR@1} & \multicolumn{4}{c:}{MRR} & \multicolumn{4}{c:}{identification} & \multicolumn{4}{c}{naming quality} \\
structure & method & 20 & 50 & 75 & 90 & 20 & 50 & 75 & 90 & 20 & 50 & 75 & 90 & 20 & 50 & 75 & 90 \\
\midrule
supplied graph & \ours{} & .485 & .415 & .415 & .238 & .624 & .560 & .551 & .400 & .627 & .537 & .535 & .351 & .596 & .518 & .500 & .332 \\
 & \ours{} with CMI & .488 & .438 & .423 & .250 & .621 & .579 & .555 & .414 & .569 & .591 & .579 & .304 & .542 & .564 & .525 & .273 \\
\midrule
discovered graph & \ours{} & .540 & .486 & .386 & .240 & .660 & .623 & .560 & .424 & .531 & .511 & .441 & .254 & .446 & .426 & .390 & .217 \\
 & \ours{} with CMI & .520 & .486 & .374 & .243 & .645 & .625 & .551 & .422 & .543 & .497 & .406 & .240 & .414 & .420 & .349 & .206 \\
\bottomrule}
\end{table}


\begin{table}[!htbp]\centering\scriptsize\setlength{\tabcolsep}{2.4pt}
\caption{Robot channels, \ours{} and \ours{} with CMI, under both structure sources.}
\label{tab:cmi-robot}
\fittable{lllcccc:cccc:cccc}{%
\toprule
 &  &  & \multicolumn{4}{c:}{NRR@1} & \multicolumn{4}{c:}{MRR} & \multicolumn{4}{c}{SDA} \\
structure & variables & method & 20 & 50 & 75 & 90 & 20 & 50 & 75 & 90 & 20 & 50 & 75 & 90 \\
\midrule
supplied graph & simulated arms & \ours{} & .146 & .198 & .181 & .055 & .317 & .341 & .354 & .225 & .170 & .259 & .222 & .117 \\
 &  & \ours{} with CMI & .146 & .215 & .187 & .064 & .301 & .372 & .361 & .227 & .170 & .250 & .199 & .169 \\
\midrule
supplied graph & real logs & \ours{} & .500 & .333 & .244 & .160 & .566 & .451 & .369 & .276 & .545 & .440 & .333 & .298 \\
 &  & \ours{} with CMI & .500 & .333 & .268 & .160 & .593 & .451 & .386 & .300 & .545 & .440 & .359 & .255 \\
\midrule
supplied graph & vehicle CAN & \ours{} & .750 & .300 & .156 & .145 & .775 & .553 & .317 & .378 & 1.000 & .300 & .400 & .133 \\
 &  & \ours{} with CMI & .500 & .250 & .267 & .227 & .608 & .517 & .414 & .408 & .400 & .233 & .450 & .333 \\
\midrule
discovered graph & simulated arms & \ours{} & .083 & .140 & .137 & .046 & .224 & .286 & .278 & .194 & .128 & .190 & .153 & .164 \\
 &  & \ours{} with CMI & .083 & .132 & .143 & .055 & .232 & .284 & .282 & .194 & .128 & .181 & .165 & .197 \\
\midrule
discovered graph & real logs & \ours{} & .417 & .296 & .195 & .120 & .485 & .385 & .314 & .248 & .455 & .360 & .359 & .255 \\
 &  & \ours{} with CMI & .417 & .296 & .171 & .140 & .519 & .409 & .304 & .263 & .545 & .360 & .359 & .298 \\
\midrule
discovered graph & vehicle CAN & \ours{} & .800 & .367 & .244 & .218 & .859 & .570 & .412 & .397 & .800 & .367 & .500 & .367 \\
 &  & \ours{} with CMI & .750 & .333 & .222 & .245 & .797 & .567 & .383 & .436 & .700 & .300 & .500 & .433 \\
\bottomrule}
\par\vspace{3pt}
\fittable{lllcccc:cccc:cccc}{%
\toprule
 &  &  & \multicolumn{4}{c:}{identification} & \multicolumn{4}{c:}{type acc} & \multicolumn{4}{c}{coord acc} \\
structure & variables & method & 20 & 50 & 75 & 90 & 20 & 50 & 75 & 90 & 20 & 50 & 75 & 90 \\
\midrule
supplied graph & simulated arms & \ours{} & .167 & .252 & .242 & .106 & .542 & .612 & .703 & .525 & .149 & .233 & .216 & .132 \\
 &  & \ours{} with CMI & .188 & .294 & .254 & .131 & .583 & .694 & .753 & .493 & .149 & .250 & .210 & .178 \\
\midrule
supplied graph & real logs & \ours{} & .500 & .407 & .268 & .208 & .750 & .593 & .610 & .500 & 1.000 & .682 & .455 & .553 \\
 &  & \ours{} with CMI & .500 & .370 & .293 & .180 & .667 & .556 & .610 & .520 & 1.000 & .636 & .485 & .605 \\
\midrule
supplied graph & vehicle CAN & \ours{} & .750 & .317 & .216 & .200 & .750 & .650 & .289 & .382 & 1.000 & .300 & .400 & .133 \\
 &  & \ours{} with CMI & .500 & .300 & .311 & .282 & .800 & .633 & .389 & .409 & .400 & .233 & .450 & .333 \\
\midrule
discovered graph & simulated arms & \ours{} & .106 & .203 & .208 & .116 & .396 & .537 & .577 & .420 & .106 & .181 & .159 & .145 \\
 &  & \ours{} with CMI & .087 & .220 & .181 & .140 & .438 & .562 & .588 & .406 & .106 & .164 & .153 & .178 \\
\midrule
discovered graph & real logs & \ours{} & .417 & .296 & .244 & .180 & .667 & .518 & .512 & .460 & .750 & .545 & .606 & .579 \\
 &  & \ours{} with CMI & .417 & .385 & .220 & .180 & .583 & .556 & .463 & .440 & 1.000 & .500 & .576 & .658 \\
\midrule
discovered graph & vehicle CAN & \ours{} & .800 & .400 & .300 & .282 & .900 & .683 & .356 & .391 & .800 & .367 & .500 & .367 \\
 &  & \ours{} with CMI & .750 & .350 & .267 & .291 & .900 & .683 & .333 & .391 & .700 & .300 & .500 & .433 \\
\bottomrule}
\end{table}


\begin{table}[!htbp]\centering\scriptsize\setlength{\tabcolsep}{3pt}
\caption{\ours{} with CMI minus \ours{}, per metric, structure source and masking level, from the values of Tables~\ref{tab:cmi-quest} and~\ref{tab:cmi-robot}. The last column is the mean over the eight cells of a row.}
\label{tab:cmi-diff}
\fittable{ll cccc:cccc:c}{%
\toprule
 &  & \multicolumn{4}{c:}{supplied graph} & \multicolumn{4}{c:}{discovered graph} & \\
variables & metric & 20 & 50 & 75 & 90 & 20 & 50 & 75 & 90 & mean \\
\midrule
items & NRR@1 & $-$.003 & +.004 & +.008 & +.001 & +.012 & $-$.004 & $-$.002 & +.008 & +.003 \\
 & MRR & +.002 & +.003 & +.004 & $-$.002 & +.012 & +.002 & +.002 & +.008 & +.004 \\
 & SDA & $-$.005 & +.002 & +.007 & +.009 & +.005 & +.004 & .000 & +.006 & +.004 \\
 & identification & +.011 & +.010 & +.005 & $-$.019 & +.011 & $-$.005 & +.018 & $-$.008 & +.003 \\
\midrule
latent variables & NRR@1 & +.003 & +.023 & +.008 & +.012 & $-$.020 & .000 & $-$.012 & +.003 & +.002 \\
 & MRR & $-$.003 & +.019 & +.004 & +.014 & $-$.015 & +.002 & $-$.009 & $-$.002 & +.001 \\
 & identification & $-$.058 & +.054 & +.044 & $-$.047 & +.012 & $-$.014 & $-$.035 & $-$.014 & $-$.007 \\
 & naming quality & $-$.054 & +.046 & +.025 & $-$.059 & $-$.032 & $-$.006 & $-$.041 & $-$.011 & $-$.017 \\
\midrule
simulated arms & NRR@1 & .000 & +.017 & +.006 & +.009 & .000 & $-$.008 & +.006 & +.009 & +.005 \\
 & MRR & $-$.016 & +.031 & +.007 & +.002 & +.008 & $-$.002 & +.004 & .000 & +.004 \\
 & SDA & .000 & $-$.009 & $-$.023 & +.052 & .000 & $-$.009 & +.012 & +.033 & +.007 \\
 & identification & +.021 & +.042 & +.012 & +.025 & $-$.019 & +.017 & $-$.027 & +.024 & +.012 \\
 & type acc & +.041 & +.082 & +.050 & $-$.032 & +.042 & +.025 & +.011 & $-$.014 & +.026 \\
 & coord acc & .000 & +.017 & $-$.006 & +.046 & .000 & $-$.017 & $-$.006 & +.033 & +.008 \\
\midrule
real logs & NRR@1 & .000 & .000 & +.024 & .000 & .000 & .000 & $-$.024 & +.020 & +.003 \\
 & MRR & +.027 & .000 & +.017 & +.024 & +.034 & +.024 & $-$.010 & +.015 & +.016 \\
 & SDA & .000 & .000 & +.026 & $-$.043 & +.090 & .000 & .000 & +.043 & +.014 \\
 & identification & .000 & $-$.037 & +.025 & $-$.028 & .000 & +.089 & $-$.024 & .000 & +.003 \\
 & type acc & $-$.083 & $-$.037 & .000 & +.020 & $-$.084 & +.038 & $-$.049 & $-$.020 & $-$.027 \\
 & coord acc & .000 & $-$.046 & +.030 & +.052 & +.250 & $-$.045 & $-$.030 & +.079 & +.036 \\
\midrule
vehicle CAN & NRR@1 & $-$.250 & $-$.050 & +.111 & +.082 & $-$.050 & $-$.034 & $-$.022 & +.027 & $-$.023 \\
 & MRR & $-$.167 & $-$.036 & +.097 & +.030 & $-$.062 & $-$.003 & $-$.029 & +.039 & $-$.016 \\
 & SDA & $-$.600 & $-$.067 & +.050 & +.200 & $-$.100 & $-$.067 & .000 & +.066 & $-$.065 \\
 & identification & $-$.250 & $-$.017 & +.095 & +.082 & $-$.050 & $-$.050 & $-$.033 & +.009 & $-$.027 \\
 & type acc & +.050 & $-$.017 & +.100 & +.027 & .000 & .000 & $-$.023 & .000 & +.017 \\
 & coord acc & $-$.600 & $-$.067 & +.050 & +.200 & $-$.100 & $-$.067 & .000 & +.066 & $-$.065 \\
\bottomrule}
\end{table}

\endgroup

\section{Settings, Compute and Licenses}
\label{apx:compute}

\paragraph{Settings.}
Table~\ref{tab:settings} lists the settings of every stage.

\begin{table}[!htbp]\centering\small\setlength{\tabcolsep}{5pt}
\caption{Settings of every stage.}
\label{tab:settings}
\begin{tabular}{@{}lp{0.56\textwidth}@{}}
\toprule
\multicolumn{2}{@{}l}{\textit{Causal structure learning, robot and vehicle channels}} \\
first difference & every state channel of a simulated arm, and a channel of a real log or a drive when its lag-one autocorrelation exceeds .95 \\
input edge added & partial correlation $\ge .03$, simulated arms only \\
edge dropped & partial correlation $< .05$ given the other parents and the inputs \\
\midrule
\multicolumn{2}{@{}l}{\textit{Constrained embedding optimization}} \\
term weights & $\lambda_{\mathrm{ind}} = .3$, $\lambda_{\mathrm{dep}} = .3$ with $\kappa = .5$ on the pairs above the .7 quantile of dependence, $\lambda_{\mathrm{con}} = 1$ \\
regularizer weights & $\lambda_{\mathrm{bank}} = 1$, $\lambda_{\mathrm{nbr}} = 1$, $\lambda_{\mathrm{norm}} = .1$ \\
\midrule
\multicolumn{2}{@{}l}{\textit{Translation via prefix adaptation}} \\
prefix & length $m = 10$, low-rank update of rank 8 \\
training & AdamW \citep{AdamW} with early stopping on the validation split, 4 epochs for the questionnaires and 2 epochs for the robots \\
loss weights & $\lambda_{\mathrm{sim}} = .5$, $\lambda_{\mathrm{ctr}} = .15$, $\lambda_{\mathrm{sib}} = .5$ \\
backbones & Qwen3-4B-Instruct-2507 (main) and Qwen3-8B, Llama-3.2-3B-Instruct \\
\midrule
\multicolumn{2}{@{}l}{\textit{Evaluation}} \\
masking & 20, 50, 75 and 90\% of the names, with five folds per scale \\
scoring encoder & e5-large-v2, frozen \\
judge & Gemma 3 12B, instruction-tuned, temperature 0, both candidate orders \\
\bottomrule
\end{tabular}
\end{table}

\paragraph{What each stage costs.}
Every run reported here used one A100 class GPU. On the 42-item DASS the original first stage of RLCD had not finished after 35 minutes, while the pipeline with BOSS as its first stage produced the graph used here in under five. HEXACO has 240 items over 5{,}000 respondents. Its structure is discovered in 20 seconds, since the first stage partitions the items into small groups before the rank tests run. A synthetic scale of 150 variables generated from 25 latent variables over 8{,}000 rows takes 107 seconds. Exact rank searches stay practical to about 150 active variables at $k \le 3$, and a larger scale relies on the first stage keeping each group small. With the main backbone of four billion parameters, the adapter and the low-rank update together hold about 51 million. Training the questionnaire adapter takes about an hour, and the robot adapter reaches its selected checkpoint after eight and a half hours. Both are trained a single time.
At inference one name is one prompt and one decoding call.

\begin{figure}[!htbp]\centering
\begin{minipage}{0.86\textwidth}\centering
{\footnotesize\sffamily Deployment cost against outcome, no structure supplied\par}
\includegraphics[width=\linewidth,trim=0 0 0 19,clip]{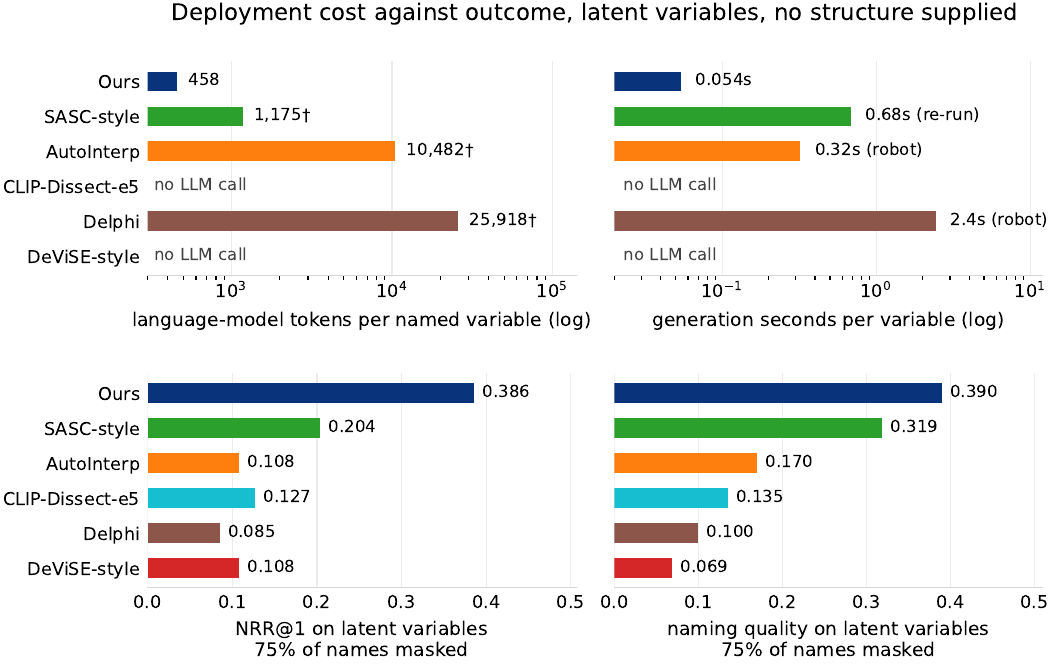}
\end{minipage}
\caption{Cost and outcome per method with no structure
supplied. Top, tokens and generation seconds per named variable on a log
scale. Bottom, NRR@1 and naming quality at 75\% masking. Tokens and naming
metrics are for latent variables. Generation times are from questionnaire
runs for \ours{} and SASC-style, and robot runs for Automated Interpretability
and Delphi. A dagger marks a reconstructed count.}
\label{fig:deploy}
\end{figure}
\paragraph{Naming a new system.}
The deployment cost is the cost of naming a system nobody has
documented, and the three steps together take minutes on one GPU. Ten
vehicle drives go from raw traces to named channels in about seven
minutes. Structure learning takes five seconds, the optimization 334 seconds and writing the names 67 seconds. The names cover the 280 unnamed channels of the four masking levels under each of the eight variants that the evaluation runs. These variants are the full method and its ablations, and they yield 2{,}240 names in all. Structure learning takes a few seconds on a real log and under half a minute on a simulated arm. The structure of a questionnaire scale is learned in seconds to a few minutes. The optimization at one masking level takes a few minutes on the smallest scale, and naming runs at 18.5 variables per second. In the timed runs, the three sampling methods take 0.32 to 2.4 seconds of generation per name against our 0.054 (Figure~\ref{fig:deploy}). They spend 2.6 to 56.6 times the tokens per latent variable (Table~\ref{tab:cost}).

\paragraph{Cost per named variable.}
\label{apx:cost}
The cost of the sampling methods grows with the evidence a column offers,
from the vehicle CAN to the questionnaires, since a channel description is shorter than an item text. Ours stays within a few hundred tokens on every column (Table~\ref{tab:cost}). The lower row of Figure~\ref{fig:deploy} reverses the order, and the cheapest method names the latent variables most accurately.

A cell of Table~\ref{tab:cost} is measured from stored request usage where the protocol was run against a served model. Otherwise it is reconstructed by replaying the published protocol through a token counter. Measured and reconstructed cells both count the protocol's nominal calls and exclude retry overhead.
The generation seconds of Figure~\ref{fig:deploy} all come from the same model snapshot. For \ours{} they are one inference pass over the questionnaire targets, and for SASC-style a timed sequential re-run of its six-call protocol. For Automated Interpretability and Delphi they are the median interval between consecutive cases of the robot run.

\begin{table}[!htb]\centering\small
\caption{Tokens spent per named variable, with the number of model calls the
protocol makes. A dagger marks a cell reconstructed by replaying the
published protocol through the token counter. The other cells are
measured usage.}
\label{tab:cost}
\fittable{lrrrr}{
\toprule
 & questionnaire latent variables & simulated arms & real logs & vehicle CAN \\
\midrule
\ours{} (1 call) & 458 & 276 & 236 & 210 \\
SASC-style (6 calls) & 1{,}175$^{\dagger}$ & 1{,}125$^{\dagger}$ & 1{,}000$^{\dagger}$ & 774$^{\dagger}$ \\
Automated Interpretability (2 calls) & 10{,}482$^{\dagger}$ & 8{,}903 & 6{,}409 & 4{,}117 \\
Delphi (3 calls) & 25{,}918$^{\dagger}$ & 21{,}878 & 15{,}665 & 9{,}960 \\
\bottomrule
}
\end{table}

\paragraph{Licenses.}
The questionnaire responses come from the publicly available anonymized
datasets listed in Table~\ref{tab:data-q}. The comma2k19 dataset, \texttt{robosuite} and two of the real logs are released under the MIT license, and the third log under Apache 2.0. ConceptNet Numberbatch is released under
CC BY-SA 4.0, the Cognitive Atlas under CC BY-SA 3.0 and WordNet under its own permissive
license. The scales themselves are the published instruments listed in
Table~\ref{tab:data-q}.

\section{Recovered Structures}
\label{apx:struct-merged}
\makeatletter\setlength{\@fptop}{0pt plus 1fil}\setlength{\@fpsep}{16pt plus 2fil}\setlength{\@fpbot}{0pt plus 1fil}\makeatother

\paragraph{Robot and vehicle channels.}
\label{apx:struct-robot}
Each figure draws the supplied graph of one arm, log or drive on the left and the discovered graph on the right, with identical node positions. The graphs are drawn at the channel level. The graphs are the summary graphs of Appendix~\ref{apx:discovery-robot}, and Table~\ref{tab:struct-robot} counts the edges of every arm in the time-lagged form of Eq.~\ref{eq:temporal}.

\begin{figure}[!htb]\centering
\makebox[\textwidth][c]{\includegraphics[width=1.15\textwidth]{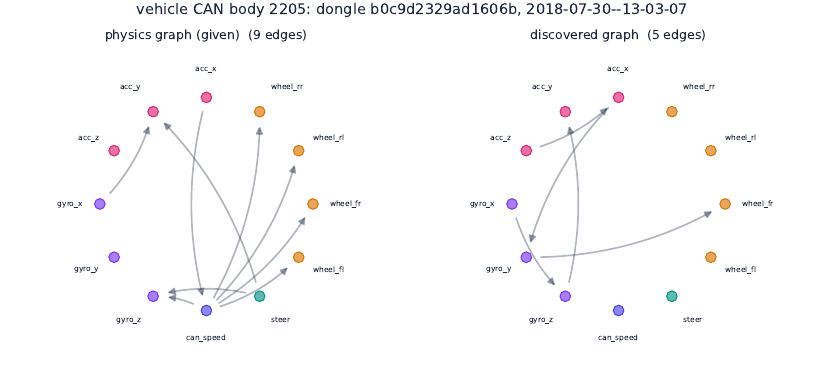}}
\caption{Vehicle CAN drive 2205, which is the drive of the introduction. In the supplied graph, the vehicle speed has edges to the four wheel speeds. The discovered graph has a path from the vertical acceleration through the longitudinal acceleration to the pitch rate. It also has an edge from the pitch rate to the front right wheel speed and two more edges among the rates and the lateral acceleration.}
\label{fig:apx-rob-2205}
\end{figure}

\begin{figure}[!htb]\centering
\makebox[\textwidth][c]{\includegraphics[width=1.15\textwidth]{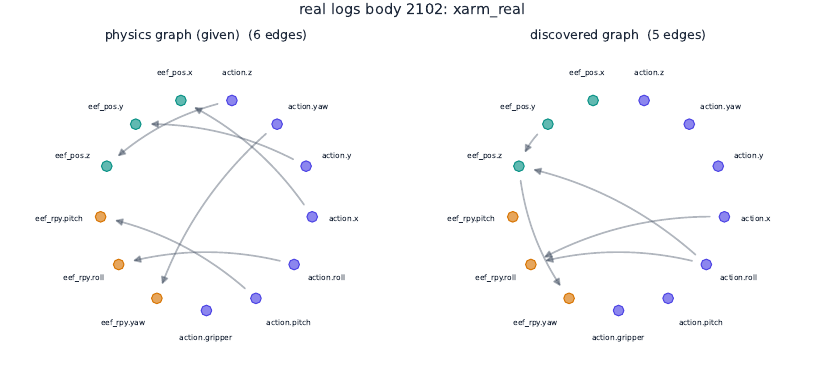}}
\caption{Real robot log 2102, which is the UTokyo xArm pick-and-place log.}
\label{fig:apx-rob-2102}
\end{figure}

\begin{figure}[!htb]\centering
\makebox[\textwidth][c]{\includegraphics[width=1.15\textwidth]{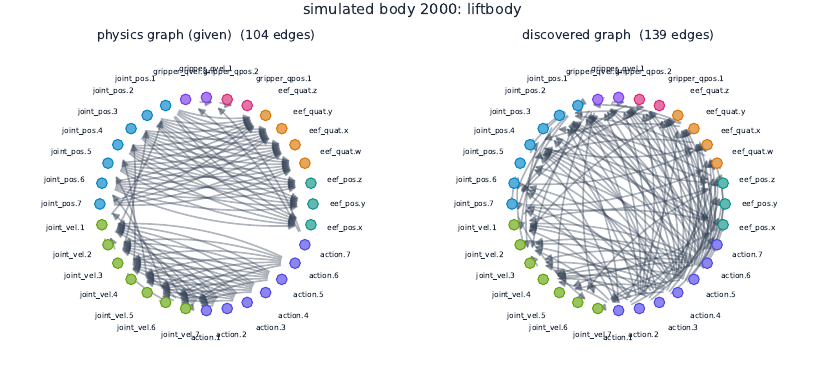}}
\caption{Simulated arm 2000, which is the Panda of the lift task. Its discovered graph is drawn against its true graph.}
\label{fig:apx-rob-2000}
\end{figure}

\paragraph{Questionnaires.}
Each figure draws the published scoring key of a scale above and the structure the discovery step recovers below. Each latent variable has a color, and the items it governs share that color. An item is drawn at its position along the row and marked with its identifier in the source data. A long identifier is truncated to fit. Table~\ref{tab:struct-q} in Appendix~\ref{apx:discovery} gives the counts behind these figures and covers the three scales not drawn.

\begin{figure}[!htb]\centering
\includegraphics[width=0.86\textwidth]{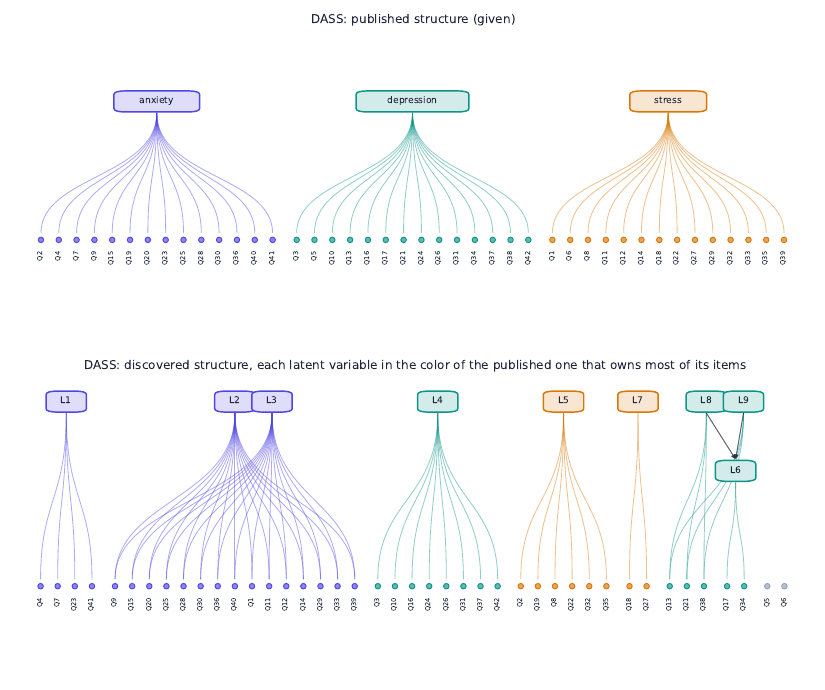}
\caption{DASS, the published structure above and the discovered structure below. Six of the nine discovered latent variables lie inside one published latent variable.}
\label{fig:apx-struct-dass}
\end{figure}

\begin{figure}[!htb]\centering
\includegraphics[width=0.86\textwidth]{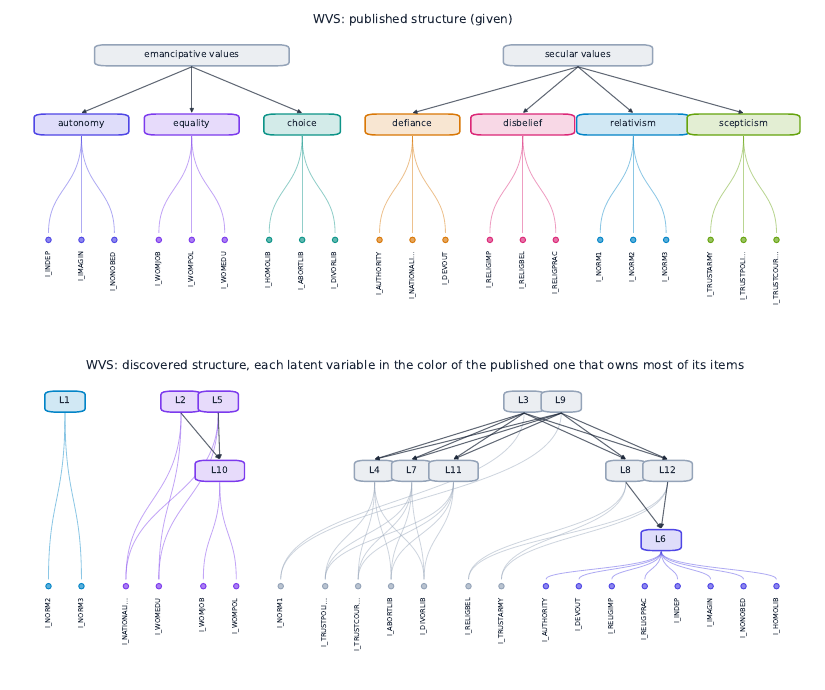}
\caption{World Values Survey, the published structure above and the discovered structure below. The discovered groups cut across the Welzel sub-indices more than on any other scale.}
\label{fig:apx-struct-wvs}
\end{figure}

\end{document}